\documentclass[10pt,twocolumn]{article}
\usepackage[letterpaper,margin=0.8in,columnsep=0.25in]{geometry}
\usepackage{fancyhdr}
\usepackage[utf8]{inputenc}
\usepackage[T1]{fontenc}
\usepackage{lmodern}
\usepackage[round]{natbib}

\usepackage{url}
\usepackage{booktabs}
\usepackage{array}
\usepackage{tabularx}
\usepackage[font=small,labelfont=bf]{caption}
\usepackage{amsmath,amssymb,amsfonts,amsthm}
\usepackage{mathtools}
\usepackage{microtype}
\usepackage{xcolor}
\usepackage{graphicx}
\usepackage{tikz}
\usetikzlibrary{arrows.meta,positioning,calc,fit,shapes.geometric}
\usepackage{enumitem}
\usepackage{float}
\floatstyle{ruled}
\newfloat{algorithm}{t}{loa}
\floatname{algorithm}{Algorithm}
\floatstyle{plain}
\usepackage{needspace}
\usepackage{longtable}
\usepackage{balance}
\usepackage{placeins}
\usepackage[colorlinks=true,linkcolor=steBlue,citecolor=steBlue,urlcolor=steBlue]{hyperref}

\newtheorem{definition}{Definition}
\newtheorem{theorem}{Theorem}
\newtheorem{proposition}{Proposition}
\newtheorem{lemma}{Lemma}
\newtheorem{corollary}{Corollary}

\newcommand{\agents}{\mathcal{A}}

\newcommand{\TC}{\mathrm{TC}}
\newcommand{\UC}{\mathrm{UC}}

\newcommand{\smin}{\operatorname{smin}}
\newcommand{\smax}{\operatorname{smax}}

\definecolor{steBlue}{HTML}{2B6CB0}
\definecolor{steRed}{HTML}{C53030}
\definecolor{steGreen}{HTML}{2F855A}
\definecolor{steGray}{HTML}{4A5568}
\definecolor{steLight}{HTML}{EDF2F7}
\newcommand{\WB}{\textnormal{STE-WB}}
\newcommand{\LSE}{\textnormal{STE-LSE}}
\newcommand{\AP}{\operatorname{AP}}

\begin{document}
\title{Soft Tournament Equilibrium:\\Differentiable Set-Valued Inference for Non-Transitive Pairwise Comparisons}
\author{Saad Alqithami\\\small\texttt{alqithami@gmail.com}}
\date{}
\maketitle
\hypersetup{pdftitle={Soft Tournament Equilibrium: Differentiable Set-Valued Inference for Non-Transitive Pairwise Comparisons},pdfauthor={Saad Alqithami},pdfsubject={Revised preprint for arXiv:2604.04328}}

\begin{abstract}
Soft Tournament Equilibrium (STE) is a differentiable layer for Top-Cycle (TC) and Uncovered-Set (UC) inference from reciprocal pair probabilities. Normalized log-sum-exp reachability and covering give smooth scores with approximation, perturbation, and margin-recovery bounds. We distinguish structural supervision, posterior uncertainty, and the final set decision through controlled synthetic studies and reconstruction of recorded ordinal profiles. Matched-epoch training improves F1 under a common structural readout, but a separate equal-budget soft-UC comparison does not demonstrate an advantage. A prospective equal-budget hard-UC study improves selective F1 at 24 alternatives from 0.6327/0.6292 for ordinary/relational native heads to 0.6676, while exact recovery remains only 1.16\%. On 36 human profiles held out by source, a separate exploratory comparison favors Jeffreys posterior inference over learned independent-edge and mixture distributions in native expected-F1 decoding (0.8742 versus 0.8704/0.8366). Only three references are non-singleton selective cores. Completion certificates and full-depth TC diagnostics clarify additional identification, smoothing, and computational limits. The evidence supports conditional synthetic overlap gains, while preserving the soft-set null and the absence of a demonstrated learning advantage over strong count-based human-profile controls.
\end{abstract}

\section{Introduction}\label{sec:intro}
Pairwise comparisons support preference learning and model evaluation, but their majority relation can cycle: $a$ defeats $b$, $b$ defeats $c$, and $c$ defeats $a$. Heterogeneous contexts can produce cycles even when each context is ordered. Language-model evaluations exhibit non-transitivity \citep{xu2025nontransitivity}; other work ensembles evaluators or filters toward acyclic relations \citep{hu2026acyclic,yu2025elspr}. A ranking and a mutually competitive leading set answer different questions.

Classical tournament solutions provide structural targets \citep{brandt2016handbook,laslier1997tournament}. The \emph{Top Cycle} (TC) consists of alternatives that can reach every other alternative along majority-win paths. The \emph{Uncovered Set} (UC) excludes an alternative when another defeats it and every opponent it defeats. Both select the Condorcet winner when one exists. Unlike a fixed-size ranking shortlist, these sets are determined by the relation itself. Their hard definitions, however, are discontinuous in estimated edge probabilities near a majority tie.

\emph{Soft Tournament Equilibrium} (STE) maps pair probabilities to continuous TC/UC scores. Equilibrium denotes a tournament core rather than a Nash equilibrium or fixed point. Smoothing reachability and covering lets core labels supervise the pair predictor. We ask when differentiating through a tournament solution helps: relation learning, soft selection, and deployed decisions need not improve together. A separate posterior branch describes model-based membership uncertainty.

\begin{figure*}[t]
\centering
\includegraphics[width=0.92\textwidth]{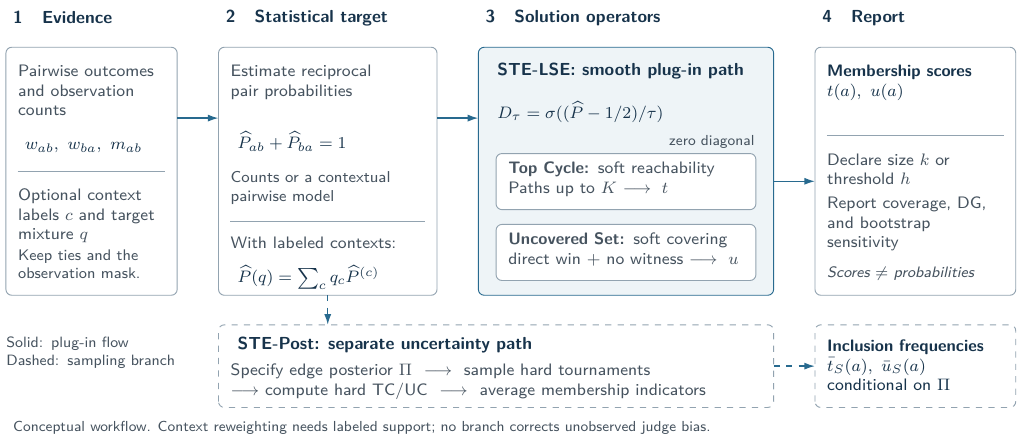}
\caption{\textbf{STE inference workflow.} Comparison records define reciprocal pair probabilities. STE-LSE computes smooth TC/UC scores for supervision; STE-Post averages hard memberships under a specified posterior. Reporting a set requires a separate decision rule. Training differentiates the smooth scores, not the selected set.}
\label{fig:overview}
\end{figure*}

We contribute smooth TC/UC operators, approximation and perturbation bounds, and counterexamples separating scores from posterior and hard membership. Three twelve-collection GPU studies compare structural supervision under matched epochs and equal budgets. Whole-source human-profile studies evaluate the complete decision against count-only and fitted posterior controls, including a learned dependent orientation model. Completion-based certificates separate logically justified membership from unresolved edges. Recorded-core reconstruction and gradient/cost diagnostics retain their separate scopes; oracle size is diagnostic only.

\noindent\textbf{Relation to prior work.} TC and UC are established social-choice objects \citep{good1971note,smith1973aggregation,miller1980graph,fishburn1977condorcet}. Fuzzy spatial-voting models study covering and uncovered sets under indifference \citep{mordeson2011uncovered}; STE instead smooths majority and core predicates into differentiable membership scores for supervision. Bradley--Terry--Luce (BTL), Rank Centrality, and HodgeRank fit scalar or spectral rankings \citep{bradley1952rank,luce1959individual,negahban2017rank,jiang2011statistical}. UCB-style dueling-bandit algorithms identify TC, UC, and other winning sets through active comparisons \citep{ramamohan2016dueling}; STE takes a supplied probability matrix. Differentiable sorting relaxes an ordering \citep{blondel2020fast}, and generic perturbed optimizers smooth discrete solution maps \citep{berthet2020learning}. Soft Condorcet Optimization learns scalar ratings for a ranking objective \citep{lanctot2024sco}; STE differentiates core membership of a potentially cyclic relation. Margin-of-victory methods measure edge changes needed to alter membership \citep{brill2022margin}. Our contribution is the explicit smooth TC/UC construction, its error analysis, and its controlled evaluation as a training layer.

\section{Statistical target and hard cores}\label{sec:problem}
Let $\agents=\{1,\ldots,n\}$, $n\ge2$, and let $P_{ab}$ be the probability that $a$ defeats $b$ under a declared evaluation distribution. We use \emph{reciprocal} probabilities, $P_{ab}+P_{ba}=1$ and $P_{aa}=1/2$; the corresponding margins or logits, not $P$ itself, are antisymmetric. Counts $w_{ab}$ and $w_{ba}$ summarize $m_{ab}=w_{ab}+w_{ba}$ decisive comparisons. Our primary estimand conditions on a decisive outcome. Observed ties are recorded separately rather than treated as wins.

\noindent\textbf{The comparison distribution is part of the target.} For labeled contexts, write $P^{(c)}_{ab}=\Pr(a\succ b\mid a,b,c,\text{decisive})$ and define $P(q)=\sum_c q_cP^{(c)}$ with declared pair-independent weights. These standardize context-conditional decisive relations. Discarding ties from pooled observations generally induces different weights when tie rates vary by pair/context; observing pairs on different tasks also need not identify this mixture. Reweighting requires labels, support, and a stated target distribution (Appendix~\ref{app:decisive-mixtures}).

For $P_{ab}\ne1/2$ on every unordered pair, $T(P)$ is the strict tournament with edge $a\to b$ iff $P_{ab}>1/2$. Exact ties are allowed in $P$, but omitting both directions produces an incomplete directed graph, not a strict tournament.
\begin{figure}[t]
\centering
% Deterministic example: all six strict-majority edges, no fitted data.
% Circles mark UC members; the square marks the covered vertex. All are in TC.
\begin{tikzpicture}[
  >=Stealth,
  every node/.style={font=\fontsize{7.2}{8.4}\selectfont},
  ucvertex/.style={circle,draw=steBlue,line width=.65pt,fill=steBlue!7,
    minimum size=4.9mm,inner sep=0pt},
  coveredvertex/.style={rectangle,draw=steGray,line width=.65pt,fill=steGray!13,
    minimum size=4.9mm,inner sep=0pt},
  majority/.style={-{Stealth[length=2.0mm,width=1.35mm]},draw=steGray,line width=.55pt},
  certificate/.style={-{Stealth[length=2.0mm,width=1.35mm]},draw=steBlue,line width=.75pt},
  panel/.style={anchor=north west,align=left,inner sep=0pt,text width=5.25cm}
]
% Bounding box: about one column wide and 3.0 cm high.
\path[use as bounding box] (-.22,-.60) rectangle (7.83,2.40);
\node[anchor=west,inner sep=0pt,font=\fontsize{7.2}{8.4}\selectfont\bfseries]
  at (-.06,2.20) {(a) Tournament};
\node[ucvertex] (a) at (0,1.74) {$a$};
\node[ucvertex] (b) at (1.95,1.74) {$b$};
\node[ucvertex] (c) at (.975,.12) {$c$};
\node[coveredvertex] (d) at (.975,.95) {$d$};

% Arrowless annotation drawn first: majority edges remain visually continuous.
\draw[draw=steGray,densely dashed,line width=.35pt]
  (d.east) -- (2.36,.95) |- (2.53,.43);

% Exactly six majority edges. Blue edges certify that b covers d.
\draw[majority] (a) -- (b);
\draw[majority] (c) to[bend left=15] (a);
\draw[certificate,preaction={draw=white,line width=2.3pt}]
  (b) to[bend left=16] (c);
\draw[majority] (a) -- (d);
\draw[certificate] (b) -- (d);
\draw[majority] (d) -- (c);

\node[anchor=west,inner sep=0pt,font=\fontsize{7.2}{8.4}\selectfont\bfseries]
  at (2.53,2.20) {(b) Membership certificates};
\node[panel] at (2.53,1.93) {
  $N^+(a)=\{b,d\},\quad N^+(b)=\{c,d\}$\\[2pt]
  $N^+(c)=\{a\},\qquad N^+(d)=\{c\}$
};
\node[panel] at (2.53,1.15) {
  \textbf{TC:} $\{a,b,c,d\}$ (full panel)\\[2pt]
  $d\to c\to a\to b\to d$ is a spanning cycle.
};
\node[panel,fill=white] at (2.53,.51) {
  \textbf{UC:} $b$ covers $d$ $\Rightarrow\UC=\{a,b,c\}$\\[2pt]
  $b\to d,\quad N^+(d)=\{c\}\subset N^+(b)$
};

% Explicit shape and line semantics make the figure readable in monochrome.
\node[ucvertex,minimum size=2.4mm] at (.02,-.39) {};
\node[anchor=west,inner sep=0pt,font=\fontsize{6.7}{7.4}\selectfont]
  at (.21,-.39) {UC member};
\node[coveredvertex,minimum size=2.4mm] at (2.03,-.39) {};
\node[anchor=west,inner sep=0pt,font=\fontsize{6.7}{7.4}\selectfont]
  at (2.22,-.39) {covered};
\node[anchor=west,inner sep=0pt,font=\fontsize{6.7}{7.4}\selectfont]
  at (3.62,-.39) {All nodes: TC};
\node[anchor=west,inner sep=0pt,font=\fontsize{6.7}{7.4}\selectfont,text=steBlue]
  at (5.65,-.39) {Blue: cover proof};
\end{tikzpicture}
\caption{One relation, distinct TC and UC targets. Solid arrows show all six majority edges; blue edges certify that $b$ covers $d$. Circles mark UC members and the square marks covered $d$. The spanning cycle certifies full TC. The dashed leader is an annotation, not an edge.}
\label{fig:worked-tournament}
\end{figure}
\Needspace{8\baselineskip}
\begin{definition}[TC and UC on a strict tournament]\label{def:cores}
With self-reachability understood, $\TC(T)=\{a:a\leadsto_T b\ \forall b\in\agents\}$. For distinct $c,a$, $c$ \emph{covers} $a$ iff $c\to a$ and $a\to b$ implies $c\to b$ for every $b\notin\{a,c\}$. Then $\UC(T)=\{a:\nexists c\text{ covering }a\}$.
\end{definition}
Reachability and covering can also be evaluated on incomplete graphs, without inheriting strict-tournament properties. We identify this convention when observed pairs are missing or tied. Figure~\ref{fig:worked-tournament} shows a strict tournament with distinct cores.

\noindent\textbf{What missing comparisons leave unidentified.} Observing only $a\to b$ and $b\to c$ permits two completions: $a\to c$ gives $\TC=\UC=\{a\}$, while $c\to a$ gives $\{a,b,c\}$. Neutral imputation does not distinguish them. We retain the observation mask and distinguish soft estimation from certified hard membership.

\section{Soft tournament inference}\label{sec:method}
\subsection{Edge estimation and canonical operators}
The edge estimator is modular. An unrestricted differentiable scorer $g_\theta$ can give
\begin{equation}\label{eq:learner}
\begin{split}
h_\theta(a,b,x)&=g_\theta(a,b,x)-g_\theta(b,a,x),\\
P_\theta(a\succ b\mid x)&=\sigma(h_\theta(a,b,x)),
\end{split}
\end{equation}
where $\sigma(z)=(1+e^{-z})^{-1}$ and $x$ is a comparison context. Pairwise cross-entropy can fit this model. A scalar scorer $s_\theta(a,x)$ gives the restricted contextual BTL logit $s_\theta(a,x)-s_\theta(b,x)$. This enforces transitivity within a context, not necessarily after marginalizing contexts. Our count experiments use $\widehat P_{ab}=(w_{ab}+1/2)/(m_{ab}+1)$, including $\widehat P_{ab}=1/2$ for missing pairs.

For edge temperature $\tau>0$, define
\begin{equation}\label{eq:edge}
D_\tau(a,b)=\begin{cases}0,&a=b,\\ \sigma((P_{ab}-1/2)/\tau),&a\ne b.\end{cases}
\end{equation}
For a nonempty $r$-vector $z$, normalized log-sum-exp (LSE) extrema are
\begin{equation}\label{eq:extrema}
\begin{split}
\smax_\gamma(z)&=\gamma\log\!\left(r^{-1}\sum_{i=1}^r e^{z_i/\gamma}\right),\\
\smin_\gamma(z)&=-\smax_\gamma(-z).
\end{split}
\end{equation}
They map $[0,1]^r$ into $[0,1]$ and approximate max/min within $\gamma\log r$. Normalization preserves the score range but makes finite-temperature values sensitive to aggregation size; it is not probability calibration. Temperatures $\tau,\gamma,\gamma_c$ respectively smooth majority edges, paths, and covering; they do not estimate observation noise. Tying them is an experimental choice.

\subsection{Top-Cycle membership}
Hard reachability is an OR over paths of an AND over their edges. Define the smooth max--min product
\begin{equation}\label{eq:product}
(X\otimes_\gamma Y)_{ab}=\smax_\gamma\{\smin_\gamma(X_{ac},Y_{cb}):c\in\agents\}.
\end{equation}
With $Q^{(1)}=D_\tau$ and $Q^{(k)}=Q^{(k-1)}\otimes_\gamma D_\tau$, set
\begin{align}
R(a,b)&=\smax_\gamma\{Q^{(k)}(a,b):1\le k\le K\},\label{eq:reach}\\
t(a)&=\smin_\gamma\{R(a,b):b\ne a\}.\label{eq:tc}
\end{align}
In the hard limit, the inner minimum tests a path's weakest edge, the maximum tests existence of a path, and the final minimum requires reachability to every opponent. Full TC recovery requires $K\ge n-1$; a smaller $K$ is a bounded-depth surrogate. At positive temperature the nested operations are a particular smooth recurrence, not a probability that a path exists: overlapping walks reuse edges, and normalized soft extrema are not Boolean probabilities.

\subsection{Uncovered-Set membership}
A counterexample to ``$c$ covers $a$'' is an opponent beaten by $a$ but not by $c$. Its soft score is
\begin{equation}\label{eq:violation}
v(c,a)=\smax_{\gamma_c}\{D_\tau(a,b)(1-D_\tau(c,b)):b\notin\{a,c\}\}.
\end{equation}
Use $v(c,a)=0$ for an empty witness set ($n=2$), and define
\begin{align}
\operatorname{cov}(c,a)&=D_\tau(c,a)(1-v(c,a)),\quad c\ne a,\\
u(a)&=1-\smax_{\gamma_c}\{\operatorname{cov}(c,a):c\ne a\}.\label{eq:uc}
\end{align}
UC does not use $R$ or $K$. Its witness test is local to triples $(c,a,b)$, whereas TC aggregates paths through the whole graph. The two scores answer different structural questions and need not select the same finite-temperature shortlist. 

\begin{algorithm}[t]
\caption{Canonical STE-LSE forward pass}\label{alg:lse}
\small
\noindent\textbf{Input:} $P\in[0,1]^{n\times n}$, $P_{ab}+P_{ba}=1$, $P_{aa}=1/2$; $n\ge2$, $\tau,\gamma,\gamma_c>0$, integer $K\ge1$.
\par\smallskip
\begingroup
\setlength{\tabcolsep}{0pt}
\renewcommand{\arraystretch}{1.02}
\begin{tabular}{@{}r@{\hspace{0.55em}}p{\dimexpr\linewidth-2em\relax}@{}}
1 & Set $D_{aa}\gets0$; for $a\ne b$, set $D_{ab}\gets\sigma((P_{ab}-1/2)/\tau)$.\\
2 & $Q\gets D$, $L\gets D/\gamma$.\\
3 & \textbf{for} $k=2,\ldots,K$:\\
4 & \quad Allocate a new $n\times n$ matrix $Q^{\rm new}$.\\
5 & \quad\textbf{for} $b\in\agents$, write column $b$ of $Q^{\rm new}$:\\
6 & \qquad $H_{ac}\gets\smin_\gamma(Q_{ac},D_{cb})$ for all $a,c$.\\
7 & \qquad $Q^{\rm new}_{ab}\gets\smax_\gamma\{H_{ac}:c\in\agents\}$ for all $a$.\\
8 & \quad $Q\gets Q^{\rm new}$; $L\gets\operatorname{logaddexp}(L,Q/\gamma)$.\\
9 & $R\gets\gamma(L-\log K)$, entrywise.\\
10 & $t_a\gets\smin_\gamma\{R_{ab}:b\ne a\}$ for all $a$.\\
11 & \textbf{for} $a\in\agents$:\\
12 & \quad\textbf{for} $c\in\agents\setminus\{a\}$:\\
13 & \qquad If $n=2$, set $v(c,a)\gets0$; otherwise set\\
14 & \qquad $v(c,a)\gets\smax_{\gamma_c}\{D_{ab}(1-D_{cb}):b\notin\{a,c\}\}$.\\
15 & \qquad $\operatorname{cov}(c,a)\gets D_{ca}(1-v(c,a))$.\\
16 & \quad $u_a\gets1-\smax_{\gamma_c}\{\operatorname{cov}(c,a):c\ne a\}$.\\
17 & \textbf{return} continuous scores $t,u$, before any set decision.\\
\end{tabular}
\endgroup
\end{algorithm}

\subsection{Variants and uncertainty}\label{sec:variants}
Eqs.~\eqref{eq:edge}--\eqref{eq:uc} define \LSE; Algorithm~\ref{alg:lse} gives a streamed forward reference, separate from the GPU implementation (Appendix~\ref{app:ibm-gpu}). The historical \WB\ variant instead confidence-weights edges, uses Boltzmann UC aggregation, and exact max--min TC paths. Its equations and 64,800-instance archive are distinguished in Appendix~\ref{app:variant-details}; the GPU studies use \LSE.

\textbf{STE-Post} instead samples a hard tournament from a specified edge posterior and averages hard memberships. For example, independent edge posteriors derived from symmetric Beta priors produce $\bar u_S(a)=S^{-1}\sum_s\mathbf1\{a\in\UC(T^{(s)})\}$. These frequencies describe the chosen posterior, not a model-free probability of membership. None of these procedures estimates judge quality or corrects unobserved judge bias.  

\subsection{Selection and core-supervised learning}\label{sec:training-objective}
A score vector $z$ becomes a reported set only after specifying a rule, such as $S_h(z)=\{a:z_a\ge h\}$. We select $h$ on separate development tournaments and hold it fixed on tests. This selects a decision rule, not calibrated probabilities or an oracle size. The distinction between smooth scores, posterior inclusion, and hard membership is formalized in Section~\ref{sec:score-meaning}.

For a minibatch of observed graphs $W$ and associated labels $y_a=\mathbf1\{a\in C(T)\}$, $C\in\{\TC,\UC\}$, we use
\begin{equation}\label{eq:coretraining}
\mathcal L(\theta)=\mathcal L_{\rm pair}(P_\theta;W)
+\lambda\operatorname{BCE}(f_C(P_\theta),y).
\end{equation}
Here $f_{\TC}=t$, $f_{\UC}=u$, and $\mathcal L_{\rm pair}$ normalizes by the total comparison count in the minibatch; BCE averages across its graphs and alternatives. Gradients pass through soft reachability or covering and the reciprocal pair predictor, not through $S_h$. Section~\ref{sec:learning} compares this objective with pairwise-only training and a direct membership head using the same target-specific core labels. Synthetic experiments use latent-generator cores; exploratory human-profile learning uses full-record empirical cores from training sources. These references do not supply population truth.

\section{Error bounds and context transfer}\label{sec:theory}
We separate approximation, estimation, and changes in the target. Appendix~\ref{app:proofs} provides complete proofs and recovery certificates. These supplied-matrix guarantees establish neither optimization convergence nor predictor generalization.

\begin{theorem}[Finite-temperature core approximation]\label{thm:finite}
Let $M$ be a strict tournament adjacency, $D\in[0,1]^{n\times n}$ have zero diagonal, $\eta=\|D-M\|_\infty$, and $K\ge n-1$. The canonical STE-LSE scores satisfy
\begin{align}
\|t(D)-\mathbf1_{\TC(M)}\|_\infty
&\le\eta+\gamma\{(K-1)\log n+\log K\},\label{eq:tc-bound}\\
\|u(D)-\mathbf1_{\UC(M)}\|_\infty
&\le3\eta+\gamma_c\log(n-1).\label{eq:uc-bound}
\end{align}
Each bound may be capped at one. If it is below $\min(h,1-h)$, thresholding at $h\in(0,1)$ recovers the hard set; $h=1/2$ maximizes this allowance.
\end{theorem}
The proof tracks the under-approximation from normalized soft maxima separately from the over-approximation from soft minima, avoiding the sum of opposing smoothing allowances. For strict-margin $P$ with $\Delta=\min_{a\ne b}|P_{ab}-1/2|>0$ and $\|\widehat P-P\|_\infty\le\epsilon<\Delta$, Eq.~\eqref{eq:edge} gives $\eta\le e^{-(\Delta-\epsilon)/\tau}$. At fixed finite $n,K$, uniform consistency and temperatures tending to zero yield hard-core consistency and Condorcet inclusion. Exact ties do not identify a unique strict orientation.

The bounds are sufficient, not necessary. Even the sharpened TC smoothing term exceeds $1/2$ at $n=12$, $K=11$, $\gamma=0.035$, and exceeds one at $n=24$; it grows with $K\log n$. It therefore does not certify the GPU study's operating points. Binary target structure sharpens UC's edge term to $2\eta-\eta^2$ (Appendix~\ref{rem:uc-binary-refinement}), without changing the operator. More observations reduce statistical error, but neither fixed smoothing bias nor posterior Monte Carlo precision. Missing pairs still require identification assumptions.

Let $t_0$ use ordinary maxima/minima in the same TC recurrence. Under Theorem~\ref{thm:finite}'s assumptions, Appendix~\ref{app:tc-certificates} proves
\[
|t_a(D)-\mathbf1_{\TC(M)}(a)|
\le\min\{1,\eta+|t_a(D)-t_{0,a}(D)|\}.
\]
The computable smoothing discrepancy still needs a valid uniform edge-error allowance $\eta$; average pair MSE does not suffice. A separate one-sided refinement uses members' shortest-path lengths. Both preserve the operator without certifying unchecked recorded operating points.

\begin{proposition}[Fixed-temperature perturbations]\label{prop:lipschitz}
For reciprocal $P,P'$ of the same size, fixed positive temperatures, and the same $K\ge1$, let $e=\|P-P'\|_\infty$. STE-LSE satisfies
\begin{equation}\label{eq:regularity}
\begin{split}
\|t(P)-t(P')\|_\infty&\le e/(4\tau),\\
\|u(P)-u(P')\|_\infty&\le 3e/(4\tau).
\end{split}
\end{equation}
These inequalities do not require strict majority margins.
\end{proposition}
Hard memberships can jump when an edge crosses $1/2$; Eq.~\eqref{eq:regularity} bounds smooth-score perturbations even there. Smaller $\tau$ sharpens the hard approximation but increases worst-case sensitivity to probability error. The TC perturbation constant is independent of $K$, whereas its \emph{approximation bias} in Theorem~\ref{thm:finite} accumulates with depth.

\begin{proposition}[Nonuniform comparison counts]\label{prop:sample}
Suppose each unordered pair has $m_{ab}>0$ independent Bernoulli observations with mean $P_{ab}$ and margin $\Delta_{ab}=|P_{ab}-1/2|>0$. For the raw empirical mean $\widehat P$,
\begin{equation}\label{eq:sample}
\Pr\{T(\widehat P)\ne T(P)\}\le 2\sum_{a<b}e^{-2m_{ab}\Delta_{ab}^2}.
\end{equation}
On the complementary event, both hard cores agree. Independence across different pairs is not needed for this union bound.
\end{proposition}
For uniform $m$ and minimum margin $\Delta$, this reduces to $n(n-1)e^{-2m\Delta^2}$. Corollary~\ref{cor:finite-sample-recovery} combines this concentration argument with edge smoothing and Theorem~\ref{thm:finite} to certify thresholded scores under the stated Bernoulli model. Missing pairs and pseudocount bias require separate treatment.

\begin{proposition}[Mixture transfer with weight error]\label{prop:mixture}
Let context probabilities $P^{(c)}$ be unchanged between training and evaluation, with $\|\widehat P^{(c)}-P^{(c)}\|_\infty\le\epsilon_c$. For target weights $q$ and estimated weights $\widehat q$ in the probability simplex,
\begin{equation}\label{eq:mixture}
\left\|\sum_c\widehat q_c\widehat P^{(c)}-\sum_cq_cP^{(c)}\right\|_\infty
\le \sum_c\widehat q_c\epsilon_c+\tfrac12\|\widehat q-q\|_1.
\end{equation}
If this bound is below the target's strict majority margin, both induced hard cores agree.
\end{proposition}
Known weights eliminate the last term; context support and stable conditional probabilities remain essential. Appendix~\ref{app:proofs} also bounds model-conditional posterior Monte Carlo error.

% Optional main-text insertion. The defining result and complete proof are in
% sections/posterior_certificates_theory.tex; do not duplicate theorem numbering.
\paragraph{Posterior decisions and logical certificates.}
A tournament law and a rule for reporting its UC are separate objects. GFM uses the joint moments $\Pr(i\in\UC(T),|\UC(T)|=s\mid D)$ to maximize modeled expected F1 \citep{waegeman2014fmeasure}; membership marginals alone omit its size-dependent denominator. Shared mixture components also induce dependent edges, so multiplying mixture edge marginals generally gives a different law.

Known orientations can alternatively bound UC without recovering every edge. Define $I(G)$ as vertices with known paths of length at most two to every opponent, and exclude $v$ from $O(G)$ when a known $u\to v$ has, for each $w\notin\{u,v\}$, known $w\to v$ or $u\to w$. Proposition~\ref{prop:uc-completion-enclosure} proves $I(G)\subseteq\UC(T)\subseteq O(G)$ for every strict completion $T$; equality certifies the complete set. This applies established two-step tournament supports \citep{aziz2015partial,contet2026supports}. Fixed-design i.i.d. pair intervals make the enclosure simultaneous with probability $1-\delta$ (Corollary~\ref{cor:uc-fixed-design-enclosure}). That sampling claim does not apply to finite human-cohort subsampling, and the certificate establishes no adaptive-query or comparison-budget savings. Appendix~\ref{app:posterior-certificates-theory} gives complete proofs and the separate mixture-training identity.

\subsection{What a soft UC score means}\label{sec:score-meaning}
A bounded smooth score is neither a calibrated membership probability nor a certificate of correct ordering. The following constructions separate these interpretations; their full proofs are in Appendix~\ref{app:proofs}.

\begin{lemma}[Neutral-input UC reference]\label{lem:neutral}
For $n\ge3$, if $P_{ab}=1/2$ for all distinct $a,b$, then the canonical UC score is $u_a=5/8$ for every alternative and every positive $\tau,\gamma_c$.
\end{lemma}
Equal neutral witness scores give $v=1/4$ and cover strength $3/8$, regardless of temperature. Thus an unobserved graph under neutral imputation can score above one half without directional evidence.

\begin{lemma}[Distinct plug-in and posterior summaries]\label{lem:distinct-summaries}
Let $\Pi$ be uniform on the eight strict three-vertex tournaments, encoded by $0/1$ edge probabilities and diagonal $1/2$. At every positive canonical temperature, $u_a(\mathbb E_\Pi P)=5/8$ but $\Pr_\Pi\{a\in\UC(T(P))\}=1/2$ for each vertex.
\end{lemma}
Six transitive tournaments and two cycles give four UC inclusions per vertex. Plug-in and posterior summaries are intrinsically different.

\begin{proposition}[A finite-temperature ordering counterexample]\label{prop:ordering-counterexample}
There is a reciprocal strict-margin probability matrix with known entries whose hard TC and UC are the singleton $\{a\}$, but whose canonical UC scores at $\tau=\gamma_c=0.035$ satisfy $u(b)>u(a)$. Consequently no upper-threshold rule or descending-score prefix recovers that singleton at these temperatures.
\end{proposition}
The construction uses $P_{ab}=P_{ac}=0.501$, $P_{bc}=0.9$: exact probabilities do not eliminate smoothing bias.

\section{Experiments}\label{sec:experiments}
The matched-epoch study, equal-budget soft-UC comparison, and prospective hard-UC confirmation retain distinct frozen primary families. Both human-profile learning studies are exploratory; the posterior-control extension uses a previously examined cohort. Appendices~\ref{app:ibm-gpu}--\ref{app:real-learning} and~\ref{app:posterior-native} give protocols, controls, and audits. Historical studies, oracle cardinalities, and latent probabilities remain separate diagnostics.

For a reported set $S$ and reference core $C$, F1 is $2|S\cap C|/(|S|+|C|)$, with value one when both are empty. Exact recovery is $\mathbf1\{S=C\}$; size is $|S|$. We average these per-graph quantities rather than pooling vertex decisions across graphs. Core classes are singleton, selective non-singleton ($1<|C|<n$), and full ($|C|=n$). Both overlap and exact recovery matter: a large selection can achieve useful F1 while failing to identify the core.

\begin{table}[t]\centering\scriptsize
{\setlength{\tabcolsep}{1.2pt}
\begin{tabular*}{\linewidth}{@{\extracolsep{\fill}}lrrrrlrr@{}}
\toprule
Core & $n$ & STE & Aux. & $\Delta$ F1 & 95\% CI & $p_{\rm Holm}$ & Pos.\\
\midrule
UC & 24 & .6078 & .5761 & $+.0317$ & $[.0182,.0453]$ & .0015 & 12/12\\
UC & 48 & .4931 & .4638 & $+.0293$ & $[.0102,.0484]$ & .0127 & 11/12\\
TC & 24 & .5405 & .5111 & $+.0294$ & $[.0109,.0480]$ & .0127 & 10/12\\
\bottomrule
\end{tabular*}
}

\caption{Frozen primary contrasts under the same structural readout and development-selected cutoff. $\Delta$ is STE minus auxiliary F1; CI uses $t_{11}$ over twelve collection means, each averaging graphs and three initializations. Two-sided exhaustive sign-flip tests receive Holm adjustment over these three endpoints. Pos. counts positive collection effects.}
\label{tab:gpu-primary}\end{table}

\subsection{Separate UC and TC structural training}\label{sec:learning}
Twelve independent collections contain 512 training and 128 development graphs at $n=12$, and 256 tests at each $n\in\{12,24,48,64\}$. Three paired initializations share each collection. The 12,288 distinct tests cross ordered, planted cyclic, TC/UC-separated, context-mixture, and random families. Heterogeneous comparison counts, planned missingness 0.25, and tie probability 0.2 produce the observations; latent conditional-decisive majority cores supply labels. Singleton, selective, and full cores are reported separately.

UC and TC train separately with a shared $16\to32\to32$ reciprocal pair encoder taking counts, ties, and masks. Auxiliary and STE objectives receive identical core labels; ordinary and relational heads add active parameters. Candidates receive 160 paired Adam epochs. Development selects loss weight, checkpoint, and readout cutoff. Full TC uses $K=n-1$ with checkpointing; temperatures are 0.035. Equal epochs do not match compute: STE TC/UC steps cost 3.66$\times$/4.10$\times$ auxiliary steps at $n=12$ (Appendix~\ref{app:ibm-gpu}).

The matched-readout gains are 0.0317/0.0293 for UC at $n=24/48$ and 0.0294 for TC at $n=24$ (Table~\ref{tab:gpu-primary}). Auxiliary heads train for their own membership predictions, whereas STE trains through the structural decoder: this comparison measures decoder alignment, not complete-system superiority. Intervals use twelve independent collection effects; exact sign flipping assumes null sign-reversal exchangeability. Graphs and initializations are not additional independent units.

\textbf{Native readouts.} Auxiliary heads have higher TC F1. STE exceeds the relational head by 0.0210/0.0262 UC F1 at $n=24/48$, but both heads are stronger at the training size $n=12$ (Table~\ref{tab:gpu-systems}). These descriptive comparisons use each readout's development selection; active capacities and compute differ.

\textbf{Selective cores and cardinality.} UC matched-readout gains are 0.0376/0.0370 on selective non-singletons at $n=24/48$, versus 0.0087/0.0007 on full cores. TC's $n=24$ gain is concentrated on singletons (0.0693 versus 0.0062 on selective cores). Selective UC exact recovery is only 5.76\%/1.79\%; selections average 13.16/29.34 versus reference sizes 9.66/13.13 (Table~\ref{tab:selective-recovery}). These descriptive strata expose over-inclusion when a cutoff selected at $n=12$ transfers to larger panels.

\subsection{Equal-budget UC complete systems}\label{sec:equal-compute}
The soft-UC follow-up uses twelve fresh collections, three paired initializations, and the unchanged generators/operator. Each system receives 220 synchronized seconds, including training, native-readout development selection, and serialization. Development selective UC F1 selects loss weight, one of five checkpoints, and a cutoff before test generation. Setup/test evaluation are excluded; every candidate and system passes the fixed $\pm2\%$ timing gate (Appendix~\ref{app:uc-equal-compute}).

\begin{table}[t]\centering\scriptsize
\begin{tabular*}{\linewidth}{@{\extracolsep{\fill}}lrr}
\toprule
Objective / readout & $n=24$ & $n=48$ \\
\midrule
Pair / structural & $0.6245\pm0.0043$ & $0.4018\pm0.0061$ \\
Aux. / head & $0.6414\pm0.0085$ & $0.4280\pm0.0068$ \\
Rel. aux. / head & $0.6416\pm0.0061$ & $0.4263\pm0.0051$ \\
STE / structural & $0.6281\pm0.0074$ & $0.4076\pm0.0046$ \\
\bottomrule
\end{tabular*}

\caption{Equal-budget native systems: selective non-singleton UC F1, mean $\pm$ SEM across twelve collections, each averaging three initializations. Only $n=24$ is primary; $n=48$ is descriptive. SEM is not a confidence interval.}
\label{tab:uc-equalcompute-main}\end{table}

At the primary $n=24$ endpoint, STE minus ordinary/relational direct heads is $-0.01329/-0.01345$, with individual 95\% $t_{11}$ intervals $[-0.03286,0.00628]/[-0.03300,0.00611]$. Both exact sign-flip tests have Holm-adjusted $p=0.3193$ in this separate two-test family. The study does not demonstrate an equal-budget advantage; it establishes neither inferiority nor equivalence. The studies differ in compute, readout, and endpoint, precluding attribution to compute alone.

The original checkpoints also yield a post hoc hard-decoding benefit and lower pair MSE (Appendix~\ref{app:uc-equal-compute}). Those descriptive outputs retain their original status and motivated the separate test below.

\subsection{Prospective hard-UC confirmation}\label{sec:hard-confirm}
A fresh frozen protocol compares six separately trained systems with decoder-aligned development selection (Appendix~\ref{app:hard-confirm}). Each receives 220 charged seconds per paired initialization; every candidate and complete system passes the unchanged $\pm2\%$ gate. STE still trains through soft UC, while checkpoint/loss-weight selection and deployment use hard UC of $\widehat P_{ab}>1/2$, without a membership cutoff. Native ordinary/relational heads and separately trained hard-decoded controls receive their own selection.

\begin{table}[t]\centering\scriptsize\setlength{\tabcolsep}{2pt}
\begin{tabular*}{\linewidth}{@{\extracolsep{\fill}}lrrr@{}}\toprule
System / readout & F1, $n=24$ & F1, $n=48$ & \shortstack{Exact (\%)\\$n=24$}\\\midrule
Pair / hard & 0.5865 (0.0078) & 0.4012 (0.0056) & 1.45\\
Ordinary / native & 0.6327 (0.0045) & 0.4309 (0.0042) & 2.02\\
Relational / native & 0.6292 (0.0051) & 0.4318 (0.0055) & 2.53\\
STE / hard & 0.6676 (0.0043) & 0.5159 (0.0030) & 1.16\\
Ordinary / hard & 0.5972 (0.0082) & 0.4121 (0.0060) & 1.41\\
Relational / hard & 0.6237 (0.0031) & 0.4324 (0.0053) & 1.44\\
\bottomrule\end{tabular*}

\caption{Prospective equal-budget UC systems on selective cores. F1 is mean (SEM) over twelve collections, each averaging three initializations; exact recovery is a percentage. Only STE-hard versus native heads at $n=24$ belongs to the separate two-test primary family; other entries are descriptive.}
\label{tab:hard-confirm-main}\end{table}

At primary $n=24$, STE-hard F1 is 0.6676 versus 0.6327/0.6292 for ordinary/relational native heads. Paired differences are 0.03493/0.03843, with individual 95\% $t_{11}$ intervals $[0.02383,0.04602]/[0.02763,0.04922]$; both exhaustive sign-flip tests have Holm-adjusted $p=0.00098$. All twelve collection effects are positive. This supports the tested hard-UC pipeline without replacing the soft-set null. Exact recovery is only 1.16\%, below the heads' 2.02/2.53\%; at $n=48$, selections average 25.99 versus 13.13 reference members. Overlap gains do not establish exact or size-faithful recovery.

\subsection{Count inference and automatic decisions}\label{sec:counts}
Count controls share observations and development calibration without test core size. Posterior inference uses 512 Beta$(1/2,1/2)$ edge-posterior draws; GFM optimizes their expected F1 \citep{waegeman2014fmeasure}. Appendix~\ref{app:ibm-gpu} gives all decision controls.

UC STE-LSE/posterior/Rank Centrality F1 is 0.5459/0.4583/0.5386 at $n=24$ and 0.3930/0.3569/0.4424 at $n=64$. TC at $n=24$ favors count-only Rank Centrality (0.5494) over both count STE (0.4669) and learned STE (0.5405). Rank Centrality uses the same observations and development labels for its cutoff, without supervised model fitting. These descriptive comparisons reuse the learning collections. Latent-$P$ and oracle-size outputs are diagnostic; full baselines and stopping checks are in Appendix~\ref{app:ibm-gpu}.

\subsection{Strict selective recorded-profile reconstruction}\label{sec:profiles}
We reparse 462 original PrefLib ordinal profiles from eleven source series, retaining record multiplicities, tied ranks, and unranked alternatives \citep{mattei2013preflib}. Eligibility is determined from raw observations before method performance. The primary UC cohort requires every decisive pair observed, no exact empirical majority ties, $1<|\UC|<n$, and no exact within-source count duplicate. Twelve profiles from four sources qualify, spanning 4--24 alternatives and 51--28,660 ordinal records. Eleven have UC size three and one has size twelve, limiting cardinality diversity.

\textbf{Source identity matters.} Only three qualifying profiles are human elections: two ERS elections and one Oakland election \citep{preflibers,oneill2013openstv}. The other nine are two Multilaps and seven Formula 1 profiles whose rankings encode athletes' lap times \citep{boehmer2022collecting}. PrefLib's generic ``voters'' field does not imply human judgments. This is a mixed recorded-ranking benchmark, not a twelve-profile human-preference validation (Appendix~\ref{app:gpu-real-inventory}).

The reference is the complete finite-profile empirical core. Ten whole-record resampling streams use nested 10\%, 25\%, and 50\% fractions. Calibration excludes the test source and equally weights the others. The primary setting is UC at 25\% records without extra missingness; repetitions average within profiles, profiles within sources, and the four sources equally. This tests count reconstruction, not neural transfer or population truth (Appendix~\ref{app:ibm-gpu}).

\begin{table}[t]\centering\small
\begin{tabular}{@{}lrr@{}}
\toprule
Method & UC F1 & Exact (\%)\\
\midrule
STE-LSE & 0.6491 & 9.64\\
Posterior core & 0.7824 & 32.32\\
Posterior GFM & 0.6920 & 11.43\\
Smooth Copeland & 0.6820 & 22.50\\
BTL & 0.6561 & 6.43\\
Copeland & 0.6548 & 27.50\\
Rank Centrality & 0.5689 & 5.71\\
Native hard UC & 0.5535 & 11.07\\
Select all & 0.5163 & 0.00\\
Select none & 0.0000 & 0.00\\
\bottomrule
\end{tabular}

\caption{Strict selective recorded UC reconstruction at the frozen 25\%-record setting: three election and nine sports profiles, four equally weighted sources. Score methods use source-excluded absolute calibration; GFM and hard/all/none controls use native decisions. Exact is complete-set recovery.}
\label{tab:gpu-real}\end{table}

Posterior UC has F1 0.7824 versus STE-LSE 0.6491 and exact recovery 32.32\% versus 9.64\% (Table~\ref{tab:gpu-real}). All four source effects favor posterior inference. The descriptive STE-minus-posterior $t_3$ interval is $[-0.3454,0.0788]$; shared calibration sources preclude a source-population coverage guarantee. Secondary TC and partial/tied references remain separate in Appendix~\ref{app:ibm-gpu}.

The edge-factorized posterior is a working model: ordinal records create dependent pairs, laps may be serially dependent, and its latent random core differs from the finite reference. Subsampling measures reconstruction of that record set, not future-race prediction. Three election profiles are insufficient to establish broad human-preference utility.

\textbf{Exploratory human-profile learning.} A separate study fits the pair model on 36 human ordinal profiles from six source series. Each whole-source fold uses four sources for training, one for development, and one for testing (Appendix~\ref{app:real-learning}). References are full-record empirical UC; only three are selective and 33 are singletons. At 10\% sampled records, equal-source macro F1 is 0.8745 for STE soft selection and 0.8808 for STE hard UC, below count-hard 0.9002 and posterior-half 0.9016. STE soft exact recovery is 80.65\% versus 85.03/83.23\% for those controls. These are exploratory finite-record reconstruction results, with no p-values or population/general-superiority claims; the cohort and prior estimator outcomes were already inspected.

\subsection{Native posterior and count controls on human profiles}\label{sec:posterior-native}
A separate extension on the same 36-profile cohort tests learned predictive distributions against strong count-based controls (Appendix~\ref{app:posterior-native}). It preserves the earlier human-profile results. Six whole-source folds hold out one source for testing and one for development; eight nested whole-record subsamples and three paired initializations retain actual rank multiplicities. All methods receive observed decisive/tied counts at test time. Jeffreys and fitted symmetric empirical-Bayes edge posteriors, learned independent orientations, and a learned two-component tournament distribution use native GFM expected-F1 decisions. Ordinary/relational heads use their own development-selected cutoffs; raw-count UC uses its native covering rule. The fitted systems share a 120-second upper allocation, with actual costs reported separately.

\begin{table}[t]\centering\scriptsize\setlength{\tabcolsep}{3pt}
\begin{tabular*}{\linewidth}{@{\extracolsep{\fill}}lrrr@{}}\toprule
System & F1 & Exact & UC Brier\\\midrule
Jeffreys--GFM & \textbf{0.8742} & 0.7656 & 0.0437\\
Empirical Bayes--GFM & 0.8705 & 0.7522 & 0.0447\\
Learned edges--GFM & 0.8704 & 0.7614 & \textbf{0.0400}\\
Learned mixture--GFM & 0.8366 & 0.7301 & 0.0486\\
Raw-count UC & 0.8481 & \textbf{0.7950} & ---\\
Relational head & 0.7944 & 0.6161 & 0.0681\\
Ordinary head & 0.7340 & 0.5730 & 0.0708\\\midrule
Count soft plug-in (secondary) & 0.8462 & 0.7679 & 0.0507\\\bottomrule
\end{tabular*}
\caption{Exploratory native decisions at 10\% records: mean within profile, then source, then six equally weighted sources. Exact is complete-set recovery; lower UC Brier measures agreement of membership scores with the finite reference. The cohort has 33 singleton and three selective targets. No significance or population-calibration claim is made.}
\label{tab:posterior-native-main}\end{table}

Jeffreys--GFM has the highest native F1, and raw-count UC the highest exact recovery (Table~\ref{tab:posterior-native-main}). Learned edges nearly match Jeffreys F1 and have lower recorded UC Brier; the mixture does not improve either endpoint. On the three selective targets, each of size three, Jeffreys F1 is 0.8411 versus 0.6749/0.6542 for learned edges/mixture. Their selections average 2.72 versus 1.70/1.58 members. No native method exactly recovers the eleven-alternative ERS profile in any saved 10\% input. These two-source selective means are descriptive, with no p-values or general superiority claim. An independent raw-output review verifies native decoding and selected-checkpoint replay, and excludes an invalid mixture hard-decoding diagnostic caused by floating-point self-loops (Appendix~\ref{app:posterior-native}).

\subsection{TC fidelity, gradients, and measured cost}\label{sec:gpu-diagnostics}
GPU diagnostics use three graphs per family/size at $n=8,16,32,64,96$, three temperatures, and four depths including $K=n-1$. At the default temperature and full depth, the \emph{long-path family} has mean absolute TC score error rising from 0.218 at $n=8$ to 0.847 at $n=96$; ordered and planted errors remain much smaller (Figure~\ref{fig:gpu-tc-error}). Score error differs from selected-set F1: selecting all can recover a full core.

\begin{figure}[t]\centering
\includegraphics[width=\columnwidth]{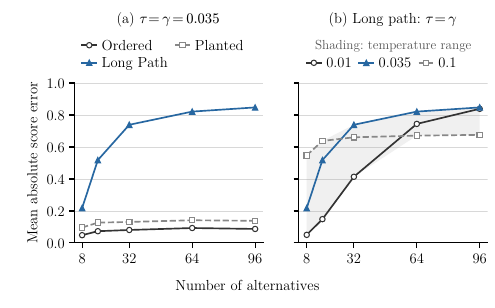}
\caption{Full-depth TC score error ($K=n-1$). (a) Families at $\tau=\gamma=0.035$. (b) Long-path sensitivity with $\tau=\gamma$; shading spans means at the three tested temperatures, not a confidence band. Each point averages vertex errors over three saved graphs. These are operator diagnostics, not trained-model transfer tests.}
\label{fig:gpu-tc-error}\end{figure}

Diagnostic gradients differentiate clipped BCE with respect to probability logits. They are finite, but low-temperature cells can have negligible or zero entries; no lower bound on useful training signal follows. All 540 plain/checkpoint pairs have identical outputs and gradients. At full depth $n=96$, mean peak allocated memory and forward-plus-backward time are 414.9 MiB and 42.7 ms for plain autograd, versus 94.5 MiB and 101.0 ms with checkpointing. Recomputing saves memory at a runtime cost.

L40S batch-four step benchmarks include forward, backward, clipping, and Adam. STE/auxiliary TC and UC ratios are 3.66$\times$/4.10$\times$ at $n=12$, increasing to 23.58$\times$/26.36$\times$ at $n=96$. These are step costs, not larger-size training: synthetic GPU fits train at $n=12$ and transfer tests stop at $n=64$.

\section{Interpretation and limitations}\label{sec:discussion}
Relation learning, probabilistic scores, and set prediction need not improve together (Table~\ref{tab:claim-evidence}). Prospective hard-UC selection supports a synthetic F1 gain, while the soft-set comparison does not demonstrate an advantage. The human-profile studies favor strong count controls; lower membership Brier for learned edges does not establish better F1, exact recovery, or calibration. LLM-judge transfer remains untested.

\begin{table}[t]\centering\scriptsize
\begin{tabular*}{\linewidth}{@{\extracolsep{\fill}}p{0.25\linewidth}p{0.37\linewidth}p{0.25\linewidth}@{}}\toprule
{\raggedright Claim\par} & {\raggedright Evidence\par} & {\raggedright Verdict\par}\\\midrule
{\raggedright Structural training helps\par} & {\raggedright Matched-readout contrasts (Table~\ref{tab:gpu-primary})\par} & {\raggedright Conditional support\par}\\
{\raggedright Soft sets beat heads at equal budget\par} & {\raggedright Selective primary contrasts (Table~\ref{tab:uc-equalcompute-main})\par} & {\raggedright Not demonstrated\par}\\
{\raggedright Hard UC beats native heads\par} & {\raggedright Frozen selective comparison (Table~\ref{tab:hard-confirm-main})\par} & {\raggedright Protocol-specific support\par}\\
{\raggedright STE beats recorded-data controls\par} & {\raggedright Count reconstruction and human-profile learning\par} & {\raggedright Not supported\par}\\
{\raggedright Learned dependence improves native F1\par} & {\raggedright Posterior/control comparison (Table~\ref{tab:posterior-native-main})\par} & {\raggedright Not supported\par}\\\bottomrule
\end{tabular*}
\caption{Claims, evidence, and scope. Descriptive diagnostics do not replace the primary comparisons.}
\label{tab:claim-evidence}\end{table}

Propositions~\ref{prop:lipschitz} and~\ref{prop:selection-stability} preserve a specified top-$k$ set when its boundary gap exceeds $2Le$, for probability error $e$ and $L=1/(4\tau)$ (TC) or $3/(4\tau)$ (UC). They neither choose $k$ nor prove correctness. The known-$P$ UC ordering failure (Proposition~\ref{prop:ordering-counterexample}) and TC long-path errors leave temperature scheduling unresolved.

TC costs $O(Kn^3)$ ($O(n^4)$ at full depth); UC costs $O(n^3)$. Streamed forward workspace is $O(n^2)$, but GPU TC materializes cubic candidates. Plain autograd retains $O(Kn^3)$ memory; checkpointing uses $O(Kn^2)$ boundaries plus $O(n^3)$ workspace. UC autograd uses $O(n^3)$, suppressing batch factors.

Missing edges remain unidentified without modeling assumptions. Posterior inference depends on its edge model; unobserved judge bias remains unresolved. One GPU class, architecture, and generator mixture limit transfer. Selected-weight and development-choice replay checks computational consistency; optimizer trajectories are not independently retrained (Appendices~\ref{app:uc-equal-compute}--\ref{app:real-learning} and~\ref{app:posterior-native}). A passing internal audit can reproduce an invalid diagnostic, so independent structural checks remain necessary.

\section{Conclusion}\label{sec:conclusion}
STE makes TC/UC structure differentiable and separates smooth scores, modeled uncertainty, and hard set decisions. Matched-epoch training improves common-readout F1; equal-budget soft UC selection does not demonstrate an advantage. A separate frozen hard-UC comparison supports selective synthetic overlap gains over native heads with low exact recovery. Exploratory human-profile learning, including native expected-F1 decoding of a learned dependent distribution, does not outperform strong count/posterior controls on its main F1 endpoint. Completion-based certificates expose what unresolved edges still permit, under explicit assumptions. Exact recovery, broad human/LLM transfer, and scalable faithful TC inference remain open.

\balance
\clearpage
\nobalance
\subsubsection*{Code and data availability}
Implementation, configurations, provider input records, result summaries, and verification code are available at \url{https://github.com/alqithami/STE}. The repository distinguishes frozen study implementations from corrected code for future runs. Large checkpoint and prediction archives are separately preserved; their integrity checks and the scope of replay verification accompany each study. Provider terms apply to reused data. The new human-profile extension reuses the previously inspected cohort and is exploratory.

\subsubsection*{AI use statement}
Generative AI assisted with research design, numerical and analysis code, implementation checks, mathematical formulation and proofs, experiment execution through computing tools, result interpretation, literature searches, figures, and manuscript writing. Executable checks and saved-state replays verify specified calculations; they are neither proof-assistant certification nor replication by an independent researcher. The author is responsible for the content and final claims.

\bibliography{references}

\clearpage
\appendix
\onecolumn
\raggedbottom
\section*{Soft Tournament Equilibrium: Supplementary Materials}
\noindent\textbf{Reading guide.} Appendices~\ref{app:proofs}--\ref{app:posterior-native} contain the complete STE proofs and separately scoped experimental studies. Appendix~\ref{app:posterior-certificates-theory} develops supporting completion and posterior-decision results. Earlier evidence is summarized in Appendix~\ref{app:implementation-provenance}; it is not pooled with the primary endpoints. The two human-profile studies share their cohort but use separately frozen subsamples and protocols.
\begin{center}\small
\begin{tabular}{@{}lp{0.74\linewidth}@{}}\toprule
Appendix & Role in this paper\\\midrule
A & Complete STE proofs, recovery conditions, and score counterexamples.\\
B & Matched-epoch training, mixed recorded rankings, and TC diagnostics.\\
C & Equal-budget soft UC: primary null, post hoc hard diagnostics, replay.\\
D & Separate frozen hard-UC confirmation: native heads, timing gate, replay.\\
E & Exploratory human-profile learning: source folds and finite references.\\
F & Count/posterior and learned-distribution comparison, native decisions, and diagnostic correction.\\
G & Supporting completion certificates, posterior UC identities, and size-aware decisions with proofs.\\
H & Implementation conventions, conditional targets, and historical provenance.\\
I & Reproduction map and notation.\\\bottomrule
\end{tabular}
\end{center}

\section{Proofs and operator properties}\label{app:proofs}
All norms on matrices or vectors below are entrywise sup norms. We assume finite $n\ge2$ and positive temperatures. The hard target is a strict tournament unless an incomplete-graph extension is explicitly stated. Finite-temperature bounds are deterministic; sampling statements are conditional on their observation model. The arguments below prove the main-text results and the additional appendix guarantees. Numerical checks complement these proofs but are not premises of them.

\subsection{Uniform approximation and continuity of soft extrema}
\begin{lemma}\label{lem:extrema}
For $z\in\mathbb R^r$, $r\ge1$, and $\gamma>0$,
\[
\max_i z_i-\gamma\log r\le\smax_\gamma(z)\le\max_i z_i,
\qquad
\min_i z_i\le\smin_\gamma(z)\le\min_i z_i+\gamma\log r.
\]
Both operators are nonexpansive in the sup norm. In particular, for $z,z'\in\mathbb R^r$,
$|\smax_\gamma(z)-\smax_\gamma(z')|\le\|z-z'\|_\infty$, and likewise for $\smin$.
\end{lemma}
\begin{proof}
Write $M=\max_i z_i$. The term attaining $M$ and the bound on every term give
\[
\frac{e^{M/\gamma}}r\le\frac1r\sum_{i=1}^r e^{z_i/\gamma}\le e^{M/\gamma}.
\]
Taking $\gamma\log$ proves the maximum bound. Applying it to $-z$ and negating proves the minimum bound. Both operators are coordinatewise nondecreasing and satisfy $F(z+c\mathbf1)=F(z)+c$. If $e=\|z-z'\|_\infty$, then $z'-e\mathbf1\le z\le z'+e\mathbf1$ entrywise, so monotonicity and translation equivariance give $F(z')-e\le F(z)\le F(z')+e$. This proves nonexpansiveness for either operator.

For completeness, their coordinate derivatives are
\[
\frac{\partial\smax_\gamma(z)}{\partial z_i}
=\frac{e^{z_i/\gamma}}{\sum_j e^{z_j/\gamma}},\qquad
\frac{\partial\smin_\gamma(z)}{\partial z_i}
=\frac{e^{-z_i/\gamma}}{\sum_j e^{-z_j/\gamma}}.
\]
Each derivative vector has nonnegative entries summing to one; normalization by $r$ changes the function value but not these derivatives. Also, monotonicity gives $\min z=F((\min z)\mathbf1)\le F(z)\le F((\max z)\mathbf1)=\max z$. Thus both operators preserve $[0,1]$. For $r=1$ they are exactly the identity and all smoothing allowances vanish.
\end{proof}
The approximation bounds are uniform in $z$, so they also apply when inputs vary with temperature. This permits induction through the nested operators; convergence only for a fixed binary input would not by itself justify such a composition.

\subsection{Proof of Theorem~\ref{thm:finite}: directional errors}
\begin{proof}[Proof of Theorem~\ref{thm:finite}]
\emph{Hard paths.} Define $H^{(1)}=M$ and
\[
H^{(k)}_{ab}=\max_{c\in\agents}\min\{H^{(k-1)}_{ac},M_{cb}\}.
\]
Since all entries are binary, $H^{(1)}_{ab}=1$ exactly when there is an edge $a\to b$. Inductively, $H^{(k)}_{ab}=1$ exactly when some $c$ has a length-$(k-1)$ walk from $a$ to $c$ and an edge $c\to b$, which is equivalent to a length-$k$ walk from $a$ to $b$. Conversely, every length-$k$ walk has such a penultimate vertex. Intermediate diagonal entries may record cycles; self-reachability in the final TC definition is understood separately.

Put $H^{\le K}_{ab}=\max_{1\le k\le K}H^{(k)}_{ab}$. For distinct $a,b$, any walk from $a$ to $b$ can be shortened to a simple path: whenever a vertex repeats, delete the segment between its two occurrences. Repeating this deletion leaves at most $n$ distinct vertices, hence at most $n-1$ edges. Therefore, when $K\ge n-1$, $H^{\le K}_{ab}$ equals hard reachability for all distinct pairs and
\[
\min_{b\ne a}H^{\le K}_{ab}=\mathbf1_{a\in\TC(M)}.
\]

\emph{TC recurrence.} Let $Q^{(1)}=D$ and use the canonical smooth product. We prove entrywise
\[
H^{(k)}-\ell_k\mathbf1\le Q^{(k)}\le H^{(k)}+r_k\mathbf1,
\quad \ell_k=\eta+(k-1)\gamma\log n,
\quad r_k=\eta+(k-1)\gamma\log2.
\]
The case $k=1$ is $\|D-M\|_\infty\le\eta$. Suppose it holds for $k-1$. For fixed $a,c,b$, both inputs $Q^{(k-1)}_{ac}$ and $D_{cb}$ differ from their hard counterparts by quantities in $[-\ell_{k-1},r_{k-1}]$, because $\ell_{k-1},r_{k-1}\ge\eta$. Monotonicity and translation equivariance of the ordinary minimum imply
\[
\min\{H^{(k-1)}_{ac},M_{cb}\}-\ell_{k-1}
\le \min\{Q^{(k-1)}_{ac},D_{cb}\}
\le \min\{H^{(k-1)}_{ac},M_{cb}\}+r_{k-1}.
\]
Replacing this two-input minimum by $\smin_\gamma$ adds an error in $[0,\gamma\log2]$. An ordinary maximum over $c$ preserves the resulting common interval. Replacing that maximum by $\smax_\gamma$ over $n$ entries adds an error in $[-\gamma\log n,0]$. Consequently the new interval is $[-\ell_{k-1}-\gamma\log n,r_{k-1}+\gamma\log2]$, establishing the induction.

For all $k\le K$, the intervals lie inside $[-\ell_K,r_K]$. An ordinary maximum over lengths therefore preserves this interval relative to $H^{\le K}$; replacing it by the normalized maximum adds only the negative allowance $\gamma\log K$. Taking an ordinary minimum over opponents again preserves the interval, and its normalized replacement adds only the positive allowance $\gamma\log(n-1)$. It follows that
\[
-\eta-\gamma\{(K-1)\log n+\log K\}
\le t_a-\mathbf1_{a\in\TC(M)}
\le \eta+\gamma\{(K-1)\log2+\log(n-1)\}.
\]
The negative-side allowance minus the positive-side allowance is
$\gamma\{(K-1)\log(n/2)+\log[K/(n-1)]\}\ge0$ for $n\ge2$, $K\ge n-1$. The larger allowance thus bounds the absolute error and proves Eq.~\eqref{eq:tc-bound}. In particular, $n=2,K=1$ has no aggregation error: the length and opponent vectors each have one entry.

\emph{UC edge error.} For a zero-diagonal matrix $X\in[0,1]^{n\times n}$, let
\[
W_X(c,a)=\max_{b\notin\{a,c\}}X_{ab}(1-X_{cb}),\quad
C_X(c,a)=X_{ca}(1-W_X(c,a)),\quad
U_X(a)=1-\max_{c\ne a}C_X(c,a),
\]
with $W_X=0$ when $n=2$. On the binary matrix $M$, a witness equals one precisely when $a\to b$ but $c\not\to b$. Thus $C_M(c,a)$ equals one exactly when $c$ covers $a$, and $U_M=\mathbf1_{\UC(M)}$.
For scalars in $[0,1]$,
\[
|x(1-y)-x'(1-y')|
\le |x-x'|(1-y)+x'|y-y'|\le |x-x'|+|y-y'|.
\]
Each witness changes by at most $2\eta$ from $M$ to $D$. Ordinary maxima are nonexpansive, hence $|W_D-W_M|\le2\eta$; the empty witnesses have zero difference. Applying the same product inequality to $C_D,C_M$ gives $|C_D-C_M|\le\eta+2\eta=3\eta$. The final ordinary maximum yields $\|U_D-U_M\|_\infty\le3\eta$.

\emph{UC smoothing error.} Now fix $D$. Set $L_c=\gamma_c\log\max\{1,n-2\}$ and $L_f=\gamma_c\log(n-1)$. Lemma~\ref{lem:extrema} gives $W_D-L_c\le v\le W_D$ for nonempty witnesses. For $n=2$, both are zero and $L_c=0$, so the same interval holds. Since $0\le D_{ca}\le1$,
\[
C_D(c,a)\le D_{ca}(1-v(c,a))\le C_D(c,a)+L_c.
\]
Writing $C_*(a)=\max_{c\ne a}C_D(c,a)$, an ordinary maximum of these soft cover values lies in $[C_*,C_*+L_c]$. Its normalized replacement lies in $[C_*-L_f,C_*+L_c]$. After subtraction from one, $u(a)-U_D(a)\in[-L_c,L_f]$. Because $L_c\le L_f$, adding the $3\eta$ edge error proves Eq.~\eqref{eq:uc-bound}.

Every intermediate soft extremum remains in $[0,1]$ by Lemma~\ref{lem:extrema}; witness products, covers, and their complements do also. Thus each final score and each hard indicator belongs to $[0,1]$, permitting the cap at one. Finally, suppose a score has uniform error $E<\min(h,1-h)$ from a binary target. A zero target has score at most $E<h$, and a one target has score at least $1-E>h$. Hence the rule $S_h=\{a:z_a\ge h\}$ recovers the target exactly. The minimum of $h$ and $1-h$ is maximized at $h=1/2$.
\end{proof}

\subsection{Estimation, hard limits, and Condorcet inclusion}
\begin{lemma}[Edge smoothing under a strict margin]\label{lem:edge}
If $\Delta=\min_{a\ne b}|P_{ab}-1/2|>0$ and $\|\widehat P-P\|_\infty\le\epsilon<\Delta$, then $T(\widehat P)=T(P)$ and, for $D_\tau$ built from $\widehat P$,
$\|D_\tau-M\|_\infty\le e^{-(\Delta-\epsilon)/\tau}$.
\end{lemma}
\begin{proof}
If $P_{ab}>1/2$, then $\widehat P_{ab}-1/2\ge P_{ab}-1/2-\epsilon\ge\Delta-\epsilon>0$. The empirical orientation is therefore unchanged, and
\[
1-D_\tau(a,b)=\frac1{1+e^{(\widehat P_{ab}-1/2)/\tau}}
\le e^{-(\widehat P_{ab}-1/2)/\tau}\le e^{-(\Delta-\epsilon)/\tau}.
\]
If $P_{ab}<1/2$, then $1/2-\widehat P_{ab}\ge\Delta-\epsilon>0$, so the orientation again agrees and
$D_\tau(a,b)=1/(1+e^{(1/2-\widehat P_{ab})/\tau})\le e^{-(\Delta-\epsilon)/\tau}$.
The diagonal entries of $D_\tau$ and $M$ are both zero. Taking the maximum over entries proves the claimed norm bound, and agreement of every off-diagonal sign proves graph equality.
\end{proof}

\begin{corollary}[Hard-limit consistency and learned-estimator transfer]\label{cor:limit}
For fixed finite $n$ and $K\ge n-1$, strict-margin $P$, and $\widehat P\to P$ uniformly, STE-LSE scores converge to $\mathbf1_{\TC(T(P))}$ and $\mathbf1_{\UC(T(P))}$ as $\tau,\gamma,\gamma_c\to0^+$. Any estimator satisfying $\|\widehat P-P\|_\infty<\Delta$ recovers the same hard graph and hard cores, irrespective of its parameterization.
\end{corollary}
\begin{proof}
Consider any sequence of estimates and positive temperatures satisfying the stated limits. Put $\epsilon_j=\|\widehat P_j-P\|_\infty$. Uniform convergence gives $\epsilon_j<\Delta/2$ for all sufficiently large $j$. Lemma~\ref{lem:edge} then gives
$\eta_j\le e^{-(\Delta-\epsilon_j)/\tau_j}\le e^{-\Delta/(2\tau_j)}\to0$.
Since $n,K$ are fixed, both $\gamma_j\{(K-1)\log n+\log K\}$ and $\gamma_{c,j}\log(n-1)$ tend to zero. Theorem~\ref{thm:finite} therefore gives uniform convergence of both score vectors to the indicated hard memberships. Separately, for any single estimate with error less than $\Delta$, Lemma~\ref{lem:edge} establishes graph equality. TC and UC are deterministic functions of that graph, so their hard memberships agree regardless of the estimator's parameterization.
\end{proof}
The fixed-$n,K$ qualification matters: if the number of agents or path depth grows with decreasing temperature, the terms involving $\gamma K\log n$ must also vanish. A full-graph recovery theorem does not identify never-observed edges without further assumptions.

\begin{corollary}[Condorcet inclusion]\label{cor:condorcet}
Under Corollary~\ref{cor:limit}, if $a^\star$ defeats every other alternative in $T(P)$, then both limiting scores are one at $a^\star$ and zero elsewhere.
\end{corollary}
\begin{proof}
The winner reaches every other vertex by a single edge. No edge enters it; hence a directed path from any other vertex to $a^\star$ cannot exist, because its last edge would enter $a^\star$. Consequently TC is $\{a^\star\}$. No vertex covers $a^\star$, since covering requires defeating it. For every opponent $b$, $a^\star$ defeats $b$ and every vertex beaten by $b$, so $a^\star$ covers $b$. Thus UC is also $\{a^\star\}$. Corollary~\ref{cor:limit} transfers these hard memberships to the limiting scores.
\end{proof}

\subsection{Fixed-temperature perturbations and selection stability}\label{app:fixed-temperature}
\begin{proof}[Proof of Proposition~\ref{prop:lipschitz}]
First compare two soft-edge matrices $D,D'$ and put $d=\|D-D'\|_\infty$. More generally, for any matrices $X,X',Y,Y'$ of the same size, set $\delta=\max\{\|X-X'\|_\infty,\|Y-Y'\|_\infty\}$. Each two-entry vector $(X_{ac},Y_{cb})$ differs from $(X'_{ac},Y'_{cb})$ by at most $\delta$. Lemma~\ref{lem:extrema} bounds the change of its soft minimum by $\delta$. Applying the same lemma to the resulting $n$-vectors of candidates proves
\[
\|X\otimes_\gamma Y-X'\otimes_\gamma Y'\|_\infty\le\delta.
\]
The base case has $\|Q^{(1)}-Q'^{(1)}\|_\infty=d$. If the difference at depth $k-1$ is at most $d$, the displayed inequality applied with $Y=D,Y'=D'$ shows that it remains at most $d$ at depth $k$. Hence this holds at every depth up to the same fixed $K$. The soft maximum over the same $K$ path lengths changes by at most $d$, as does the final soft minimum over the same $n-1$ opponents. Therefore $\|t(D)-t(D')\|_\infty\le d$.

For UC, the product inequality in the proof of Theorem~\ref{thm:finite} bounds each witness change by $2d$. The normalized witness maximum is nonexpansive, so $|v-v'|\le2d$. For $n=2$ both violations are identically zero. All violations lie in $[0,1]$, and
\[
|D_{ca}(1-v)-D'_{ca}(1-v')|
\le |D_{ca}-D'_{ca}|+|v-v'|\le3d.
\]
The final normalized maximum is again nonexpansive; thus $\|u(D)-u(D')\|_\infty\le3d$. This conservative bound includes $n=2$.

Lastly, the off-diagonal edge map has derivative
\[
\frac{\partial D_\tau(a,b)}{\partial P_{ab}}
=\frac1\tau\sigma((P_{ab}-1/2)/\tau)
\bigl[1-\sigma((P_{ab}-1/2)/\tau)\bigr]\le\frac1{4\tau},
\]
since $x(1-x)\le1/4$ on $[0,1]$. The mean value theorem gives $d\le\|P-P'\|_\infty/(4\tau)$, while both diagonals remain zero. Substitution proves the two inequalities. No step uses a preserved majority sign or a strict margin.
\end{proof}

\begin{proposition}[Selection stability]\label{prop:selection-stability}
Let $z=f(P)$ and $z'=f(P')$ satisfy $\|z-z'\|_\infty\le Le$. For $1\le k<n$, write $z_{(1)}\ge\cdots\ge z_{(n)}$ and $g_k=z_{(k)}-z_{(k+1)}$. If $g_k>2Le$, the top-$k$ sets of $z$ and $z'$ coincide. For a fixed threshold $h$, if $\min_a|z_a-h|>Le$, then $S_h(z)=S_h(z')$.
\end{proposition}
\begin{proof}
The strict boundary gap makes the initial top-$k$ set $S$ unambiguous even if ties occur within $S$ or its complement. For $i\in S$ and $j\notin S$,
\[
z'_i-z'_j\ge(z_i-Le)-(z_j+Le)\ge g_k-2Le>0.
\]
Every one of the $k$ selected scores remains above every unselected score; hence the selected set is unchanged. For a threshold, if $z_a\ge h$, the assumed strict distance gives $z_a-h>Le$ and consequently $z'_a\ge z_a-Le>h$. If $z_a<h$, then $h-z_a>Le$, so $z'_a\le z_a+Le<h$. This proves threshold-set equality.
\end{proof}
Proposition~\ref{prop:lipschitz} supplies $L=1/(4\tau)$ for TC and $3/(4\tau)$ for UC. These are deterministic robustness statements for a specified probability-error radius. They neither choose $k$ or $h$, provide a population confidence level without a sampling argument, nor establish correctness relative to an unknown latent core. The top-$n$ set is trivially the entire panel; no boundary gap is needed. The results concern STE-LSE and are not asserted for the Boltzmann mean.

\subsection{Proof of Proposition~\ref{prop:sample}}
\begin{proof}[Proof of Proposition~\ref{prop:sample}]
For each unordered pair $a<b$, let $X_1,\ldots,X_{m_{ab}}$ be its independent Bernoulli observations with common mean $P_{ab}$, and let $\widehat P_{ab}=m_{ab}^{-1}\sum_iX_i$. Define
$E_{ab}=\{|\widehat P_{ab}-P_{ab}|\ge\Delta_{ab}\}$.
Hoeffding's inequality \citep{hoeffding1963probability} gives
\[
\Pr(E_{ab})\le2e^{-2m_{ab}\Delta_{ab}^2}.
\]
Outside $E_{ab}$, if $P_{ab}>1/2$, then $\widehat P_{ab}>P_{ab}-\Delta_{ab}=1/2$; if $P_{ab}<1/2$, then $\widehat P_{ab}<P_{ab}+\Delta_{ab}=1/2$. Thus neither a wrong majority sign nor an empirical tie can occur. Reciprocal orientation identifies the reversed pair as well. Therefore
\[
\{T(\widehat P)\ne T(P)\}\subseteq\bigcup_{a<b}E_{ab}.
\]
The union bound proves Eq.~\eqref{eq:sample}. It does not require the different $E_{ab}$ to be independent. On the complementary graph-equality event, the deterministic TC and UC sets agree.

For $m_{ab}=m$ and $\Delta_{ab}\ge\Delta$, there are $n(n-1)/2$ unordered pairs, so the bound is $n(n-1)e^{-2m\Delta^2}$. For $\delta\in(0,1)$, $m\ge(2\Delta^2)^{-1}\log(n(n-1)/\delta)$ makes it at most $\delta$. This is a sufficient count, not a lower bound or a sample-optimality result. The assumptions require positive counts and independent observations within each pair; missing pairs and dependent repeated judgments need additional modeling.
\end{proof}

\begin{lemma}[Symmetric-pseudocount displacement]\label{lem:pseudocount-bias}
For $m>0$, $0\le w\le m$, and $\alpha\ge0$, the smoothed mean $(w+\alpha)/(m+2\alpha)$ differs from the raw mean by at most $\alpha/(m+2\alpha)$.
\end{lemma}
\begin{proof}
Direct subtraction gives
\[
\left|\frac{w+\alpha}{m+2\alpha}-\frac wm\right|
=\frac{\alpha|m-2w|}{m(m+2\alpha)}
\le\frac{\alpha}{m+2\alpha},
\]
because $|m-2w|\le m$. If a raw-mean estimation error is at most $\epsilon$, the triangle inequality therefore bounds the smoothed-mean error by $\epsilon+\alpha/(m+2\alpha)$ for that pair. Maximizing over observed pairs gives the corresponding uniform bound.
\end{proof}
This displacement must be included when Lemma~\ref{lem:edge} is applied to a posterior mean; it is not absorbed into the raw-mean concentration statement. At $m=0$ the raw mean is undefined, so this lemma supplies no missing-pair certificate.

\begin{corollary}[Finite-sample score recovery]\label{cor:finite-sample-recovery}
Under the observation assumptions of Proposition~\ref{prop:sample}, with positive counts $m_{ab}$ fixed by a nonadaptive design, let $\Delta=\min_{a\ne b}|P_{ab}-1/2|>0$, choose a deterministic $0<\epsilon<\Delta$, and construct canonical scores from the reciprocal raw empirical means $\widehat P$. Fix the positive temperatures and the threshold $h\in(0,1)$ in advance of those observations, set $b=\min(h,1-h)$, and write $C_{n,K}=(K-1)\log n+\log K$. Put
\[
E_{\TC}=e^{-(\Delta-\epsilon)/\tau}+\gamma C_{n,K},\qquad
E_{\UC}=3e^{-(\Delta-\epsilon)/\tau}+\gamma_c\log(n-1).
\]
If $K\ge n-1$ and $E_{\TC}<b$, the thresholded TC scores recover $\TC(T(P))$ with failure probability at most $2\sum_{a<b}e^{-2m_{ab}\epsilon^2}$. If $E_{\UC}<b$, the analogous conclusion holds for UC. When both conditions hold, both sets are recovered on the same event with that same failure bound.
\end{corollary}
\begin{proof}
For each unordered pair, Hoeffding's inequality gives
\[
\Pr\{|\widehat P_{ab}-P_{ab}|>\epsilon\}\le2e^{-2m_{ab}\epsilon^2}.
\]
Apply a union bound to these events. The reversed entries have the same absolute errors by reciprocity, and the diagonals agree at $1/2$. Therefore the common event
$G=\{\|\widehat P-P\|_\infty\le\epsilon\}$ has failure probability at most the stated sum. On $G$, Lemma~\ref{lem:edge} gives $\eta\le e^{-(\Delta-\epsilon)/\tau}$ relative to the latent hard adjacency. Theorem~\ref{thm:finite}, with $K\ge n-1$ for TC, then bounds the respective score errors by $E_{\TC}$ and $E_{\UC}$. Each strict inequality below $b$ invokes its threshold-recovery conclusion. Both conclusions use this single event $G$, so a second union bound over core types is unnecessary. Independence across pairs is never used.
\end{proof}
This combines the existing approximation and concentration results; it is not a new estimator or experiment. The temperatures, $\epsilon$, margin, and counts enter a stipulated Bernoulli-model certificate. A bound exceeding one is vacuous and may be capped at one. Applying it to smoothed means requires the displacement in Lemma~\ref{lem:pseudocount-bias}; applying it to dependent or missing observations requires additional assumptions. It does not certify the recorded ordinal-data operating points or select temperatures from an unknown test margin.

\subsection{Posterior-edge Monte Carlo precision}
\begin{proposition}[Conditional posterior simulation error]\label{prop:posterior-mc}
Fix a distribution $\Pi$ over hard tournaments and draw $T^{(1)},\ldots,T^{(S)}$ independently from it, with integer $S\ge1$. For $C\in\{\TC,\UC\}$, let $\mu_C(a)=\mathbb E_\Pi\mathbf1\{a\in C(T)\}$ and $\bar z_{C,S}(a)=S^{-1}\sum_s\mathbf1\{a\in C(T^{(s)})\}$. Then, for every $\varepsilon>0$,
\[
\Pr_\Pi\{\|\bar z_{C,S}-\mu_C\|_\infty>\varepsilon\}
\le2n e^{-2S\varepsilon^2}.
\]
\end{proposition}
\begin{proof}
For fixed $a$, the variables $Z_s(a)=\mathbf1\{a\in C(T^{(s)})\}$ are independent across $s$, belong to $[0,1]$, and share mean $\mu_C(a)$. Hoeffding's inequality yields
$\Pr_\Pi\{|S^{-1}\sum_sZ_s(a)-\mu_C(a)|>\varepsilon\}\le2e^{-2S\varepsilon^2}$.
The sup-norm event is the union of these $n$ coordinate events; a union bound proves the claim. The argument applies to either deterministic core map and does not assume independence of different agents' memberships within a tournament, or of that tournament's edges.
\end{proof}
The probability is conditional on the specified $\Pi$, including observed data if $\Pi$ is a posterior. The archived implementation uses independent Beta edge posteriors, but the proposition also permits dependent edges when whole-tournament draws are independent. Increasing $S$ reduces this simulation error; it does not establish posterior correctness, population coverage, or recovery of unobserved task dependence.

\subsection{Proof of Proposition~\ref{prop:mixture}}
\begin{proof}[Proof of Proposition~\ref{prop:mixture}]
Add and subtract $\sum_c\widehat q_cP^{(c)}$:
\[
\sum_c\widehat q_c\widehat P^{(c)}-\sum_cq_cP^{(c)}
=\sum_c\widehat q_c(\widehat P^{(c)}-P^{(c)})
+\sum_c(\widehat q_c-q_c)P^{(c)}.
\]
Because $\widehat q_c\ge0$, the first term has norm at most $\sum_c\widehat q_c\epsilon_c$. Put $d_c=\widehat q_c-q_c$. Both weight vectors sum to one, hence $\sum_cd_c=0$. For any matrix entry $p_c=P^{(c)}_{ab}\in[0,1]$, write $s=\sum_{d_c>0}d_c=-\sum_{d_c<0}d_c=\tfrac12\sum_c|d_c|$. Then
\[
-s\le\sum_{d_c<0}d_cp_c\le0,
\qquad 0\le\sum_{d_c>0}d_cp_c\le s,
\]
so $|\sum_cd_cp_c|\le s$. This holds for every entry, bounding the second term by $\tfrac12\|\widehat q-q\|_1$. The triangle inequality proves Eq.~\eqref{eq:mixture}.

The target $P(q)=\sum_cq_cP^{(c)}$ and estimate $\widehat P(\widehat q)=\sum_c\widehat q_c\widehat P^{(c)}$ remain reciprocal. If the displayed error allowance is smaller than the target's strict majority margin, each estimated off-diagonal entry stays on the same side of $1/2$ as its target entry. The induced graphs, and hence both hard cores, agree.
\end{proof}
This result assumes a common target mixture across pairs and stable context-conditional comparisons. If the observation process selects different contexts for different pairs, naive pooling need not estimate $\sum_cq_cP^{(c)}$. One needs the context labels and corresponding selection model or an appropriate sampling design. With observed ties, reweighting conditional decisive probabilities also requires the intended decisive-outcome mixture; unconditional outcome mixtures can induce pair-dependent conditioning weights. The context experiment uses Bernoulli outcomes without observed ties and a common training mixture, so it avoids this additional identification problem.

\subsection{Fixed-temperature regularity and the weighted variant}
\begin{proposition}[Smooth canonical inference]\label{prop:canonical-smooth}
For fixed finite $n,K$ and positive temperatures, the canonical maps $P\mapsto t(P)$ and $P\mapsto u(P)$ are smooth. Composing either map with a differentiable reciprocal edge predictor gives a differentiable score map.
\end{proposition}
\begin{proof}
Each off-diagonal sigmoid is smooth in its probability input; diagonal entries are fixed. The exponentials in every finite LSE sum are positive, so its logarithm has a positive argument and is smooth. The soft minimum is obtained by negation and the soft maximum. Finite compositions of these operations are smooth, establishing smoothness at every TC depth, the finite length aggregation, and the final opponent aggregation. UC additionally uses products and affine complements, which are smooth; for $n=2$ its empty witness is the constant zero. Restricting these maps to the affine space of reciprocal inputs preserves smoothness. The chain rule proves the composition claim, including gradients through the predictor rather than through a subsequently selected hard set.
\end{proof}
This establishes availability of derivatives, not useful gradient magnitude, a favorable optimum, or an optimization convergence rate. The implemented clipped BCE objective is a separate numerical choice: its clipped regions need not supply a nonzero derivative. Exact maxima and minima in historical WB paths are only piecewise differentiable.

\begin{lemma}[Boltzmann approximation]\label{lem:boltzmann-bound}
For a nonempty $r$-vector $z$ and $\gamma>0$, put $p_i=e^{z_i/\gamma}/\sum_j e^{z_j/\gamma}$, $H(p)=-\sum_ip_i\log p_i$, and $\mathcal B_\gamma(z)=\sum_ip_iz_i$. Then
\[
\mathcal B_\gamma(z)=\gamma\log\sum_i e^{z_i/\gamma}-\gamma H(p),\qquad
\max_i z_i-\gamma\log r\le\mathcal B_\gamma(z)\le\max_i z_i.
\]
\end{lemma}
\begin{proof}
From $\log p_i=z_i/\gamma-\log\sum_j e^{z_j/\gamma}$, multiplication by $p_i$ and summation gives $-H(p)=\mathcal B_\gamma(z)/\gamma-\log\sum_j e^{z_j/\gamma}$, establishing the identity. The unnormalized LSE is at least $\max z$. To bound the entropy, concavity of the logarithm yields
$H(p)=\sum_ip_i\log(1/p_i)\le\log(\sum_ip_i/p_i)=\log r$.
The identity then gives the lower bound. The upper bound follows from $\mathcal B_\gamma$ being a convex combination of the entries. It is also at least $\min z$, so it preserves the score range. The bound is uniform in $z$ and vanishes as $\gamma\to0$ for fixed $r$.
\end{proof}
The Boltzmann mean therefore has the same zero-temperature maximum limit, but a different finite-temperature value. No monotonicity or nonexpansiveness claim for it is used here.

\begin{corollary}[Boltzmann UC approximation]\label{cor:boltzmann-uc}
Let $M$ be a strict tournament adjacency and $D\in[0,1]^{n\times n}$ have zero diagonal, with $\eta=\|D-M\|_\infty$. Form $u^{\rm B}(D)$ by replacing both UC maxima in Eqs.~\eqref{eq:violation}--\eqref{eq:uc} with $\mathcal B_{\gamma_c}$, retaining the zero empty witness when $n=2$. Then
\[
\|u^{\rm B}(D)-\mathbf1_{\UC(M)}\|_\infty\le3\eta+\gamma_c\log(n-1).
\]
\end{corollary}
\begin{proof}
The ordinary-max UC computation $U_D$ in the proof of Theorem~\ref{thm:finite} differs from hard UC by at most $3\eta$; that comparison does not involve a soft aggregation. With $L_c=\gamma_c\log\max\{1,n-2\}$, Lemma~\ref{lem:boltzmann-bound} gives $W_D-L_c\le v^{\rm B}\le W_D$. Multiplying the complement by $D_{ca}\in[0,1]$ gives $C_D\le\operatorname{cov}^{\rm B}\le C_D+L_c$. If $C_* = \max_{c\ne a}C_D(c,a)$ and $L_f=\gamma_c\log(n-1)$, the same lemma applied to the final cover vector gives
$C_*-L_f\le\mathcal B_{\gamma_c}(\operatorname{cov}^{\rm B})\le C_*+L_c$.
Thus $u^{\rm B}-U_D\in[-L_c,L_f]$. Since $L_c\le L_f$, adding the ordinary-max edge error proves the displayed allowance. For $n=2$ both $L_c,L_f$ vanish and the witness is exactly zero.
\end{proof}
This transfers approximation accounting, not the STE-LSE perturbation constant. All Boltzmann UC scores remain in $[0,1]$, so the bound may be capped at one. For historical weighted edges the relevant $\eta$ includes the weighting error; no unweighted sigmoid error is substituted for it.

\begin{lemma}[Confidence-weighted edge allowance]\label{lem:weighted-edge-bound}
Let $r_{ab}=m_{ab}/(m_{ab}+5)$ off the diagonal, $\rho=\max_{a\ne b}(1-r_{ab})$, and let $D_\tau$ and $M$ have zero diagonal. Then
\[
\|r\odot D_\tau-M\|_\infty\le\|D_\tau-M\|_\infty+\rho.
\]
\end{lemma}
\begin{proof}
For each off-diagonal entry, $0\le r_{ab},D_{\tau,ab}\le1$, and
\[
|r_{ab}D_{\tau,ab}-M_{ab}|
\le |D_{\tau,ab}-M_{ab}|+(1-r_{ab})D_{\tau,ab}
\le |D_{\tau,ab}-M_{ab}|+\rho.
\]
The diagonals are zero, so maximizing proves the bound.
\end{proof}
At fixed finite counts, temperature alone does not send the allowance $\rho$ to zero. In fact, for any strict winning edge with $M_{ab}=1$ and fixed $r_{ab}<1$, the idealized weighted edge tends to $r_{ab}$ as $\tau\to0$, leaving edge error $1-r_{ab}$. A sufficient full-target hard-limit argument using these bounds therefore requires $r_{ab}\to1$ on all target pairs, or must retain the weighting error explicitly. This is not a necessary condition for every individual selected set. A missing pair has $\rho=1$, so this sufficient certificate is uninformative. Historical code additionally uses numerical floors and logit clipping; those guards are described in Appendix~\ref{app:variant-details}, rather than silently treated as exact zero-temperature limits.

\subsection{Neutral evidence and numerical meaning of a sufficient bound}\label{app:neutral}
\begin{proof}[Proof of Lemma~\ref{lem:neutral}]
Every off-diagonal soft edge is $\sigma(0)=1/2$. For distinct $c,a$, the $n-2$ eligible witnesses are nonempty and each is $(1/2)(1-1/2)=1/4$. A normalized soft maximum of a constant $r$-vector equals its constant, since $\gamma\log(r^{-1}r e^{x/\gamma})=x$. Thus $v(c,a)=1/4$, every cover equals $(1/2)(3/4)=3/8$, and the final normalized maximum over $n-1$ equal covers is $3/8$. Its complement gives $u_a=5/8$. For $n=2$ the witness convention is zero, so the single cover is $1/2$ and the UC score is $1/2$ instead.
\end{proof}
Under the neutral missing-pair convention, an entirely unobserved graph has this same score. It contains no directional evidence and satisfies no strict-margin recovery condition. This algebraic reference is not a posterior membership probability, does not imply that the population core is the whole panel, and does not establish the cause of a particular human-data failure. In exact arithmetic the rule $u_a\ge5/8$ includes every neutral alternative; numerical comparisons at an exact boundary additionally depend on floating-point rounding. The development-selected .625 cutoff in Section~\ref{sec:unknown} coincides with this reference value. It remains the frozen empirical rule, not a confidence level.

\begin{proof}[Proof of Lemma~\ref{lem:distinct-summaries}]
Each unordered edge has either orientation in four of the eight tournaments, so $\mathbb E_\Pi P$ has off-diagonal entries $1/2$. Lemma~\ref{lem:neutral} gives its canonical score. Among the eight tournaments, six are transitive and two are three-cycles. Each transitive tournament has only its Condorcet winner in UC, and each cycle has all vertices in UC. For a fixed vertex $a$, exactly two transitive tournaments have $a$ as winner, corresponding to the two possible orders of its opponents. It also belongs to UC in both cycles. Thus its posterior inclusion is $(2+2)/8=1/2$.
\end{proof}
This finite construction establishes that the two summaries are different functions; it is not an empirical claim about any recorded posterior or a calibration guarantee.

\paragraph{When the sufficient certificate is informative.}\label{app:bias-budget}
Define $C_{n,K}=(K-1)\log n+\log K$. The sharpened smoothing allowances are $B_{\TC}=\gamma C_{n,K}$ and $B_{\UC}=\gamma_c\log(n-1)$. Table~\ref{tab:smoothing-budgets} compares them with the historical TC allowance $\gamma\{(K-1)\log(2n)+\log[K(n-1)]\}$ at the executed study's full depth and default temperature. These are uncapped sufficient bounds with the idealized $\eta=0$; the actual sigmoid generally has $\eta>0$ even when $P$ is known. The TC certificate remains vacuous throughout the learning study. A smaller bound does not establish a smaller measured error, useful gradients, or better selected sets.

\begin{table}[ht]\centering\small
\begin{tabular}{@{}lrrrr@{}}
\toprule
$n$ & $K$ & Original TC & Sharpened TC & Sharpened UC\\
\midrule
12 & 11 & 1.280 & 0.954 & 0.084\\
24 & 23 & 3.200 & 2.557 & 0.110\\
48 & 47 & 7.618 & 6.367 & 0.135\\
64 & 63 & 10.819 & 9.170 & 0.145\\
96 & 95 & 17.616 & 15.176 & 0.159\\
\bottomrule
\end{tabular}

\caption{Smoothing-only allowances at $\gamma=\gamma_c=0.035$, $K=n-1$. Values greater than one are uninformative for score recovery. Statistical/edge-smoothing error must additionally fit within the threshold allowance.}
\label{tab:smoothing-budgets}\end{table}

\paragraph{Binary-target UC refinement.}\label{rem:uc-binary-refinement}
For the same canonical UC operator with $n\ge2$ and $\gamma_c>0$, let $M$ be a strict tournament adjacency, let $D\in[0,1]^{n\times n}$ have zero diagonal, and put $\eta=\|D-M\|_\infty$. Then
\[
\|u(D)-\mathbf1_{\UC(M)}\|_\infty
\le\min\{1,\;2\eta-\eta^2+\gamma_c\log(n-1)\}.
\]
This optional refinement uses the binary target and does not require reciprocal $D$. It neither changes the operator nor invalidates the conservative $3\eta$ theorem and its corollaries; it is not a general perturbation bound between arbitrary soft inputs.
\begin{proof}[Proof of the binary-target UC refinement]
Write $A=1-(1-\eta)^2=2\eta-\eta^2$, $L_c=\gamma_c\log\max\{1,n-2\}$, and $L_f=\gamma_c\log(n-1)$. Since $0\le\eta\le1$, $A\ge\eta$ and $L_c\le L_f$.

If $a$ is covered in $M$, choose a covering vertex $c$. Then $M_{ca}=1$, and each eligible witness $b$ has either $M_{ab}=0$ or $M_{cb}=1$. In the first case $D_{ab}\le\eta$; in the second $1-D_{cb}\le\eta$. Thus every witness product is at most $\eta$, and its normalized soft maximum satisfies $v(c,a)\le\eta$. This also holds for the zero empty witness when $n=2$. Because $D_{ca}\ge1-\eta$,
\[
\operatorname{cov}(c,a)=D_{ca}(1-v(c,a))\ge(1-\eta)^2.
\]
The normalized maximum over all covers is at least this cover minus $L_f$. Therefore the zero-target score obeys $u(a)\le A+L_f$.

If $a\in\UC(M)$, consider each $c\ne a$. When $M_{ca}=0$, $D_{ca}\le\eta$, hence $\operatorname{cov}(c,a)\le\eta\le A$. When $M_{ca}=1$, the absence of a hard cover supplies an eligible witness $b$ with $M_{ab}=1$ and $M_{cb}=0$. Its soft product satisfies
$D_{ab}(1-D_{cb})\ge(1-\eta)^2$.
Lemma~\ref{lem:extrema} gives $v(c,a)\ge(1-\eta)^2-L_c$, so
\[
\operatorname{cov}(c,a)\le D_{ca}(A+L_c)\le A+L_c.
\]
For $n=2$ this second branch cannot occur: with no witnesses, any incoming hard edge would itself cover $a$. Thus the first branch handles every potential cover of a two-vertex UC member. All cover scores of a member are at most $A+L_c$. Their normalized maximum is no greater than their ordinary maximum, giving $1-u(a)\le A+L_c\le A+L_f$.

These two cases prove the uniform allowance. Scores and binary targets lie in $[0,1]$, so their distance is also at most one. At $h=1/2$, write $B=\gamma_c\log(n-1)$. If $B<1/2$, the sufficient condition $A+B<1/2$ is equivalent, for $\eta\in[0,1]$, to
\[
\eta<1-\sqrt{1/2+B},
\]
since $A=1-(1-\eta)^2$. The threshold-recovery argument in Theorem~\ref{thm:finite} then applies.
\end{proof}
The refinement remains a deterministic sufficient certificate requiring a valid edge-error allowance. It does not certify any empirical operating point or supply an unknown test margin.

\begin{corollary}[A sufficient temperature budget]\label{cor:temperature-budget}
Under Lemma~\ref{lem:edge} and, for TC, $K\ge n-1$, fix $h\in(0,1)$, $b=\min(h,1-h)$, and $e_0\in(0,b)$ for TC (or $e_0\in(0,b/3)$ for UC). If
\[
\tau\le\frac{\Delta-\epsilon}{\log(1/e_0)},\qquad
\gamma C_{n,K}<b-e_0\quad\text{(TC)},\qquad
\gamma_c\log(n-1)<b-3e_0\quad\text{(UC)},
\]
then the relevant thresholded score recovers its strict hard core. A zero smoothing coefficient imposes no constraint on its temperature.
\end{corollary}
\begin{proof}
The assumptions give $0<e_0<1$ and $\Delta-\epsilon>0$, so the denominator $\log(1/e_0)$ is positive. The condition on $\tau$ implies $(\Delta-\epsilon)/\tau\ge\log(1/e_0)$. Hence Lemma~\ref{lem:edge} gives $\eta\le e^{-(\Delta-\epsilon)/\tau}\le e_0$. For TC, the full-depth assumption permits Theorem~\ref{thm:finite}, yielding error at most $e_0+\gamma C_{n,K}<b$. For UC, the same theorem gives error at most $3e_0+\gamma_c\log(n-1)<b$. Its threshold conclusion proves recovery in either case. If $C_{n,K}=0$ or $\log(n-1)=0$, the respective smoothing term is identically zero, so any positive value of that temperature satisfies the relevant inequality; $n=2,K=1$ realizes both zero coefficients.
\end{proof}
The margin and uniform estimation error must be known or bounded appropriately. Choosing $\tau$ from an unknown test margin would be oracle tuning and is not an automatic method. With growing $n$ and full $K=n-1$, the TC coefficient is at most $n\log n$, and the UC coefficient is at most $\log n$. Therefore $\gamma_n n\log n\to0$ and $\gamma_{c,n}\log n\to0$, together with $(\Delta_n-\epsilon_n)/\tau_n\to\infty$, make both sufficient error bounds vanish. Uniform estimation with growing dimension is an additional assumption, not supplied by this deterministic argument. No empirical validation of such a schedule follows from it.

\subsection{A worked tournament and finite-temperature ordering}
Figure~\ref{fig:worked-tournament} has edges $a\to b$, $b\to c$, $c\to a$, $a\to d$, $b\to d$, and $d\to c$. Its out-neighborhoods are
\[
N^+(a)=\{b,d\},\quad N^+(b)=\{c,d\},\quad
N^+(c)=\{a\},\quad N^+(d)=\{c\}.
\]
The directed cycle $a\to b\to d\to c\to a$ visits all four vertices, so each vertex reaches every other and TC is the full panel. Only $d$ is covered: $b\to d$ and $N^+(d)\subseteq N^+(b)$. To exclude other covers, $a$ loses only to $c$, but $c$ does not beat $b\in N^+(a)$; $b$ loses only to $a$, but $a$ does not beat $c\in N^+(b)$; and $c$ loses to $b,d$, neither of which beats $a\in N^+(c)$. Thus UC is exactly $\{a,b,c\}$. In particular, $a$ does not cover $d$, since $d\to c$ but $c\to a$.

\begin{proof}[Proof of Proposition~\ref{prop:ordering-counterexample}]
With row/column order $a,b,c$, take
\[
P=\begin{pmatrix}0.5&0.501&0.501\\0.499&0.5&0.9\\0.499&0.1&0.5\end{pmatrix}.
\]
The strict edges are $a\to b$, $a\to c$, $b\to c$. Vertex $a$ is a Condorcet winner, so the hard-core argument in Corollary~\ref{cor:condorcet} gives both hard sets as $\{a\}$.
Write $d=\sigma(0.001/\tau)$ and $q=\sigma(0.4/\tau)$. Each witness set has one entry, hence its normalized maximum is the identity. Substitution into Eqs.~\eqref{eq:violation}--\eqref{eq:uc} gives the following two covers for each target:
\[
\begin{array}{ll}
a:&A_1=(1-d)[1-d(1-q)],\quad A_2=(1-d)(1-dq),\\
b:&B_1=d[1-q(1-d)],\quad B_2=(1-q)[1-d(1-d)],\\
c:&d[1-(1-q)(1-d)],\quad q[1-d(1-d)].
\end{array}
\]
At the specified $\tau$, $0.5<d<0.51$: the sigmoid derivative is at most $1/4$, so $d\le1/2+(1/35)/4<0.51$. Also $q>0.99$, because $0.4/0.035=80/7>\log99$. The differences obey
\[
\begin{split}
A_1-B_1&=1-3d+d^2+2d(1-d)q\\
&>1-3(0.51)+(0.5)^2+2(0.5)(0.49)(0.99)>0,\\
A_2-B_2&=-d^2+q(1-2d+2d^2)\\
&>-(0.51)^2+(0.99)/2>0,
\end{split}
\]
where $1-2d+2d^2=2(d-1/2)^2+1/2\ge1/2$. Thus $(A_1,A_2)$ is strictly larger coordinatewise than $(B_1,B_2)$. The positive derivatives in Lemma~\ref{lem:extrema} make its normalized maximum strictly larger. Subtracting from one gives $u(a)<u(b)$, for any positive final cover temperature and in particular for $\gamma_c=0.035$. A threshold that includes $a$ must include the higher-scoring $b$; a descending prefix that includes $a$ does also. Neither can yield $\{a\}$.
\end{proof}
Numerical evaluation of these explicit formulas gives approximately $(0.531378,0.767042,0.274183)$. There is no estimation error; this is an ordering failure of finite-temperature smoothing, distinct from a separating ordering with a poorly transferred cutoff. The margin is $0.001$, outside the small-error recovery regime. Changing only $\tau$ to $0.00035$ gives approximately $(0.962687,0.129937,0.052807)$, restoring the member's lead. These examples neither contradict limiting Condorcet inclusion nor estimate failure prevalence or validate a new temperature. The accompanying saved local check verifies the formulas and general operator; it is not new GPU training.

\subsection{Data-dependent TC certificates}\label{app:tc-certificates}
Theorem~\ref{thm:finite} bounds every aggregation by its worst-case allowance. The following result instead measures the smoothing discrepancy on the supplied matrix, keeping the canonical operator unchanged. It does not remove the need for a valid uniform edge-error bound.

\begin{proposition}[Computable TC smoothing discrepancy]\label{prop:tc-computable-certificate}
Let $n\ge2$, let $M$ be a strict tournament adjacency, and let $D\in[0,1]^{n\times n}$ have zero diagonal. Fix an integer $K\ge n-1$ and $\gamma>0$, and write $\eta=\|D-M\|_\infty$. Let $t_\gamma(D)$ be the canonical TC score at this depth and aggregation temperature. Define $t_0(D)$ by replacing every soft maximum and minimum in Eqs.~\eqref{eq:product}--\eqref{eq:tc} with its ordinary maximum or minimum, retaining the same $D$ and $K$. Put
\[
B_a(D)=|t_{\gamma,a}(D)-t_{0,a}(D)|.
\]
Then, for every alternative $a$,
\[
|t_{\gamma,a}(D)-\mathbf1_{\TC(M)}(a)|
\le E_a:=\min\{1,\eta+B_a(D)\}.
\]
For a threshold $h\in(0,1)$, $E_a<\min(h,1-h)$ certifies the threshold decision for $a$. If this inequality holds for every $a$, it certifies the entire selected set.
\end{proposition}
\begin{proof}
For matrices $X,Y,X',Y'$ of the same size, ordinary maxima and minima satisfy
\[
\left\|\max_c\min\{X_{ac},Y_{cb}\}
-\max_c\min\{X'_{ac},Y'_{cb}\}\right\|_\infty
\le\max\{\|X-X'\|_\infty,\|Y-Y'\|_\infty\}.
\]
Indeed, each two-entry minimum changes by at most the greatest change of its entries, and an ordinary maximum preserves that allowance. The same inequality applies when intermediate matrices have nonzero diagonals. The unsmoothed recurrence starts at $D$ and $M$ with error $\eta$. If its preceding pair of matrices differs by at most $\eta$, the displayed inequality bounds the next pair by $\eta$, since the second inputs remain $D$ and $M$. Induction therefore gives this bound at every depth $k\le K$. Taking ordinary maxima over lengths and minima over opponents yields
\[
\|t_0(D)-t_0(M)\|_\infty\le\eta.
\]
The hard-path argument in the proof of Theorem~\ref{thm:finite} gives $t_0(M)=\mathbf1_{\TC(M)}$ when $K\ge n-1$. The triangle inequality now gives the allowance $\eta+B_a(D)$. Both scores and the binary target lie in $[0,1]$, giving the cap at one. If the target is zero, $t_{\gamma,a}(D)\le E_a<h$; if it is one, $t_{\gamma,a}(D)\ge1-E_a>h$. This proves the coordinate and complete-set conclusions.
\end{proof}

$B_a(D)$ is computable without the latent tournament. The unsmoothed score is the minimum, over opponents, of maximum path-bottleneck weights; its closure can also be evaluated by ordinary max--min Floyd--Warshall updates. This is an auxiliary certificate computation, not a replacement training objective or decoder. A supplied upper bound on $\eta$ may replace $\eta$ throughout. Lemma~\ref{lem:edge} supplies one under its strict-margin and uniform-estimation assumptions; low average pair MSE or an unvalidated posterior does not. A small discrepancy certifies neither the edge estimator nor an unknown population target by itself. No claim that this certificate holds at a recorded operating point follows without checking both terms.

\begin{corollary}[A one-sided shortest-path refinement]\label{cor:tc-short-path}
Under the assumptions of Proposition~\ref{prop:tc-computable-certificate}, let $a\in\TC(M)$ and let $\ell_a$ be the greatest shortest-path distance from $a$ to any distinct vertex in $M$. Then $1\le\ell_a\le n-1$ and
\[
1-t_{\gamma,a}(D)
\le\min\{1,\eta+\gamma[(\ell_a-1)\log n+\log K]\}.
\]
\end{corollary}
\begin{proof}
For every $b\ne a$, choose a shortest path of length $k_b\le\ell_a$ from $a$ to $b$. The hard length-$k_b$ reachability entry is one. The lower recurrence bound in the proof of Theorem~\ref{thm:finite} therefore gives
\[
Q^{(k_b)}_{ab}\ge1-\eta-(k_b-1)\gamma\log n
\ge1-\eta-(\ell_a-1)\gamma\log n.
\]
All these lengths are included because $K\ge n-1$. Lemma~\ref{lem:extrema} bounds the normalized maximum over lengths from below by any selected entry minus $\gamma\log K$. Thus every opponent reachability score is at least $1-\eta-\gamma[(\ell_a-1)\log n+\log K]$. The final normalized minimum is no smaller than the ordinary minimum, so the same lower bound holds for $t_{\gamma,a}(D)$. Since this score lies in $[0,1]$, its one-sided error is also at most one.
\end{proof}
The refinement explains why short-path members can have a useful allowance when the uniform full-depth bound is vacuous. It requires their paths in the target graph to be known or certified and says nothing about nonmembers. In particular, substituting $\ell_a$ for $K$ in the full two-sided theorem would be unjustified: the unchanged operator still aggregates all $K$ depths. These are conditional approximation certificates, not optimization, calibration, or predictive-performance claims. Deterministic CPU checks of the formulas support implementation consistency; the written proofs, rather than those checks, establish the results.

\section{Matched-epoch study: protocol and verification}\label{app:ibm-gpu}
The matched-epoch study uses a cloud NVIDIA L40S for separate TC/UC fits, source-excluded recorded-profile calibration, numerical diagnostics, and cost measurements. Its protocol was frozen before execution but not externally preregistered. %Reanalysis of preserved outputs is not another independent experiment. 
The equal-budget follow-up and earlier studies are reported separately. Provider descriptions distinguish election ballots from sports lap rankings in the recorded-data cohort; this provenance correction changes neither observations nor numerical endpoints.

\subsection{Execution, provenance, and verification scope}
The environment uses Python 3.11.7, PyTorch 2.6.0+cu124 (CUDA runtime 12.4), NumPy 1.26.4, SciPy 1.13.1, and pandas 2.2.3. Training uses float32 without AMP or TF32, deterministic PyTorch settings, and one numerical thread. All 504 candidate fits record GPU placement of models, inputs, labels, and optimizer state and complete 160 epochs. Continuous GPU exclusivity was not monitored, so wall-clock differences between stages are not interpreted as individual fit times.

Archive integrity and frozen code/input comparisons pass. The reviewed evidence retains development/test scores, calibration choices, diagnostics, and raw ordinal profiles, while the separately preserved full archive additionally contains weights, training data, ordinal subsamples, posterior draws, and count-score arrays. Thus this study's verification does not include weight-to-prediction replay, training-ID checks, or independent regeneration of every ordinal sample, posterior draw, and count decision.

Learning metrics and development choices reproduce from saved labels and scores, with maximum per-case discrepancy $4.45\times10^{-16}$. The 1,536 development and 12,288 test identifiers do not overlap. Exported count aggregates also reproduce. All 462 raw ordinal profiles were reparsed and their reference cores checked; source-excluded calibration and recorded-profile metrics reproduce within $5.56\times10^{-16}$. The 1,215 numerical diagnostic cells were checked, and all 540 plain/checkpoint output and gradient pairs agree exactly in their saved float32 representations. These checks establish computational consistency within the retained evidence, not an independent training replication. Thresholds retain their full numerical precision because rounding a cutoff can change an inclusive boundary decision.

\subsection{Synthetic generators, inputs, and matched model selection}
Independent streams use semantic seed namespaces for collection, phase, size, and graph. Each collection has 512 training and 128 development graphs at $n=12$, and 256 test graphs at each $n=12,24,48,64$. Five families cycle by graph index. Ordered utilities are a perturbed decreasing grid. Planted odd regular cyclic cores dominate transitive outsiders. The separated family embeds Figure~\ref{fig:worked-tournament}'s four-vertex motif dominating outsiders, so TC and UC differ; unlike the historical construction, its TC need not be the entire panel. Rank mixtures use five BTL contexts with Dirichlet weights. Random orientations produce frequent full cores. Orientation-based families have independent unordered-pair margins uniform on $[0.04,0.35]$; vertex permutation removes fixed index cues. Two independently coded hard oracles agree on generated labels.

For each unordered pair, $H$ is unit-mean lognormal with log standard deviation 0.7. Planned missingness 0.25 removes its $1+\operatorname{Poisson}(10H)$ comparisons; otherwise each comparison ties with probability 0.2 and decisive outcomes use latent $P$. Matrix $W$ stores decisive wins and $T$ tied counts. The neural input has four edge features (smoothed margin, log decisive count, tie fraction, observation indicator) and six summaries for each endpoint, totaling sixteen. The shared encoder has two width-32 tanh layers and a scalar antisymmetrized pair head. The ordinary membership head pools incoming/outgoing encoded edges. The relational control adds four summaries of predicted relation strengths, including two-step paths; it is not the two-round message-passing head of the older study.

For each target, collection, and initialization, pairwise-only uses one fit; auxiliary, relational auxiliary, and STE each use two fits at weights 0.1 and 1: $12\times3\times2\times(1+2+2+2)=504$. Models share initial pair parameters and minibatch order. Inactive heads in pairwise-only/STE are not treated as trained direct baselines. Their active pair predictor has 1,633 parameters; auxiliary has 3,746 active parameters and relational auxiliary 3,874. Nominal base-model parity does not mean equal active parameter counts, but the pair estimator architecture and supervision labels are matched. All fit 160 epochs with Adam at 0.001, batch 16, and norm-5 clipping. Soft temperatures are tied at 0.035 and TC depth is $n-1$. At epochs 40/80/160, readout-specific development F1 selects weight, snapshot, and threshold. Ties prefer earlier epochs, smaller weights, cutoff nearer 0.5, then the smaller cutoff. The absolute grid is $0,0.025,\ldots,1$ with all/none sentinels. Relative selection first min--max normalizes a profile, returning constant 0.5 for constant scores. Largest-gap selection is a fixed rule with all-selected for a constant vector.

Training BCE inputs, including pair probabilities and core/head scores, are clamped to $[10^{-6},1-10^{-6}]$ for numerical stability. This differs from the unclipped mathematical operator: a saturated clipped coordinate can have zero loss gradient. Pair likelihood divides by total decisive comparisons in the minibatch; target BCE averages its graphs and alternatives. Clipping norm 5 acts on model parameters after backpropagation. Saved diagnostic gradients instead differentiate clipped BCE with respect to upper-triangle probability logits, not these model parameters.

\begin{table}[H]\centering\small
\begin{tabular}{@{}lrrrrrrr@{}}
\toprule
Target & $n$ & Pair/struct. & Aux./struct. & STE/struct. & Aux./head & Rel./struct. & Rel./head\\
\midrule
UC & 12 & 0.6446 & 0.6527 & 0.6710 & 0.6873 & 0.6636 & 0.6894\\
UC & 24 & 0.5572 & 0.5761 & 0.6078 & 0.5772 & 0.5858 & 0.5868\\
UC & 48 & 0.4383 & 0.4638 & 0.4931 & 0.4588 & 0.4663 & 0.4669\\
UC & 64 & 0.4018 & 0.4227 & 0.4499 & 0.4288 & 0.4238 & 0.4366\\
TC & 12 & 0.6416 & 0.6524 & 0.6722 & 0.7106 & 0.6501 & 0.7104\\
TC & 24 & 0.4816 & 0.5111 & 0.5405 & 0.5729 & 0.5128 & 0.5761\\
TC & 48 & 0.4004 & 0.4268 & 0.4399 & 0.4574 & 0.4266 & 0.4597\\
TC & 64 & 0.3845 & 0.4004 & 0.4202 & 0.4321 & 0.4041 & 0.4356\\
\bottomrule
\end{tabular}

\caption{Complete-system descriptive mean F1 with absolute calibration. ``Struct.'' uses the relevant TC/UC operator; ``head'' is the model's own supervised membership output. Each readout receives its own development selection. Relational control has 128 extra parameters.}
\label{tab:gpu-systems}\end{table}

\begingroup\footnotesize
\begin{longtable}{@{}lrl lrrrrr@{}}
\caption{All matched-epoch learned systems under absolute or native selection. Exact is complete-set recovery. AP omits full/empty reference targets.}\label{tab:gpu-learning-all}\\
\toprule
Target & $n$ & Objective & Readout & F1 & Exact (\%) & AP & Size & Truth\\
\midrule\endfirsthead
\toprule
Target & $n$ & Objective & Readout & F1 & Exact (\%) & AP & Size & Truth\\
\midrule\endhead
TC & 12 & all & native & 0.5055 & 19.73 & -- & 12.00 & 5.01\\
TC & 12 & aux & direct & 0.7106 & 25.16 & 0.8638 & 6.15 & 5.01\\
TC & 12 & aux & structural & 0.6524 & 32.34 & 0.8338 & 7.69 & 5.01\\
TC & 12 & none & native & 0.0000 & 0.00 & -- & 0.00 & 5.01\\
TC & 12 & pair & structural & 0.6416 & 29.57 & 0.8277 & 8.03 & 5.01\\
TC & 12 & relational\_aux & direct & 0.7104 & 25.30 & 0.8650 & 6.14 & 5.01\\
TC & 12 & relational\_aux & structural & 0.6501 & 32.12 & 0.8366 & 7.45 & 5.01\\
TC & 12 & ste & structural & 0.6722 & 33.02 & 0.8506 & 6.52 & 5.01\\
TC & 24 & all & native & 0.4282 & 19.95 & -- & 24.00 & 8.71\\
TC & 24 & aux & direct & 0.5729 & 21.77 & 0.8279 & 14.68 & 8.71\\
TC & 24 & aux & structural & 0.5111 & 25.85 & 0.7850 & 18.92 & 8.71\\
TC & 24 & none & native & 0.0000 & 0.00 & -- & 0.00 & 8.71\\
TC & 24 & pair & structural & 0.4816 & 23.24 & 0.7804 & 20.34 & 8.71\\
TC & 24 & relational\_aux & direct & 0.5761 & 21.89 & 0.8275 & 14.59 & 8.71\\
TC & 24 & relational\_aux & structural & 0.5128 & 26.02 & 0.7889 & 18.64 & 8.71\\
TC & 24 & ste & structural & 0.5405 & 28.16 & 0.8005 & 17.37 & 8.71\\
TC & 48 & all & native & 0.3869 & 20.15 & -- & 48.00 & 16.34\\
TC & 48 & aux & direct & 0.4574 & 20.07 & 0.7891 & 33.02 & 16.34\\
TC & 48 & aux & structural & 0.4268 & 21.96 & 0.7631 & 41.73 & 16.34\\
TC & 48 & none & native & 0.0000 & 0.00 & -- & 0.00 & 16.34\\
TC & 48 & pair & structural & 0.4004 & 20.97 & 0.7565 & 45.35 & 16.34\\
TC & 48 & relational\_aux & direct & 0.4597 & 20.19 & 0.7893 & 32.67 & 16.34\\
TC & 48 & relational\_aux & structural & 0.4266 & 22.76 & 0.7650 & 41.12 & 16.34\\
TC & 48 & ste & structural & 0.4399 & 24.05 & 0.7693 & 40.76 & 16.34\\
TC & 64 & all & native & 0.3789 & 20.31 & -- & 64.00 & 21.66\\
TC & 64 & aux & direct & 0.4321 & 20.04 & 0.7721 & 44.89 & 21.66\\
TC & 64 & aux & structural & 0.4004 & 21.00 & 0.7456 & 56.70 & 21.66\\
TC & 64 & none & native & 0.0000 & 0.00 & -- & 0.00 & 21.66\\
TC & 64 & pair & structural & 0.3845 & 20.67 & 0.7430 & 61.70 & 21.66\\
TC & 64 & relational\_aux & direct & 0.4356 & 20.21 & 0.7755 & 44.48 & 21.66\\
TC & 64 & relational\_aux & structural & 0.4041 & 22.17 & 0.7526 & 55.70 & 21.66\\
TC & 64 & ste & structural & 0.4202 & 23.26 & 0.7602 & 56.03 & 21.66\\
UC & 12 & all & native & 0.4605 & 2.38 & -- & 12.00 & 4.32\\
UC & 12 & aux & direct & 0.6873 & 13.90 & 0.8513 & 5.69 & 4.32\\
UC & 12 & aux & structural & 0.6527 & 12.94 & 0.8204 & 4.94 & 4.32\\
UC & 12 & none & native & 0.0000 & 0.00 & -- & 0.00 & 4.32\\
UC & 12 & pair & structural & 0.6446 & 8.80 & 0.8141 & 5.32 & 4.32\\
UC & 12 & relational\_aux & direct & 0.6894 & 15.07 & 0.8538 & 5.53 & 4.32\\
UC & 12 & relational\_aux & structural & 0.6636 & 15.49 & 0.8277 & 5.00 & 4.32\\
UC & 12 & ste & structural & 0.6710 & 13.41 & 0.8385 & 4.74 & 4.32\\
UC & 24 & all & native & 0.4069 & 13.41 & -- & 24.00 & 8.27\\
UC & 24 & aux & direct & 0.5772 & 16.13 & 0.8149 & 13.50 & 8.27\\
UC & 24 & aux & structural & 0.5761 & 7.34 & 0.7637 & 11.75 & 8.27\\
UC & 24 & none & native & 0.0000 & 0.00 & -- & 0.00 & 8.27\\
UC & 24 & pair & structural & 0.5572 & 7.31 & 0.7564 & 12.67 & 8.27\\
UC & 24 & relational\_aux & direct & 0.5868 & 16.66 & 0.8162 & 13.14 & 8.27\\
UC & 24 & relational\_aux & structural & 0.5858 & 10.64 & 0.7715 & 12.03 & 8.27\\
UC & 24 & ste & structural & 0.6078 & 9.00 & 0.7858 & 11.20 & 8.27\\
UC & 48 & all & native & 0.3663 & 19.92 & -- & 48.00 & 15.49\\
UC & 48 & aux & direct & 0.4588 & 20.04 & 0.7634 & 29.93 & 15.49\\
UC & 48 & aux & structural & 0.4638 & 18.20 & 0.7291 & 28.13 & 15.49\\
UC & 48 & none & native & 0.0000 & 0.00 & -- & 0.00 & 15.49\\
UC & 48 & pair & structural & 0.4383 & 19.76 & 0.7173 & 30.24 & 15.49\\
UC & 48 & relational\_aux & direct & 0.4669 & 20.16 & 0.7688 & 29.08 & 15.49\\
UC & 48 & relational\_aux & structural & 0.4663 & 19.86 & 0.7362 & 28.39 & 15.49\\
UC & 48 & ste & structural & 0.4931 & 19.48 & 0.7315 & 25.73 & 15.49\\
UC & 64 & all & native & 0.3562 & 19.92 & -- & 64.00 & 20.35\\
UC & 64 & aux & direct & 0.4288 & 20.02 & 0.7454 & 40.68 & 20.35\\
UC & 64 & aux & structural & 0.4227 & 20.46 & 0.7122 & 40.27 & 20.35\\
UC & 64 & none & native & 0.0000 & 0.00 & -- & 0.00 & 20.35\\
UC & 64 & pair & structural & 0.4018 & 20.40 & 0.6990 & 43.23 & 20.35\\
UC & 64 & relational\_aux & direct & 0.4366 & 20.14 & 0.7495 & 39.59 & 20.35\\
UC & 64 & relational\_aux & structural & 0.4238 & 21.01 & 0.7163 & 40.61 & 20.35\\
UC & 64 & ste & structural & 0.4499 & 20.89 & 0.7170 & 36.49 & 20.35\\
\bottomrule
\end{longtable}

\endgroup

The twelve collection means, not 36 initializations or individual graph/metric rows, are the primary units. The three effects in Table~\ref{tab:gpu-primary} use $t_{11}$ intervals and all $2^{12}=4096$ sign flips, with Holm correction across exactly those three comparisons. The intervals are approximate and not simultaneous. Readout, family, size-$12/64$, and reference-class contrasts are descriptive secondary analyses. AP/AUROC exclude full/empty targets, so their support differs from F1; the retained class counts identify that denominator.

\begin{table}[ht]\centering\small
\begin{tabular}{@{}lr lrrr@{}}
\toprule
Target & $n$ & Reference class & STE & Aux. & $\Delta$ F1\\
\midrule
UC & 24 & singleton & 0.3881 & 0.3564 & $+0.0317$\\
UC & 24 & selective non singleton & 0.6835 & 0.6459 & $+0.0376$\\
UC & 24 & full & 0.9616 & 0.9530 & $+0.0087$\\
UC & 48 & singleton & 0.2361 & 0.2007 & $+0.0355$\\
UC & 48 & selective non singleton & 0.4842 & 0.4472 & $+0.0370$\\
UC & 48 & full & 0.9986 & 0.9979 & $+0.0007$\\
TC & 24 & singleton & 0.3444 & 0.2751 & $+0.0693$\\
TC & 24 & selective non singleton & 0.5014 & 0.4952 & $+0.0062$\\
TC & 24 & full & 0.9981 & 0.9970 & $+0.0011$\\
\bottomrule
\end{tabular}

\caption{Descriptive common-structural-readout effects by reference-core class. Within-class means average graphs/initializations and then twelve collections. These are inspected strata, not extra primary endpoints.}
\label{tab:gpu-core-classes}\end{table}

Family means are positive for all five generators on the primary endpoints, but heterogeneous. UC at $n=48$ ranges from a 0.0007 gain for random tournaments to 0.0688 for rank mixtures. TC at $n=24$ ranges from 0.0010 for random to 0.0870 for rank mixtures. Near-ceiling full-core cases and large singleton gains constrain interpretation as selective cyclic-set recovery. Full family/class and selection-rule results accompany these tables in the source package.

\begin{table}[ht]\centering\small
\begin{tabular*}{\linewidth}{@{\extracolsep{\fill}}lrrrrr@{}}
\toprule
Core & $n$ & F1 & Exact (\%) & Selected & Reference\\
\midrule
UC & 24 & 0.6835 & 5.76 & 13.16 & 9.66\\
UC & 48 & 0.4842 & 1.79 & 29.34 & 13.13\\
TC & 24 & 0.5014 & 1.89 & 21.08 & 8.48\\
\bottomrule
\end{tabular*}

\caption{STE structural readout on selective non-singleton targets only ($1<|C|<n$), with the original development-selected absolute cutoff. Metrics average the saved per-collection stratum means over twelve collections. Exact is complete-set recovery; selected and reference are mean cardinalities. These are inspected descriptive strata, not additional primary tests.}
\label{tab:selective-recovery}\end{table}

\subsection{Count baselines and diagnostic targets}
Count estimators use the same development/test observations as learning. They include Jeffreys-smoothed point probabilities, independent Beta$(1/2,1/2)$ posterior hard cores, native hard graph solutions, BTL (500 iterations, regularization 0.001), Rank Centrality, HodgeRank, Copeland, smooth Copeland, and win rate. Posterior draws number 512. GFM maximizes F1 for the finite joint hard-membership distribution, including the empty set; it supplies no real truth. Absolute/relative cutoffs are learned on development, with gap and native controls separate. All automatic outputs withhold true core size. Latent-$P$ and oracle-size rows are explicitly diagnostic and excluded from automatic superiority comparisons.

\begin{table}[ht]\centering\small
\begin{tabular}{@{}lrrrrrrrr@{}}
\toprule
Target & $n$ & STE-LSE & Post. & RC & BTL & Copeland & Smooth C. & All\\
\midrule
UC & 12 & 0.6303 & 0.6717 & 0.6109 & 0.5527 & 0.5439 & 0.5437 & 0.4605\\
UC & 24 & 0.5459 & 0.4583 & 0.5386 & 0.4538 & 0.4383 & 0.4382 & 0.4069\\
UC & 48 & 0.4269 & 0.3692 & 0.4656 & 0.3738 & 0.3678 & 0.3692 & 0.3663\\
UC & 64 & 0.3930 & 0.3569 & 0.4424 & 0.3481 & 0.3454 & 0.3485 & 0.3562\\
TC & 12 & 0.6266 & 0.6033 & 0.6244 & 0.5667 & 0.5687 & 0.5703 & 0.5055\\
TC & 24 & 0.4669 & 0.4290 & 0.5494 & 0.4699 & 0.4727 & 0.4763 & 0.4282\\
TC & 48 & 0.3914 & 0.3869 & 0.4733 & 0.3923 & 0.4102 & 0.4211 & 0.3869\\
TC & 64 & 0.3807 & 0.3789 & 0.4498 & 0.3691 & 0.3955 & 0.4080 & 0.3789\\
\bottomrule
\end{tabular}

\caption{Count-only descriptive F1 using absolute development calibration; all-selected is native. RC is Rank Centrality. These observations are reused from the learning collections and do not form an independent confirmation.}
\label{tab:gpu-counts}\end{table}

BTL convergence flags pass throughout. Rank Centrality reaches its stopping limit in 2 of 13,824 distinct synthetic development/test cases, so not every secondary spectral fit is certified converged. Independent edge-posterior draws may generate different joint cores from a latent dependent model; finite Monte Carlo precision is not posterior calibration.

\subsection{Recorded-data eligibility, source identity, and calibration}\label{app:gpu-real-inventory}
The 462 raw PrefLib profiles come from a fixed eligibility screen for selective non-singleton candidates, supplemented with fixed singleton controls. The source package preserves source attribution and data integrity information. Eligibility is conditional on this screen, not random selection from human preference populations. The strict selective UC cohort has two profiles in source 00007, one in 00019, two in 00049, and seven in 00053. These twelve profiles across four sources pass the frozen numerical minimum of ten profiles and three sources, but that gate does not verify human-preference provenance. Tied/partial relation cohorts are secondary; they do not inherit strict-tournament guarantees.

Source 00007 contains Electoral Reform Society election ballots \citep{preflibers}; source 00019 contains Oakland election ballots \citep{oneill2013openstv}. Source 00049 instead encodes completion-time rankings in multi-lap speed-skating/cycling competitions, and source 00053 encodes Formula 1 lap-time rankings \citep{boehmer2022collecting}. Each sports ``vote'' is a lap, not a respondent preference. The provider descriptions explicitly identify these constructions. Thus the primary cohort contains three election profiles and nine sports profiles (Table~\ref{tab:gpu-real-inventory}); the protocol's intended minimum-sized human validation is not fulfilled. The full eleven-series manifest also includes other ordinal data types, so 462 is a count of input profiles, not human elections.

\begin{table}[ht]\centering\small
\begin{tabular}{@{}lllrrrr@{}}
\toprule
Source & Profile file & Record type & $n$ & Records & UC size & TC size\\
\midrule
00007 & \texttt{00007-00000068.soi} & Ballots & 11 & 902 & 3 & 5\\
00007 & \texttt{00007-00000086.soi} & Ballots & 4 & 279 & 3 & 3\\
00019 & \texttt{00019-00000004.toi} & Ballots & 8 & 28,660 & 3 & 6\\
00049 & \texttt{00049-00000288.soc} & Laps & 15 & 77 & 3 & 4\\
00049 & \texttt{00049-00000479.soc} & Laps & 15 & 51 & 12 & 15\\
00053 & \texttt{00053-00000094.soi} & Laps & 22 & 53 & 3 & 4\\
00053 & \texttt{00053-00000158.soi} & Laps & 19 & 59 & 3 & 3\\
00053 & \texttt{00053-00000268.soi} & Laps & 24 & 70 & 3 & 4\\
00053 & \texttt{00053-00000309.soi} & Laps & 21 & 53 & 3 & 13\\
00053 & \texttt{00053-00000342.soi} & Laps & 16 & 52 & 3 & 16\\
00053 & \texttt{00053-00000382.soi} & Laps & 17 & 71 & 3 & 5\\
00053 & \texttt{00053-00000420.soi} & Laps & 20 & 51 & 3 & 3\\
\bottomrule
\end{tabular}

\caption{Every primary strict selective UC profile. ``Records'' is the raw PrefLib multiplicity total, called ``voters'' by its generic format; its meaning is ballots or laps according to source. Complete pair coverage concerns the aggregate decisive matrix, not whether every individual order is complete.}
\label{tab:gpu-real-inventory}\end{table}

Whole-record subsamples without replacement preserve multiplicities and within-record pair dependence. Ten streams per profile use nested fractions 0.1/0.25/0.5; pair-missingness overlays 0.25/0.5 apply only at fraction 0.25, producing fifty cases/profile and 23,100 total. For elections these records are ballots; for the sports sources they are laps. Uniform lap subsampling evaluates reconstruction of the recorded race, not temporal generalization, future-race outcomes, or independent human judgments. It does not remove serial dependence among laps. Unranked and tied alternatives do not create artificial decisive wins. The primary reference is the complete finite-profile conditional-decisive empirical core. Numerically neutral missing/tie-only probabilities equal 0.5, but their evidence states remain distinct in $W,T$ and the mask. Observed 50/50 decisive pairs have different posterior concentration from unobserved ones.

For a held-out source, calibration uses strict selective, nonduplicate profiles in all other sources. For each method, target, and rule, it averages the ten resampling streams at all three fractions without extra missingness, with equal-source weighting; equal numbers of cases/profile also give equal profile weights within source. It selects one common cutoff that transfers to every fraction and missingness setting, rather than fitting fraction-specific cutoffs. The $11\times2\times8\times3=528$ calibration records reflect sources, targets, score methods, and rules. Test-source reference labels never choose its cutoff. Profile means average ten repetitions before source aggregation. Four source groups, not twelve profiles or 120 primary subsamples, determine the primary macro comparison. Calibration sets overlap between source folds; together with purposive source inclusion, this limits treating the displayed source-level $t$ interval as a population confidence guarantee.

\begin{table}[ht]\centering\small
\begin{tabular}{@{}lrrr@{}}
\toprule
Source & STE-LSE & Posterior & $\Delta$ F1\\
\midrule
00007 & 0.7967 & 0.8258 & $-0.0292$\\
00019 & 0.6433 & 0.9500 & $-0.3067$\\
00049 & 0.5290 & 0.6984 & $-0.1693$\\
00053 & 0.6273 & 0.6552 & $-0.0279$\\
\bottomrule
\end{tabular}

\caption{Primary UC F1 by held-out source; STE and posterior both use absolute source-excluded calibration. Sources 00007/00019 are elections; 00049/00053 are sports lap rankings. Every source favors posterior. These four purposively assembled clusters limit population inference.}
\label{tab:gpu-real-sources}\end{table}

STE minus posterior mean F1 is $-0.1333$, with descriptive $t_3$ interval $[-0.3454,0.0788]$. No precise population conclusion follows from this small, mixed source count. Native GFM reaches 0.6920 and posterior-half 0.6450, showing the importance of decision rules even under the same posterior. A post hoc election-only description averages the original source means for 00007 and 00019 equally: STE/posterior F1 is 0.7200/0.8879 on three profiles. Original source-excluded cutoffs, including sports-source calibration inputs, are retained. This is neither a replacement primary endpoint nor an independently calibrated human-only experiment. No extra significance test or population interval is assigned to this two-source subset.

The secondary strict selective TC cohort contains ten profiles from four sources; STE/Rank Centrality/posterior F1 is 0.7005/0.8347/0.5925 at the analogous setting. Complete rule/fraction/missingness/source summaries, including partial/tied cases, are retained. Rank Centrality stopping flags fail in 559 of 23,100 overall recorded-profile cases (2.42\%), but none of the 120 UC or 100 TC strict primary-setting cases; BTL flags all pass. Independent-edge posterior sampling is a working model despite within-record and possible serial dependence; its inclusion frequencies do not supply a population confidence certificate.

\subsection{TC depth, gradients, checkpoint parity, and actual memory}
Diagnostics span $n=8,16,32,64,96$, ordered/planted/long-path families, three independent graphs/cell, and temperatures 0.01/0.035/0.1. TC depths are the distinct values among $1,\lceil\log_2 n\rceil,\lceil n/2\rceil,n-1$, each plain and checkpointed; UC is evaluated separately. This gives 1,215 cells. Gradients are of clipped true-label BCE with respect to upper-triangle logits, not a universal operator Jacobian analysis. Two warmups and five measured repetitions retain forward/backward arrays, outputs, probabilities, targets, and gradients.

\begin{table}[ht]\centering\small
\begin{tabular}{@{}lrrr@{}}
\toprule
$n$ & Ordered & Planted & Long path\\
\midrule
8 & 0.0486 & 0.0963 & 0.2180\\
16 & 0.0733 & 0.1268 & 0.5181\\
32 & 0.0809 & 0.1313 & 0.7389\\
64 & 0.0932 & 0.1417 & 0.8213\\
96 & 0.0874 & 0.1382 & 0.8475\\
\bottomrule
\end{tabular}

\caption{Full-depth TC mean absolute score error at the default coupled temperature 0.035. Each value averages three diagnostic graphs and their vertices. Full long-path cores do not make this a selective-set classification test.}
\label{tab:gpu-tc-error}\end{table}

The full-depth long-path errors at $n=96$ are 0.8387/0.8475/0.6759 for temperatures 0.01/0.035/0.1. Simply decreasing a coupled temperature on this grid does not establish a useful size-dependent strategy. All raw gradients are finite. The largest fraction with absolute gradient below $10^{-8}$ is 79.03\%; the largest exactly-zero fraction is 67.86\%, in low-temperature cells. No whole saved gradient vector has every entry zero, though float32 norm calculations can underflow. A finite BCE gradient and a perturbation upper bound do not establish healthy gradients for all losses, all graphs, or all sizes.

The 540 saved plain/checkpoint output/gradient pairs are bit-identical. At $n=96$, full depth, temperature 0.035, means over all three families and three graphs give 414.9 MiB and 42.7 ms for plain autograd, versus 94.5 MiB and 101.0 ms with checkpointing, for peak allocated memory and forward-plus-backward time respectively. About 64 MiB of resident allocation precedes these measurements; the peaks are total allocated memory, not pure operator increments. Warmups/repetitions and allocator state limit precision and portability.

The GPU code vectorizes each max--min TC step into a $B\times n\times n\times n$ candidate array. Without checkpointing, backward retains order $BKn^3$ candidate intermediates plus order $BKn^2$ path states. Non-reentrant step checkpointing retains order $BKn^2$ boundaries and log-length accumulator states while recomputing order $Bn^3$ step workspace. At full $K=n-1$, the respective dominant theoretical orders are $Bn^4$ and $Bn^3$. The canonical streamed-column forward algorithm can use order $Bn^2$ workspace, but this GPU implementation does not claim that forward memory bound. UC streams one target at a time, with order $Bn^2$ forward intermediates and order $Bn^3$ retained across targets for backward. These are implementation analyses, separate from measured allocator peaks.

\subsection{Complete neural-training-step costs}
Benchmarks use batch four, five warmups, and twenty timed steps including model forward, target loss, backward, clipping, and Adam. At a given size all objectives share the same observed batch; each objective/target starts a fresh model and optimizer from the same initialization seed, and repeatedly updates on that batch. TC STE is checkpointed; no mixed precision or TF32 is enabled. Equal epochs in the learning experiment are therefore not equal compute. Larger-size step benchmarks are distinct from full model training at $n=12$.

\begin{table}[ht]\centering\small
\begin{tabular}{@{}lrrrl@{}}
\toprule
Target & $n$ & Aux. (ms) & STE (ms) & Ratio\\
\midrule
TC & 12 & 4.37 & 15.98 & 3.66$\times$\\
UC & 12 & 4.35 & 17.86 & 4.10$\times$\\
TC & 24 & 4.37 & 28.60 & 6.54$\times$\\
UC & 24 & 4.38 & 31.41 & 7.17$\times$\\
TC & 48 & 4.41 & 53.79 & 12.19$\times$\\
UC & 48 & 4.40 & 58.82 & 13.38$\times$\\
TC & 64 & 4.38 & 70.40 & 16.07$\times$\\
UC & 64 & 4.39 & 78.54 & 17.90$\times$\\
TC & 96 & 4.43 & 104.40 & 23.58$\times$\\
UC & 96 & 4.44 & 117.06 & 26.36$\times$\\
\bottomrule
\end{tabular}

\caption{Recorded L40S complete batch-four training-step medians. Ratios compare STE structural loss with the auxiliary own-head loss on the common base model. They exclude data preparation and posterior inference and are not fit-to-convergence timings.}
\label{tab:gpu-cost}\end{table}

These diagnostics do not validate a size-dependent temperature strategy. Any such strategy would need development selection followed by fresh confirmation, with hard-score fidelity and gradient degradation assessed together. Improving approximation alone would not establish a better training layer or resolve direct auxiliary-head competitiveness.

\section{Equal-budget UC follow-up}\label{app:uc-equal-compute}
The equal-budget follow-up asks whether structural UC supervision improves selective-set recovery over direct membership heads when each complete system receives the same training and development allocation. It uses fresh collections and a prospectively frozen protocol, without external preregistration. Its two primary comparisons remain separate from the matched-epoch study's three-comparison family (Appendix~\ref{app:ibm-gpu}). The operators, observation model, generators, pair estimator, and optimizer are unchanged. This experiment concerns UC only.

\subsection{Data, objectives, and native readouts}
A separate frozen root seed defines twelve independent collections. Each contains 512 training and 128 development graphs at $n=12$, with 256 held-out graphs at each $n=24$ and $n=48$. The ordered, planted, separated, rank-mixture, and random families cycle by graph index as in Appendix~\ref{app:ibm-gpu}. The same count intensity 10, planned pair missingness 0.25, tie probability 0.20, and heterogeneous count construction apply. Latent graphs have strict majority relations; observation ties and missingness are retained separately in the decisive-win matrix $W$ and tied-count matrix $T$. Two independently coded hard oracles agree on the latent UC labels. All natural core classes remain in the generated data: the selective evaluation stratum is neither an acceptance rule for generating graphs nor a source of test-time cardinality information.

Three paired initializations per collection share encoder/pair parameters and epoch-specific minibatch permutations. The four complete systems are pairwise likelihood with structural UC readout; pairwise likelihood plus ordinary membership-head BCE with its direct readout; pairwise likelihood plus relational membership-head BCE with its direct readout; and pairwise likelihood plus STE UC BCE with structural readout. The relational control uses the four relation summaries, including two-step paths, described in Appendix~\ref{app:ibm-gpu}. Each supervised system considers $\lambda\in\{0.1,1\}$; pairwise-only has one candidate. Thus the run contains $12\times3\times(1+2+2+2)=252$ trained candidates and 144 selected complete systems. The pair architecture is matched, while active parameter counts differ between structural systems and direct-head systems.

Training uses Adam at 0.001, batch size 16, norm-5 gradient clipping, and the unchanged UC temperature $\tau=\gamma=0.035$. Pairwise likelihood is normalized by decisive comparisons in the minibatch; target BCE averages graphs and vertices. As in the original implementation, probabilities and BCE scores are clipped to $[10^{-6},1-10^{-6}]$ for numerical stability. This clipping can zero a saturated loss gradient and is not a modification of the mathematical operator's definition. Model inference takes only $W,T$ as inputs. Training labels supervise the appropriate losses; latent probabilities and test labels enter truth construction and diagnostic evaluation, not the prediction function.

\subsection{Charged time and development-only selection}
Each supervised system receives two 110-second candidate allocations; pairwise-only receives one 220-second allocation. Consequently every complete system receives 220 seconds per collection and initialization, totaling 31,680 allocated seconds. The charged synchronized wall clock includes training, native-readout development prediction, cutoff selection, checkpoint/development serialization, and in-loop progress bookkeeping. In particular, development evaluation of a direct-head baseline does not charge it for the unused UC operator. Data generation, initialization/data transfer, common warmup, post-budget integrity checks, and held-out evaluation are recorded outside this allocation. Equal allocation is not equal epoch count or equal update count, nor a claim of equal total application runtime.

Five checkpoints target fractions 0.025, 0.1, 0.25, 0.5, and 1 of a candidate allocation. At a synchronized minibatch boundary the scheduler reserves the largest previously measured checkpoint cost before triggering the next checkpoint; the first checkpoint has zero reserve. Evaluation and serialization count toward the same candidate budget. A prospectively fixed validity gate requires every candidate and complete-system charged total to remain within $\pm2\%$ of its allocation; a violation would withhold the confirmatory equal-budget inference. The measured charged total is 31,713.35 seconds, with maximum absolute candidate deviation 0.510\% and complete-system deviation 0.189\%. All 252 candidates and 144 systems pass, with intervening training at all five checkpoints.

Each checkpoint is calibrated on development cases with $1<|\mathrm{UC}|<12$, maximizing mean per-case F1 for its native readout. The 43 absolute cutoffs comprise $0,0.025,\ldots,1$ and the all/none sentinels $-10^{-6}$ and $1.000001$. Sets use the inclusive rule $s_a\ge h$ without an oracle-size constraint. Across candidates, ties prefer an earlier checkpoint, then smaller $\lambda$, then a cutoff nearer 0.5, then the smaller cutoff. This permits a fast baseline to select an early model rather than obliging it to use its final state. All choices for the three initializations of a collection are frozen before its held-out graphs are generated. Test outcomes choose neither $\lambda$, checkpoint, cutoff, nor operator temperature.

\subsection{Primary comparison and descriptive diagnostics}
The primary stratum is selective non-singleton UC, $1<|\mathrm{UC}|<n$, at $n=24$. For each collection and initialization, F1 first averages its eligible test graphs; the three initialization means are then averaged within the collection. The two contrasts are STE structural F1 minus ordinary auxiliary direct F1 and minus relational auxiliary direct F1. Only the twelve paired collection differences enter inference. Intervals use Student's $t_{11}$ distribution and are approximate, individual 95\% intervals rather than simultaneous intervals. The two-sided sign-flip test enumerates all $2^{12}=4096$ patterns and requires exchangeability under independent sign reversals under its null. Holm adjustment covers exactly these two contrasts. Collection independence and the approximate normality of collection differences remain relevant assumptions; individual graphs and initializations are not treated as independent replicates.

\begin{table}[ht]\centering\small
\begin{tabular}{lrrrrr}
\toprule
Objective / readout & $n$ & F1 (all) & F1 (selective) & Exact (\%) & Size \\
\midrule
Pair / structural & 24 & 0.5252 & 0.6245 & 3.76 & 16.64 \\
Pair / structural & 48 & 0.4104 & 0.4018 & 0.91 & 39.14 \\
Aux. / head & 24 & 0.5865 & 0.6414 & 3.31 & 16.57 \\
Aux. / head & 48 & 0.4633 & 0.4280 & 0.31 & 35.04 \\
Rel. aux. / head & 24 & 0.5879 & 0.6416 & 2.87 & 16.56 \\
Rel. aux. / head & 48 & 0.4651 & 0.4263 & 0.23 & 35.17 \\
STE / structural & 24 & 0.5484 & 0.6281 & 3.82 & 16.85 \\
STE / structural & 48 & 0.4276 & 0.4076 & 0.33 & 38.04 \\
\bottomrule
\end{tabular}

\caption{Complete-system UC follow-up. Each method uses its own development-selected native readout. All-class F1 uses the natural mixture; selective columns restrict to $1<|\mathrm{UC}|<n$. Exact is complete-set recovery and selected size is predicted cardinality on that stratum. Means average graphs, then three initializations, then twelve collections. Only $n=24$ selective F1 is primary.}
\label{tab:uc-equalcompute-systems}\end{table}

\begin{table}[ht]\centering\small
\begin{tabular}{lrrrrr}
\toprule
Baseline & $\Delta$ & 95\% CI & $p$ & $p_{\rm Holm}$ & Pos. \\
\midrule
Ordinary head & $-0.01329$ & $[-0.03286,+0.00628]$ & 0.1621 & 0.3193 & 5/12 \\
Relational head & $-0.01345$ & $[-0.03300,+0.00611]$ & 0.1597 & 0.3193 & 3/12 \\
\bottomrule
\end{tabular}

\caption{The two prespecified primary contrasts on selective non-singleton UC at $n=24$. Differences are STE minus the indicated direct-head system; intervals are individual paired $t_{11}$ intervals. Holm adjustment applies to the two exact sign-flip tests.}
\label{tab:uc-equalcompute-primary}\end{table}

\begin{figure}[ht]\centering
\includegraphics[width=\linewidth]{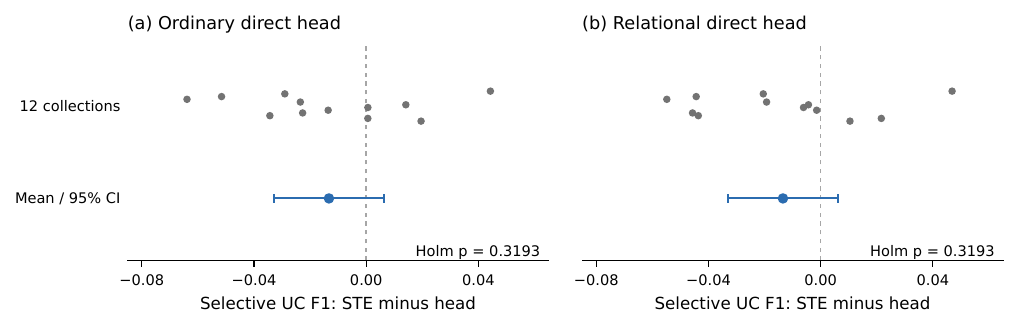}
\caption{All twelve independent collection differences for each primary contrast. Each small point averages the three paired initializations within one collection. The accent point and interval show the mean difference and individual, unadjusted 95\% $t_{11}$ interval. Zero marks equal mean F1; positive differences favor STE. No collection is excluded.}
\label{fig:uc-equalcompute-effects}\end{figure}

STE's primary mean F1 is 0.6281, compared with 0.6414 for the ordinary direct head and 0.6416 for the relational direct head. Corresponding differences are $-0.01329$ (95\% interval $[-0.03286,0.00628]$) and $-0.01345$ ($[-0.03300,0.00611]$); both Holm-adjusted $p$-values are 0.3193. STE has positive collection differences in five of twelve ordinary-head comparisons and three of twelve relational-head comparisons. This study therefore does not demonstrate an equal-budget STE advantage over either direct-head baseline. Negative point estimates do not establish inferiority; intervals crossing zero do not establish equivalence or noninferiority. The earlier common-structural-readout, matched-epoch result concerns a different comparison and does not resolve this complete-system question. Compute allocation, baseline readout, development criterion, and primary stratum differ between studies; their change in outcomes cannot be attributed to compute alone.

The $n=48$ results, mixed-class averages, singleton/full-core strata, exact recovery, cardinalities, AP/AUROC, pairwise-only control, and generator breakdowns are descriptive secondary outputs. AP/AUROC are defined only when both positive and negative reference memberships exist, so their support differs from F1. The explicit all/none predictions are constant controls, not trained methods. Saved predicted pair matrices permit additional diagnostics of relation quality, independently of development-selected set cutoffs.

\begin{table}[ht]\centering\small
\begin{tabular}{lrrrrrr}
\toprule
Objective / readout & $n$ & Pair MSE & Brier & Log loss & Hard UC F1 & Ties \\
\midrule
Pair / structural & 24 & 0.0417 & 0.2394 & 0.9518 & 0.5231 & 0.0252 \\
Pair / structural & 48 & 0.0416 & 0.2401 & 0.9541 & 0.4143 & 0.1407 \\
Aux. / head & 24 & 0.0427 & 0.2404 & 0.8550 & 0.4911 & 0.0002 \\
Aux. / head & 48 & 0.0421 & 0.2407 & 0.8558 & 0.4084 & 0.0017 \\
Rel. aux. / head & 24 & 0.0380 & 0.2357 & 0.8180 & 0.5703 & 0.0005 \\
Rel. aux. / head & 48 & 0.0366 & 0.2352 & 0.8150 & 0.4668 & 0.0016 \\
STE / structural & 24 & 0.0289 & 0.2266 & 0.6813 & 0.6186 & 0.0001 \\
STE / structural & 48 & 0.0283 & 0.2269 & 0.6816 & 0.5614 & 0.0004 \\
\bottomrule
\end{tabular}

\caption{Post hoc, descriptive all-class relation-quality and hard-decoding diagnostics for the selected systems. Hard UC uses each stored float32 direction $\widehat P_{ab}>1/2$ without the learned membership cutoff. Pair errors use latent simulated truth unavailable to the predictor. Ties count upper-triangle probabilities equal to $1/2$ per graph; numerical reciprocity can leave the reverse direction above $1/2$. These are not additional primary endpoints.}
\label{tab:uc-equalcompute-secondary}\end{table}

For latent probability $q_{ab}$ and prediction $\widehat P_{ab}$, pair MSE averages $(\widehat P_{ab}-q_{ab})^2$ over unordered pairs. Expected Bernoulli Brier score adds $q_{ab}(1-q_{ab})$; expected log loss uses $-q_{ab}\log\widehat P_{ab}-(1-q_{ab})\log(1-\widehat P_{ab})$, with predictions clipped to $[10^{-6},1-10^{-6}]$. These diagnose relation quality against latent simulator truth. Hard UC decoding applies the frozen cover definition independently to each stored direction $\widehat P_{ab}>1/2$, without imposing an additional orientation. If both directions equal $1/2$, neither edge is included, so the decoded graph need not be a strict tournament. The archived upper-triangle tie count is not necessarily the number of such two-direction ties because of float32 rounding.

STE supervision yields lower pair MSE and higher hard-decoded F1 at both sizes. At $n=24$, MSE is 0.0289 versus 0.0427/0.0380 for ordinary/relational auxiliary supervision, and hard UC F1 is 0.6186 versus 0.4911/0.5703. At $n=48$, hard UC F1 is 0.5614 versus 0.4084/0.4668. The STE hard-decoded means also exceed the direct heads' native all-class F1: 0.5865/0.5879 at $n=24$ and 0.4633/0.4651 at $n=48$. These are post hoc complete-system descriptions using the already selected checkpoints, with no additional significance test. Hard decoding avoids a transferred membership cutoff, but its all-class comparison does not replace the prespecified selective-core endpoint. The results distinguish a useful learned relation from successful calibration of its soft scores across sizes.

\begin{table}[ht]\centering\scriptsize
\begin{tabular*}{\linewidth}{@{\extracolsep{\fill}}lrrrrrr}
\toprule
System & $n$ & Native F1 & Hard UC F1 & Native exact (\%) & Hard exact (\%) & Pair MSE \\
\midrule
Pair / structural & 24 & 0.6245 & 0.6035 & 3.76 & 1.73 & 0.0412 \\
Pair / structural & 48 & 0.4018 & 0.3960 & 0.91 & 0.38 & 0.0410 \\
Aux. / head & 24 & 0.6414 & 0.5816 & 3.31 & 1.09 & 0.0420 \\
Aux. / head & 48 & 0.4280 & 0.3967 & 0.31 & 0.54 & 0.0414 \\
Rel. aux. / head & 24 & 0.6416 & 0.6420 & 2.87 & 1.61 & 0.0402 \\
Rel. aux. / head & 48 & 0.4263 & 0.4445 & 0.23 & 0.84 & 0.0377 \\
STE / structural & 24 & 0.6281 & 0.6761 & 3.82 & 1.81 & 0.0294 \\
STE / structural & 48 & 0.4076 & 0.5480 & 0.33 & 1.86 & 0.0283 \\
\bottomrule
\end{tabular*}

\caption{Post hoc diagnostics restricted to selective non-singleton cores. Native and hard readouts use the same selected weights; checkpoint selection was not repeated for hard decoding. Exact is complete-set recovery. Pair MSE is measured against latent simulated truth. Means average twelve collections, each averaging graphs and three initializations. No additional inferential tests are claimed.}
\label{tab:harddecode-selective}\end{table}

The positive F1 pattern also holds within selective non-singleton cores (Table~\ref{tab:harddecode-selective}), so it is not explained solely by singleton or full-core cases. STE hard F1 exceeds the ordinary and relational heads' native F1 by 0.0347/0.0346 at $n=24$ and 0.1200/0.1217 at $n=48$. This does not imply successful complete-set reconstruction: at $n=24$, STE hard exact recovery is 1.81\%, versus 3.31/2.87\% for the native heads. These inspected secondary comparisons suggest a prospective test with development selection aligned to each complete system's readout; they are not such a test themselves.

\subsection{Execution evidence and computational replay}
The cloud L40S environment uses Python 3.11.17, PyTorch 2.6.0+cu124, CUDA runtime 12.4, NumPy 1.26.4, SciPy 1.13.1, and pandas 2.2.3; float32 training disables AMP and TF32. All candidate models, training inputs, and non-scalar optimizer states are GPU-resident, with positive CUDA training time. The full evidence retains initial states, all 1,260 checkpoints, training histories, raw datasets, selection records, and predictions. Unlike the matched-epoch review in Appendix~\ref{app:ibm-gpu}, verification here uses the selected weights as well as saved predictions.

Archive and source-integrity checks pass. Regeneration reproduces observations, labels, identifiers, and family metadata for all 48 datasets exactly; latent probabilities differ by at most $2.23\times10^{-16}$ across environments. Independently coded hard oracles verify all 13,824 latent graphs. Checkpoints differ from their initial states with increasing update counts and finite recorded updates. Development recalibration reconstructs every checkpoint evaluation and all 144 system choices exactly. CPU inference from the 144 selected states reproduces all 73,728 held-out predictions within cross-platform rounding tolerance, with maximum pair-probability difference $2.93\times10^{-6}$. Reconstructed scores, hard sets, diagnostics, metrics, and primary tests agree with preserved outputs.

These checks support computational consistency across data generation, training, selection, prediction, and analysis. They constitute saved-state replay rather than retraining the optimizer trajectory or replication by another researcher. The findings concern the tested architecture, temperature, generator mixture, and allocation; they do not establish universal rankings or human-preference transfer.

\section{Prospective hard-UC confirmation}\label{app:hard-confirm}
This separate synthetic study tests whether soft-UC supervision improves a prospectively selected hard-UC complete system over native membership heads under equal measured training/development allocation. Its protocol and two primary contrasts were frozen before this new run, without external preregistration. It was motivated by the post hoc diagnostic in Appendix~\ref{app:uc-equal-compute}; its fresh collections and inference family are not pooled with that study or the matched-epoch comparisons. The previous soft-set null remains unchanged.

\subsection{Data, training, and decoder-aligned selection}
Twelve independent collections each contain 512 training and 128 development graphs at $n=12$, and 256 tests at each of $n=24,48$. The five generator families, heterogeneous observation counts, planned missingness 0.25, and tie probability 0.20 match Appendix~\ref{app:uc-equal-compute}. There are 6,144 distinct held-out graphs and three paired initializations per collection. The pair architecture, Adam learning rate 0.001, batch size 16, gradient clipping 5, float32 execution, and canonical UC temperature 0.035 are unchanged. Paired initial states and minibatch-order seeds are shared; equal time permits different update counts. Candidate order rotates across systems.

Six systems are independently trained and timed: pairwise-only with hard UC; ordinary and relational auxiliary supervision with native heads; STE soft-UC supervision with hard UC; and separately trained ordinary/relational auxiliary controls with hard UC. Auxiliary controls retain their membership-head training loss. STE differentiates the original soft operator, not hard UC. All hard systems select through their own hard decoder, rather than reusing the native-head choices. This controls decoder-specific selection; structural-loss alignment remains part of the treatment.

Hard UC applies the unchanged covering oracle to each stored direction $\widehat P_{ab}>1/2$. There is no membership-score cutoff, per-test-size calibration, or oracle cardinality. Float32 reciprocal probabilities can round asymmetrically at $1/2$; if neither direction exceeds $1/2$, both arcs are omitted. The resulting graph can be incomplete. Upper-triangle exact ties, genuinely unoriented pairs, and one-sided reciprocal rounding are recorded separately.

Each supervised system fits two loss weights, $\lambda\in\{0.1,1\}$, for 110 seconds each; pairwise-only fits one candidate for 220 seconds. Five checkpoints target allocation fractions 0.025, 0.1, 0.25, 0.5, and 1. Selection maximizes development per-graph F1 on selective non-singleton UC at $n=12$. Native heads also optimize the original 43 absolute cutoffs. Exact ties prefer an earlier checkpoint, then smaller $\lambda$, and for heads a cutoff nearer 0.5 then smaller cutoff. Every system's three initialization choices are frozen and hashed before that collection's test graphs are generated.

\subsection{Budget and inference validity}
The study contains 396 candidate fits and 216 selected complete systems. Each system receives 220 charged seconds per collection/initialization, totaling 47,520 allocated seconds. Charged work includes synchronized training, native-decoder development prediction/selection, checkpoint serialization, in-loop bookkeeping, and final winner selection/freezing. Shared final serialization cost is apportioned equally across systems. Data generation, initialization/transfers, common warmup, post-budget integrity bookkeeping, and held-out evaluation are recorded separately. The checkpoint scheduler reserves the largest prior checkpoint cost at each synchronized minibatch boundary. All candidates and complete-system totals satisfy the prospectively fixed $\pm2\%$ gate; charged complete-system time totals 47,574.23 seconds. This satisfies the specified allocation tolerance, not total application runtime or active parameter count.

The two primary contrasts are STE-hard minus ordinary native-head F1 and minus relational native-head F1 on selective non-singleton UC at $n=24$. Eligible graphs average within each initialization, then the three initializations within collection. All twelve independent paired collection differences enter inference. Individual 95\% intervals use $t_{11}$ and approximate normality of collection differences; two-sided exhaustive sign flipping enumerates all 4,096 patterns and assumes null exchangeability under independent sign reversals. Holm adjustment covers exactly these two contrasts. Secondary sizes, hard controls, exact recovery, relation errors, and cardinalities are descriptive.

\begin{table}[ht]\centering\small
\begin{tabular*}{\linewidth}{@{\extracolsep{\fill}}lrrrr@{}}\toprule
Native comparator & Difference & 95\% CI & Raw $p$ & Holm $p$\\\midrule
Ordinary head & 0.03493 & $[0.02383,0.04602]$ & 0.0004883 & 0.0009766\\
Relational head & 0.03843 & $[0.02763,0.04922]$ & 0.0004883 & 0.0009766\\
\bottomrule\end{tabular*}

\caption{The separate frozen hard-UC primary family. Differences are STE-hard minus native heads at selective $n=24$. Both comparisons retain all twelve collection effects. Intervals are individual paired $t_{11}$ intervals; Holm covers the two exhaustive sign-flip tests.}\label{tab:hard-confirm-primary}
\end{table}

\begin{table}[ht]\centering\scriptsize
\begin{tabular*}{\linewidth}{@{\extracolsep{\fill}}llrrrrrr@{}}\toprule
$n$ & System / readout & All F1 & Selective F1 & Exact (\%) & Selected & Reference & Pair MSE\\\midrule
24 & Pair / hard & 0.5131 & 0.5865 & 1.45 & 16.72 & 9.57 & 0.0407\\
24 & Ordinary / native & 0.5902 & 0.6327 & 2.02 & 16.34 & 9.57 & 0.0445\\
24 & Relational / native & 0.5832 & 0.6292 & 2.53 & 16.43 & 9.57 & 0.0398\\
24 & STE / hard & 0.6130 & 0.6676 & 1.16 & 10.75 & 9.57 & 0.0354\\
24 & Ordinary / hard & 0.5231 & 0.5972 & 1.41 & 15.87 & 9.57 & 0.0386\\
24 & Relational / hard & 0.5621 & 0.6237 & 1.44 & 14.21 & 9.57 & 0.0402\\
\midrule
48 & Pair / hard & 0.4170 & 0.4012 & 0.36 & 40.31 & 13.13 & 0.0405\\
48 & Ordinary / native & 0.4725 & 0.4309 & 0.00 & 35.22 & 13.13 & 0.0434\\
48 & Relational / native & 0.4667 & 0.4318 & 0.48 & 35.17 & 13.13 & 0.0377\\
48 & STE / hard & 0.5272 & 0.5159 & 1.75 & 25.99 & 13.13 & 0.0345\\
48 & Ordinary / hard & 0.4293 & 0.4121 & 1.02 & 38.54 & 13.13 & 0.0378\\
48 & Relational / hard & 0.4554 & 0.4324 & 1.10 & 35.57 & 13.13 & 0.0378\\
\bottomrule\end{tabular*}

\caption{All six independently timed systems. All-class F1 uses the natural mixture; remaining metrics use selective non-singleton cores. Exact is complete-set recovery (\%). Pair MSE uses latent simulated truth unavailable to the predictor. Means average graphs, initializations, and twelve collections. $n=48$ and the common-hard controls are descriptive.}\label{tab:hard-confirm-systems}
\end{table}

STE-hard improves primary F1 by 0.03493/0.03843 over ordinary/relational native heads (Table~\ref{tab:hard-confirm-primary}); all twelve differences are positive for each contrast. The common-hard controls are also below STE-hard descriptively, and STE has lower latent pair MSE. These findings support relation learning and the tested complete hard-UC pipeline. They do not establish general architectural superiority or erase the earlier soft-set null. Exact selective recovery is only 1.16\% at $n=24$, below native-head recovery of 2.02/2.53\%. At $n=48$, STE-hard selects 25.99 alternatives against a 13.13-member reference, despite its higher F1. Successful overlap remains distinct from accurate membership and cardinality.

\subsection{Saved evidence and replay}
Production uses an NVIDIA L40S with PyTorch 2.6.0+cu124/CUDA 12.4, NumPy 1.26.4, SciPy 1.13.1, and pandas 2.2.3. TF32 and mixed precision are disabled. The archive retains observations, latent labels/probabilities, paired initial states, candidate histories and budgets, development predictions/choices, selected checkpoints, and test predictions. Selected-checkpoint CUDA replay verifies 432 prediction units containing 110,592 case outputs. An independent raw-output audit recomputes all 110,592 selected predictions, checks 216 development choices and 396 candidate/216 system timing records, and reproduces the primary statistics and all reported system summaries. Archive payload hashes and CRCs pass for both the full and review copies. Replayed stored states and independently reconstructed metrics support computational consistency, not independent retraining, optimizer convergence, formal execution attestation, or external replication.

\section{Exploratory learning from human ordinal profiles}\label{app:real-learning}
This additional study trains the existing pair model on archived human ordinal records and reconstructs each full profile's empirical UC from subsamples. It is exploratory: the cohort was selected in an earlier feasibility screen and prior estimator outcomes had been inspected. Whole-source exclusion reduces leakage but does not create a previously unseen confirmatory benchmark. No population-core truth or real-data superiority is claimed.

\subsection{Cohort and whole-source folds}
Thirty-six profiles from six PrefLib sources have 4--16 alternatives, every full-record decisive pair observed, no exact full-record majority ties, and no exact duplicate full-count matrix. The sources are ERS (9), Oakland (7), San Francisco (3), San Leandro (2), CSES (10), and Voter Autrement (5) \citep{mattei2013preflib,preflibers,oneill2013openstv,preflibcses,preflibvoterautrement}. They represent election ballots, survey-derived rankings, or online voting records; they are not all election ballots. Sports and constituency-result rankings are excluded. Only three references are selective non-singletons: two ERS and one Oakland profile. All three selective references have size three; the other 33 references are singletons. This restricted cohort does not represent all human preferences.

CSES preferences are derived from respondents' ratings of political parties/leaders; its provider requests the module releases and dataset-construction thesis \citep{csesmodule1,csesmodule2,csesmodule3,csesmodule4,csesmodule5,boehmer2023thesis}. Whole-source folds keep related respondent/year variants together. The six folds assign each source once to testing, the cyclic next source to development, and the other four to training. Full-record test-source cores are used only as references; full test counts do not enter fitting or decoder selection. Subsample counts remain the observed test inputs. The ERS-held-out fold has no selective training profile, and the Oakland-held-out fold has two. This further restricts nontrivial-set learning claims. Source series need not be independent populations.

Eight shared nested subsets sample 5\%, 10\%, and 25\% of each record's multiplicity without replacement via multivariate hypergeometric draws. Tied ranks yield tied-count observations and unranked alternatives remain unobserved. Strictness is required only of the full reference; subsample inputs can be missing or tied. All 864 recorded subsets are preserved. Training draws a source, profile, fraction, and replicate uniformly in that order from its training sources; repeated draws reuse actual recorded subsets and do not simulate new ballots. Training labels are hard UC of the full empirical reference.

\subsection{Models, selection, and count controls}
Three paired initializations train pairwise-only, ordinary/relational auxiliary, and STE soft-UC objectives. The frozen model/operator files are unchanged. Hidden width is 32, Adam learning rate 0.003, gradient clipping 5, batch size 8, float32, and temperature 0.035, with no TF32 or mixed precision. Each system receives a configured 120-second training/development/serialization allocation per fold/initialization; supervised systems split it over $\lambda\in\{0.1,1\}$. Four checkpoint opportunities are scheduled per candidate. Final work quanta can overshoot; measured costs are retained rather than asserting identical actual runtimes. There are 126 candidates, 8,640 allocated seconds, and 8,645.90 measured seconds, excluding count baselines, testing, auditing, and packaging.

Checkpoint and weight selection is separate for hard UC, structural soft UC, and native heads where available. Each decoder maximizes development macro-profile F1 at 10\% records. Soft/head cutoffs use a fixed 0.02 grid with all/none endpoints; ties prefer an earlier checkpoint, smaller weight, cutoff nearer 0.5, then smaller cutoff. Hard UC has no membership cutoff. All choices precede test model prediction. Neural methods receive cross-profile training-core labels; count-only controls do not. This comparison asks whether supervised learning adds reconstruction utility, rather than matching supervision.

Count controls include raw majority hard UC, plug-in soft UC, Jeffreys independent-edge posterior inclusion, posterior half-cutoff, posterior GFM, BTL, Rank Centrality, HodgeRank, Copeland/smooth Copeland, win rate, and all/none. They share the same subsample counts. The posterior uses 4,096 sampled tournaments per case; half-cutoff and GFM are native decisions, while score cutoffs use development-source calibration. Neural features use tied-count fractions; decisive-only posterior/ranking models do not use ties to infer a direction. The independent-edge posterior remains a working model for dependent comparisons derived from ordinal records.

\subsection{Descriptive results and limits}
The designated descriptive setting is 10\% records: resamples and initializations average within profile, profiles within source, then all six source means equally. Resamples, profiles from shared series, fractions, and initializations are not additional population replicates. There are no p-values or population confidence claims.

\begin{table}[ht]\centering\scriptsize
\begin{tabular*}{\linewidth}{@{\extracolsep{\fill}}lrrrrrr@{}}\toprule
Method / readout & F1 & Exact (\%) & Precision & Recall & Selected & Reference\\\midrule
\multicolumn{7}{l}{\textit{Learned relations and readouts}}\\
Ordinary / native head & 0.7982 & 64.51 & 0.7851 & 0.8573 & 1.165 & 1.122\\
Ordinary / hard UC & 0.8913 & 85.37 & 0.9042 & 0.8997 & 1.136 & 1.122\\
Ordinary / soft UC & 0.8564 & 72.84 & 0.8396 & 0.9211 & 1.327 & 1.122\\
Pair / hard UC & 0.8955 & 85.94 & 0.9087 & 0.9067 & 1.145 & 1.122\\
Pair / soft UC & 0.8863 & 77.71 & 0.8724 & 0.9454 & 1.337 & 1.122\\
Relational / native head & 0.7352 & 57.25 & 0.7150 & 0.8038 & 1.120 & 1.122\\
Relational / hard UC & 0.8902 & 85.49 & 0.9046 & 0.8962 & 1.114 & 1.122\\
Relational / soft UC & 0.8608 & 76.49 & 0.8519 & 0.9048 & 1.239 & 1.122\\
STE / hard UC & 0.8808 & 84.89 & 0.8984 & 0.8829 & 1.082 & 1.122\\
STE / soft UC & 0.8745 & 80.65 & 0.8751 & 0.9016 & 1.175 & 1.122\\
\midrule\multicolumn{7}{l}{\textit{Count-only controls}}\\
All alternatives & 0.2891 & 0.00 & 0.1766 & 1.0000 & 7.428 & 1.122\\
BTL & 0.7989 & 63.29 & 0.7841 & 0.8727 & 1.321 & 1.122\\
Copeland & 0.7631 & 59.67 & 0.7401 & 0.8888 & 1.550 & 1.122\\
Count / hard UC & 0.9002 & 85.03 & 0.9090 & 0.9156 & 1.146 & 1.122\\
HodgeRank & 0.7842 & 58.68 & 0.7577 & 0.8815 & 1.399 & 1.122\\
Empty set & 0.0000 & 0.00 & 0.0000 & 0.0000 & 0.000 & 1.122\\
Posterior / calibrated & 0.8898 & 81.54 & 0.8885 & 0.9231 & 1.231 & 1.122\\
Posterior / GFM & 0.8970 & 79.16 & 0.8835 & 0.9564 & 1.343 & 1.122\\
Posterior / half & 0.9016 & 83.23 & 0.9025 & 0.9252 & 1.187 & 1.122\\
Rank Centrality & 0.7820 & 53.02 & 0.7526 & 0.8840 & 1.362 & 1.122\\
Smooth Copeland & 0.6449 & 50.22 & 0.6149 & 0.7764 & 1.455 & 1.122\\
Count / soft UC & 0.8796 & 75.94 & 0.8616 & 0.9459 & 1.365 & 1.122\\
Win rate & 0.5617 & 25.51 & 0.4869 & 0.7918 & 1.880 & 1.122\\
\bottomrule\end{tabular*}

\caption{Whole-source held-out human-profile reconstruction at 10\% records. F1, exact recovery (\%), precision, recall, and cardinalities are equally weighted source macro means. All 36 references are included, of which 33 are singletons. Hard uses a majority decoder; soft/head/posterior inclusion cutoffs are development-selected; posterior-half and GFM are native decisions. No formal tests are performed.}\label{tab:real-learning-primary}
\end{table}

STE soft/hard F1 is 0.8745/0.8808, below count-hard 0.9002 and posterior-half 0.9016. STE soft exact recovery is 80.65\% versus 85.03\%/83.23\% for those controls. The ordinary and relational native heads have lower F1, but this does not establish superiority over useful count-based alternatives. The all-class outcome is dominated by singleton references and does not validate recovery of nontrivial sets. Table~\ref{tab:real-selective-profiles} reports all three selective references individually; their size and overlap failures are retained with the secondary 5\%/25\% results. Earlier recorded-profile findings in Appendix~\ref{app:ibm-gpu} remain unchanged.

\begin{table}[ht]\centering\scriptsize
\begin{tabular}{llrrr}
\toprule
Profile & Method & F1 & Exact (\%) & $|\widehat S|$\\
\midrule
ERS A & Ordinary / native & 0.4444 & 0.00 & 1.62\\
 & Relational / native & 0.4367 & 4.17 & 1.62\\
 & STE / soft UC & 0.5456 & 0.00 & 2.54\\
 & STE / hard UC & 0.4583 & 0.00 & 1.58\\
 & Count / hard UC & 0.4958 & 0.00 & 2.00\\
 & Posterior / half & 0.5488 & 0.00 & 2.75\\
\midrule
ERS B & Ordinary / native & 0.5000 & 0.00 & 1.08\\
 & Relational / native & 0.4417 & 0.00 & 0.92\\
 & STE / soft UC & 0.5792 & 8.33 & 1.29\\
 & STE / hard UC & 0.5625 & 12.50 & 1.25\\
 & Count / hard UC & 0.6000 & 12.50 & 1.38\\
 & Posterior / half & 0.6750 & 12.50 & 1.62\\
\midrule
Oakland & Ordinary / native & 0.4125 & 0.00 & 0.88\\
 & Relational / native & 0.2750 & 0.00 & 0.58\\
 & STE / soft UC & 0.8875 & 62.50 & 2.50\\
 & STE / hard UC & 0.7083 & 41.67 & 1.83\\
 & Count / hard UC & 0.8125 & 62.50 & 2.25\\
 & Posterior / half & 0.7875 & 50.00 & 2.12\\
\bottomrule
\end{tabular}

\caption{All three selective human references at 10\% records: ERS A ($n=11$), ERS B ($n=4$), and Oakland ($n=8$), each with reference size three. Rows average eight shared resamples; learned readouts additionally average three initializations. These descriptive profile-level outcomes are not independent population replications. All methods and other fractions remain in the retained outputs.}\label{tab:real-selective-profiles}
\end{table}

\subsection{Evidence and independent checks}
Production refuses CPU fallback and uses the L40S/pinned numerical environment. Raw profile bytes, parsed rank multiplicities, full/subsample counts, fold assignments, training updates, measured costs, development choices, selected weights, and predicted matrices/scores are retained. The production audit reconstructs all 864 subsets from raw ordinal records, verifies development choices, and replays 1,080 selected-weight prediction files on CPU; the maximum saved-prediction difference is $6.62\times10^{-6}$. A separate raw-output audit recomputes all 37,152 set-metric rows from saved predictions and references, checks 126 CUDA fits and 180 neural decoder choices, and independently reproduces the profile-to-source aggregation. Both archive payload hashes and CRCs pass. Saved-state replay checks numerical consistency rather than retraining optimizer trajectories or providing independent population evidence.

\FloatBarrier
\section{Posterior and count controls on human profiles}\label{app:posterior-native}
This extension evaluates native UC decisions for learned orientation distributions, simple posterior models, and direct membership heads on the previously examined human-profile cohort. Its data, sampling, model fitting, and decoder choices were fixed before production. Prior outcomes and the cohort were already known, so the study remains exploratory. It neither replaces the earlier human-profile study in Appendix~\ref{app:real-learning} nor adds an independent preference dataset. The extension uses separately seeded nested subsamples and a distinct fitting/selection protocol. Both studies preserve their actual sampled inputs and endpoints; their means are not pooled or interpreted as controlled changes between experiments. In particular, the different count-control means across studies reflect different observed subsets and native decisions, not a demonstrated change in estimator quality.

\subsection{Finite-record target and source separation}
The 36 actual ordinal profiles come from six PrefLib source series: ERS, Oakland, San Francisco, San Leandro, CSES, and Voter Autrement \citep{mattei2013preflib,preflibers,oneill2013openstv,preflibcses,preflibvoterautrement}. The source-level inventory is in Table~\ref{tab:posterior-cohort}. These are election ballots, survey-derived rankings, or online voting records; no sports profiles or generated ballots enter this study. There are 4--16 alternatives and 134--144,229 records per profile. Full decisive counts observe every unordered pair without an exact majority tie. All targets are the hard UC of the complete finite recorded profile: 33 singletons and three selective sets, each of size three. These references do not identify a voter-population core.

Positive integer rank multiplicities are retained. An explicit tied rank supplies a symmetric tied count and no decisive win; a comparison involving an unranked alternative remains unobserved. Eight resamples draw nested 5\%, 10\%, and 25\% fractions without replacement through multivariate hypergeometric increments over the actual record multiplicities. Each fraction uses at least one record, with the sample size otherwise rounded down. The saved 864 inputs preserve dependence among pair counts induced by shared records. Repetitions, fractions, and initializations do not create additional profiles.

In the six cyclic folds, one whole source is held out for testing, the next source in the displayed source order for development, and the other four for training. Each source occurs once as test and once as development. Full training counts support prior fitting or reference labels; development references select checkpoints and decisions. Held-out predictors receive only observed decisive and tied counts. They receive no full test counts, target memberships, core sizes, or core-class flags. The ERS-held-out fold contains no selective training profile, and the Oakland-held-out fold contains two. This limited selective supervision is part of the benchmark's scope.
\begin{table}[ht]\centering\small
\begin{tabular*}{\linewidth}{@{\extracolsep{\fill}}lrrrr@{}}\toprule
Source & Profiles & Alternatives & Records & Selective references\\\midrule
ERS & 9 & 4--16 & 134--1,123 & 2\\
Oakland & 7 & 4--11 & 11,266--144,229 & 1\\
San Francisco & 3 & 4--7 & 24,109--34,931 & 0\\
San Leandro & 2 & 4--5 & 23,236--25,355 & 0\\
CSES & 10 & 9--9 & 937--2,338 & 0\\
Voter Autrement & 5 & 11--12 & 743--11,219 & 0\\
\bottomrule\end{tabular*}
\caption{Actual human-profile inventory for the posterior-control extension. The displayed order also defines cyclic source folds. Each full reference is strict; subsample relations can be incomplete or tied. CSES construction and provider references are retained in Appendix~\ref{app:real-learning}.}\label{tab:posterior-cohort}\end{table}

\subsection{Systems, training information, and native decisions}
Write $q_{ab}$ for the predictive probability of a \emph{majority orientation}, distinct from the matchup probability $P_{ab}$. Jeffreys inference uses independent Beta$(w_{ab}+1/2,w_{ba}+1/2)$ matchup posteriors and sets $q_{ab}=\Pr(P_{ab}>1/2\mid W)$. Empirical Bayes fits one symmetric Beta$(\alpha,\alpha)$ prior from full decisive counts of unique training profiles. Its bounded scalar objective averages pair log marginals within profile, profiles within source, and sources equally. Repeated subsamples add no prior-fitting evidence. Because pairs share ordinal records, this is a composite Beta-binomial working likelihood. Empirical Bayes receives richer full-count supervision than the majority/membership labels used by neural methods; common observed test inputs do not imply matched training information.

The four neural systems share an observed-count encoder with eight permutation-equivariant node features and five oriented pair features. Features summarize Jeffreys-smoothed proportions, decisive margins/counts, tied fractions, and observed coverage. Two width-32 layers use tanh activations. The ordinary head predicts membership from the node representation; the relational head additionally aggregates learned pair messages. Both use BCE against full-profile UC membership. Learned edges instead predict reciprocal orientation probabilities and minimize unordered-pair BCE against full-profile majority labels. The two-component mixture adds a global component law; one component is drawn for the \emph{whole} tournament, followed by conditionally independent edges. It minimizes orientation BCE plus a likelihood-ratio gradient estimate of sampled hard-UC F1 utility, using a leave-one-out reward baseline and fixed weights $\lambda\in\{0.1,1\}$ (Proposition~\ref{prop:mixture-score-identity}). This objective does not establish improvement of the final GFM action. This produces a learned predictive distribution, not an exact Bayesian posterior or an externally calibrated joint law.

Neural updates process one training input at a time through random permutations of the saved training cases. Adam uses learning rate 0.001, width 32, gradient clipping 5, float32, deterministic operations, and disabled TF32. Paired initializations share the initialization seed within each fold. The two mixture candidates share one system allocation. They use 64 sampled tournaments per utility update; development decisions use 512 draws. Checkpoint opportunities occur at 20\%, 45\%, and 70\% of each candidate's allocation. Development selection uses only its source's 10\%-record inputs, averaging resamples within profile and profiles equally. Exact ties prefer an earlier checkpoint and smaller utility weight. Direct-head and soft plug-in cutoffs use $0,0.05,\ldots,1$, with explicit all/empty endpoints and ties preferring proximity to $0.5$, then the smaller cutoff. Choices are serialized before test scoring.

Jeffreys, empirical Bayes, learned edges, and the mixture all draw 4,096 strict tournaments per held-out input and use the established General F-measure Maximizer (GFM) \citep{waegeman2014fmeasure}. For each possible size $k$, it selects the $k$ alternatives with largest sampled expected F1 contributions and maximizes that utility over $k$. This is the native expected-F1 decision for the sampled joint UC distribution; it need not maximize F1 against the finite recorded reference. The native ordinary/relational decisions use their own development-selected heads. Raw-count UC evaluates covering on the observed strict-majority relation, omitting unresolved directions. The prespecified secondary count soft plug-in applies canonical soft UC to Jeffreys pair means and development-selects its cutoff. Posterior half-threshold and hard-orientation diagnostics are not substituted for these native decisions after seeing outcomes.

\subsection{Native outcomes and aggregation}
The designated descriptive comparison includes all profiles at 10\% records. Each metric averages the eight resamples and three initializations within profile, profiles within source, and the six source means equally. Count procedures are retained in the same paired case structure. The source-macro reference size is 1.1217, which differs from the profile-pooled size because sources are equally weighted. Table~\ref{tab:posterior-native-full} reports the overlap, exact-set, and cardinality endpoints; Tables~\ref{tab:posterior-native-sources} and~\ref{tab:posterior-native-profiles} expose source and profile heterogeneity. There are no significance tests, population confidence intervals, or general-superiority claims.
\begin{table}[ht]\centering\footnotesize
\begin{tabular*}{\linewidth}{@{\extracolsep{\fill}}lrrrrrr@{}}\toprule
System & F1 & Exact & $|S|$ & Size MAE & Correct size & UC Brier\\\midrule
Jeffreys--GFM & 0.8742 & 0.7656 & 1.352 & 0.2713 & 0.8254 & 0.0437\\
Empirical Bayes--GFM & 0.8705 & 0.7522 & 1.369 & 0.2825 & 0.8119 & 0.0447\\
Learned edges--GFM & 0.8704 & 0.7614 & 1.217 & 0.2485 & 0.8183 & 0.0400\\
Learned mixture--GFM & 0.8366 & 0.7301 & 1.185 & 0.2300 & 0.8257 & 0.0486\\
Raw-count UC & 0.8481 & 0.7950 & 1.194 & 0.1658 & 0.9038 & ---\\
Count soft plug-in$^{\dagger}$ & 0.8462 & 0.7679 & 1.233 & 0.2117 & 0.8635 & 0.0507\\
Relational head & 0.7944 & 0.6161 & 1.307 & 0.3157 & 0.7197 & 0.0681\\
Ordinary head & 0.7340 & 0.5730 & 1.301 & 0.5336 & 0.6591 & 0.0708\\
\bottomrule\end{tabular*}
\caption{Native 10\%-record source macro results. Exact and correct size are frequencies, not percentages; size MAE is $\mathbb E\lvert |S|-|C|\rvert$. UC Brier is the mean squared error of membership scores against finite-reference indicators; lower values establish neither calibration nor better set decisions. $^{\dagger}$Prespecified secondary soft plug-in control.}\label{tab:posterior-native-full}\end{table}

Jeffreys--GFM has the highest native F1 (0.8742). Learned edges are lower by 0.003779 and the mixture by 0.037609. Raw-count UC has the highest exact recovery (0.7950). Learned edges have the lowest recorded UC Brier (0.0400), but the improved marginal-score loss does not establish an F1 or exact-set advantage. These endpoints cannot be collapsed into one performance claim. Empirical-Bayes early convergence, neural capacity, different supervision, and source dependence further limit causal interpretations.

For compact tables, J/E/L/M denote Jeffreys, empirical Bayes, learned independent edges, and the mixture with GFM; R is raw-count UC; S is the secondary soft plug-in; Rel/Ord are native relational/ordinary heads. Each paired cell below is F1/exact-set frequency.
\begin{table}[ht]\centering\scriptsize\setlength{\tabcolsep}{2pt}
\begin{tabular*}{\linewidth}{@{\extracolsep{\fill}}lrrrrrrrr@{}}\toprule
Source & J & E & L & M & R & S$^{\dagger}$ & Rel & Ord\\\midrule
ERS & 0.719/0.444 & 0.714/0.431 & 0.702/0.435 & 0.680/0.426 & 0.638/0.431 & 0.651/0.444 & 0.618/0.343 & 0.516/0.296\\
Oakland & 0.976/0.929 & 0.981/0.929 & 0.941/0.851 & 0.944/0.863 & 0.982/0.964 & 0.979/0.946 & 0.979/0.911 & 0.824/0.810\\
San Francisco & 0.910/0.833 & 0.896/0.792 & 0.968/0.903 & 0.829/0.708 & 0.924/0.875 & 0.924/0.792 & 0.792/0.514 & 0.880/0.736\\
San Leandro & 0.833/0.750 & 0.833/0.750 & 0.861/0.750 & 0.812/0.729 & 0.750/0.750 & 0.750/0.750 & 0.750/0.750 & 0.750/0.750\\
CSES & 0.882/0.787 & 0.882/0.787 & 0.886/0.804 & 0.895/0.838 & 0.865/0.850 & 0.858/0.850 & 0.893/0.854 & 0.854/0.846\\
Voter Autrement & 0.926/0.850 & 0.917/0.825 & 0.865/0.825 & 0.859/0.817 & 0.930/0.900 & 0.915/0.825 & 0.735/0.325 & 0.580/0.000\\
\bottomrule\end{tabular*}
\caption{Source-specific 10\%-record F1/exact recovery, with profiles weighted equally within source. Variation across these six purposively assembled series is descriptive and does not justify a six-source population test.}\label{tab:posterior-native-sources}\end{table}

\begingroup\scriptsize\setlength{\tabcolsep}{2pt}
\begin{longtable}{@{}lrrllllllll@{}}
\caption{All 36 profile results at 10\% records. A profile label gives its source and catalogue identifier, not a new dataset; $n$ is alternative count and $c$ reference UC size. Every method cell is F1/exact recovery, averaged over the shared resamples and initializations.}\label{tab:posterior-native-profiles}\\
\toprule Profile & $n$ & $c$ & J & E & L & M & R & S$^{\dagger}$ & Rel & Ord\\\midrule\endfirsthead
\multicolumn{11}{l}{\textit{Table~\ref{tab:posterior-native-profiles} continued}}\\\toprule
Profile & $n$ & $c$ & J & E & L & M & R & S$^{\dagger}$ & Rel & Ord\\\midrule\endhead
\midrule\multicolumn{11}{r}{\textit{Continued on next page}}\\\endfoot
\bottomrule\endlastfoot
ERS 10 & 6 & 1 & 0.429/0.250 & 0.425/0.250 & 0.653/0.375 & 0.597/0.375 & 0.458/0.250 & 0.396/0.250 & 0.569/0.250 & 0.444/0.333\\
ERS 33 & 16 & 1 & 0.958/0.875 & 0.958/0.875 & 0.958/0.875 & 0.958/0.875 & 1.000/1.000 & 1.000/1.000 & 0.764/0.292 & 0.596/0.042\\
ERS 68 & 11 & 3 & 0.652/0.000 & 0.652/0.000 & 0.670/0.000 & 0.671/0.000 & 0.539/0.000 & 0.646/0.000 & 0.658/0.000 & 0.635/0.000\\
ERS 69 & 4 & 1 & 0.571/0.375 & 0.571/0.375 & 0.599/0.417 & 0.589/0.417 & 0.583/0.500 & 0.583/0.500 & 0.513/0.458 & 0.500/0.500\\
ERS 70 & 4 & 1 & 0.792/0.375 & 0.792/0.375 & 0.764/0.458 & 0.653/0.458 & 0.375/0.375 & 0.542/0.375 & 0.500/0.333 & 0.375/0.375\\
ERS 71 & 7 & 1 & 0.661/0.250 & 0.619/0.125 & 0.547/0.250 & 0.590/0.250 & 0.675/0.375 & 0.717/0.500 & 0.686/0.542 & 0.625/0.375\\
ERS 73 & 5 & 1 & 0.583/0.375 & 0.583/0.375 & 0.479/0.375 & 0.458/0.375 & 0.458/0.375 & 0.458/0.375 & 0.482/0.417 & 0.444/0.417\\
ERS 75 & 9 & 1 & 0.938/0.875 & 0.938/0.875 & 0.875/0.875 & 0.875/0.875 & 0.938/0.875 & 0.938/0.875 & 0.806/0.667 & 0.833/0.625\\
ERS 86 & 4 & 3 & 0.888/0.625 & 0.888/0.625 & 0.771/0.292 & 0.729/0.208 & 0.720/0.125 & 0.575/0.125 & 0.583/0.125 & 0.188/0.000\\
Oakland 1 & 8 & 1 & 1.000/1.000 & 1.000/1.000 & 1.000/1.000 & 1.000/1.000 & 1.000/1.000 & 1.000/1.000 & 1.000/1.000 & 1.000/1.000\\
Oakland 2 & 11 & 1 & 1.000/1.000 & 1.000/1.000 & 1.000/1.000 & 1.000/1.000 & 1.000/1.000 & 1.000/1.000 & 1.000/1.000 & 1.000/1.000\\
Oakland 3 & 6 & 1 & 1.000/1.000 & 1.000/1.000 & 1.000/1.000 & 1.000/1.000 & 1.000/1.000 & 1.000/1.000 & 1.000/1.000 & 1.000/1.000\\
Oakland 4 & 8 & 3 & 0.912/0.750 & 0.950/0.750 & 0.629/0.083 & 0.608/0.042 & 0.875/0.750 & 1.000/1.000 & 0.894/0.500 & 0.104/0.000\\
Oakland 5 & 7 & 1 & 1.000/1.000 & 1.000/1.000 & 1.000/1.000 & 1.000/1.000 & 1.000/1.000 & 0.938/0.875 & 0.972/0.917 & 1.000/1.000\\
Oakland 6 & 5 & 1 & 0.917/0.750 & 0.917/0.750 & 0.958/0.875 & 1.000/1.000 & 1.000/1.000 & 0.917/0.750 & 0.986/0.958 & 0.708/0.708\\
Oakland 7 & 4 & 1 & 1.000/1.000 & 1.000/1.000 & 1.000/1.000 & 1.000/1.000 & 1.000/1.000 & 1.000/1.000 & 1.000/1.000 & 0.958/0.958\\
San Francisco 4 & 4 & 1 & 1.000/1.000 & 1.000/1.000 & 1.000/1.000 & 0.972/0.917 & 1.000/1.000 & 1.000/1.000 & 0.889/0.667 & 1.000/1.000\\
San Francisco 7 & 7 & 1 & 0.729/0.500 & 0.688/0.375 & 0.903/0.708 & 0.542/0.292 & 0.771/0.625 & 0.771/0.375 & 0.569/0.125 & 0.639/0.208\\
San Francisco 9 & 5 & 1 & 1.000/1.000 & 1.000/1.000 & 1.000/1.000 & 0.972/0.917 & 1.000/1.000 & 1.000/1.000 & 0.917/0.750 & 1.000/1.000\\
San Leandro 2 & 4 & 1 & 0.667/0.500 & 0.667/0.500 & 0.722/0.500 & 0.625/0.458 & 0.500/0.500 & 0.500/0.500 & 0.500/0.500 & 0.500/0.500\\
San Leandro 3 & 5 & 1 & 1.000/1.000 & 1.000/1.000 & 1.000/1.000 & 1.000/1.000 & 1.000/1.000 & 1.000/1.000 & 1.000/1.000 & 1.000/1.000\\
CSES 47 & 9 & 1 & 0.737/0.375 & 0.737/0.375 & 0.786/0.500 & 0.857/0.667 & 0.750/0.750 & 0.833/0.750 & 0.806/0.750 & 0.750/0.750\\
CSES 48 & 9 & 1 & 0.413/0.250 & 0.413/0.250 & 0.365/0.250 & 0.401/0.292 & 0.438/0.375 & 0.375/0.375 & 0.361/0.333 & 0.375/0.375\\
CSES 72 & 9 & 1 & 0.875/0.875 & 0.875/0.875 & 0.903/0.875 & 0.903/0.875 & 0.875/0.875 & 0.875/0.875 & 0.931/0.875 & 0.875/0.875\\
CSES 169 & 9 & 1 & 1.000/1.000 & 1.000/1.000 & 1.000/1.000 & 1.000/1.000 & 1.000/1.000 & 1.000/1.000 & 1.000/1.000 & 1.000/1.000\\
CSES 172 & 9 & 1 & 0.917/0.750 & 0.917/0.750 & 0.931/0.792 & 0.944/0.833 & 0.833/0.750 & 0.875/0.875 & 0.972/0.917 & 0.750/0.750\\
CSES 249 & 9 & 1 & 0.917/0.750 & 0.917/0.750 & 0.917/0.750 & 0.861/0.750 & 0.750/0.750 & 0.750/0.750 & 0.931/0.792 & 0.806/0.750\\
CSES 256 & 9 & 1 & 0.958/0.875 & 0.958/0.875 & 0.958/0.875 & 0.986/0.958 & 1.000/1.000 & 0.875/0.875 & 0.931/0.875 & 0.986/0.958\\
CSES 257 & 9 & 1 & 1.000/1.000 & 1.000/1.000 & 1.000/1.000 & 1.000/1.000 & 1.000/1.000 & 1.000/1.000 & 1.000/1.000 & 1.000/1.000\\
CSES 286 & 9 & 1 & 1.000/1.000 & 1.000/1.000 & 1.000/1.000 & 1.000/1.000 & 1.000/1.000 & 1.000/1.000 & 1.000/1.000 & 1.000/1.000\\
CSES 287 & 9 & 1 & 1.000/1.000 & 1.000/1.000 & 1.000/1.000 & 1.000/1.000 & 1.000/1.000 & 1.000/1.000 & 1.000/1.000 & 1.000/1.000\\
Voter Autrement 6 & 11 & 1 & 1.000/1.000 & 1.000/1.000 & 0.986/0.958 & 1.000/1.000 & 1.000/1.000 & 0.958/0.875 & 0.667/0.000 & 0.667/0.000\\
Voter Autrement 7 & 11 & 1 & 1.000/1.000 & 1.000/1.000 & 1.000/1.000 & 1.000/1.000 & 1.000/1.000 & 1.000/1.000 & 0.681/0.042 & 0.583/0.000\\
Voter Autrement 8 & 11 & 1 & 0.671/0.375 & 0.629/0.250 & 0.381/0.292 & 0.336/0.208 & 0.648/0.500 & 0.682/0.375 & 0.569/0.292 & 0.579/0.000\\
Voter Autrement 14 & 12 & 1 & 1.000/1.000 & 1.000/1.000 & 1.000/1.000 & 1.000/1.000 & 1.000/1.000 & 1.000/1.000 & 0.889/0.667 & 0.521/0.000\\
Voter Autrement 15 & 12 & 1 & 0.958/0.875 & 0.958/0.875 & 0.958/0.875 & 0.958/0.875 & 1.000/1.000 & 0.938/0.875 & 0.868/0.625 & 0.551/0.000\\
\end{longtable}\endgroup

\subsection{Selective references, exact recovery, and cardinality}
Only ERS profiles 68 and 86 and Oakland profile 4 have selective non-singleton references, all of size three. The selective source macro averages the two ERS profiles, then ERS and Oakland equally. It uses two represented sources, not six selective sources or independent replicated trials. In that aggregate, Jeffreys/empirical-Bayes F1 is 0.8411/0.8599, versus 0.6749/0.6542 for learned edges/mixture. The latter select only 1.70/1.58 alternatives on average, and exactly recover 0.1146/0.0729 of the saved sets. Table~\ref{tab:posterior-native-selective} retains every profile's overlap, exact recovery, and size results.

Every evaluated native method has zero exact recovery on the eleven-alternative ERS profile at the saved 10\% inputs. Oakland's secondary soft plug-in recovers the finite reference exactly in every saved input; that favorable secondary profile does not replace the native main endpoint. The two-source aggregate does not conceal these profile-specific limitations or establish that the method transfers to diverse nontrivial human cores.
\begin{table}[ht]\centering\footnotesize
\begin{tabular*}{\linewidth}{@{\extracolsep{\fill}}llrrrrr@{}}\toprule
Profile ($n$) & System & F1 & Exact & $|S|$ & Correct size & Size MAE\\\midrule
ERS 68 (11) & Jeffreys--GFM & 0.6521 & 0.0000 & 3.125 & 0.7500 & 0.3750\\
 & Empirical Bayes--GFM & 0.6521 & 0.0000 & 3.125 & 0.7500 & 0.3750\\
 & Learned edges--GFM & 0.6704 & 0.0000 & 1.875 & 0.0833 & 1.2083\\
 & Learned mixture--GFM & 0.6708 & 0.0000 & 1.750 & 0.1250 & 1.2500\\
 & Raw-count UC & 0.5387 & 0.0000 & 2.500 & 0.3750 & 1.0000\\
 & Count soft plug-in$^{\dagger}$ & 0.6464 & 0.0000 & 2.375 & 0.3750 & 0.8750\\
 & Relational head & 0.6583 & 0.0000 & 2.000 & 0.2500 & 1.0000\\
 & Ordinary head & 0.6347 & 0.0000 & 1.708 & 0.0833 & 1.2917\\
\midrule
ERS 86 (4) & Jeffreys--GFM & 0.8875 & 0.6250 & 2.500 & 0.6250 & 0.5000\\
 & Empirical Bayes--GFM & 0.8875 & 0.6250 & 2.500 & 0.6250 & 0.5000\\
 & Learned edges--GFM & 0.7708 & 0.2917 & 2.000 & 0.2917 & 1.0000\\
 & Learned mixture--GFM & 0.7292 & 0.2083 & 1.833 & 0.2083 & 1.1667\\
 & Raw-count UC & 0.7196 & 0.1250 & 2.000 & 0.1250 & 1.2500\\
 & Count soft plug-in$^{\dagger}$ & 0.5750 & 0.1250 & 1.375 & 0.1250 & 1.6250\\
 & Relational head & 0.5833 & 0.1250 & 1.375 & 0.1250 & 1.6250\\
 & Ordinary head & 0.1875 & 0.0000 & 0.375 & 0.0000 & 2.6250\\
\midrule
Oakland 4 (8) & Jeffreys--GFM & 0.9125 & 0.7500 & 2.625 & 0.7500 & 0.3750\\
 & Empirical Bayes--GFM & 0.9500 & 0.7500 & 2.750 & 0.7500 & 0.2500\\
 & Learned edges--GFM & 0.6292 & 0.0833 & 1.458 & 0.0833 & 1.5417\\
 & Learned mixture--GFM & 0.6083 & 0.0417 & 1.375 & 0.0417 & 1.6250\\
 & Raw-count UC & 0.8750 & 0.7500 & 2.500 & 0.7500 & 0.5000\\
 & Count soft plug-in$^{\dagger}$ & 1.0000 & 1.0000 & 3.000 & 1.0000 & 0.0000\\
 & Relational head & 0.8944 & 0.5000 & 2.542 & 0.5417 & 0.4583\\
 & Ordinary head & 0.1042 & 0.0000 & 0.208 & 0.0000 & 2.7917\\
\bottomrule\end{tabular*}
\caption{All three selective reference profiles at 10\% records, each with true UC size three. Eight shared inputs and three initializations are averaged within profile. Exact-set and correct-size frequencies answer different questions.}\label{tab:posterior-native-selective}\end{table}

\subsection{Secondary fractions}
Table~\ref{tab:posterior-native-fractions} preserves all systems at the two nested secondary fractions. Five- and 25\%-record outcomes reuse the same finite profiles and resampling streams; they are neither new independent datasets nor substitute primary settings. The mixture remains below Jeffreys--GFM in F1 at both fractions. Learned edges' slightly higher 5\% mean cannot establish equivalence, a learning advantage, or a revised headline outcome.
\begin{table}[ht]\centering\footnotesize
\begin{tabular*}{\linewidth}{@{\extracolsep{\fill}}lrrrr@{}}\toprule
& \multicolumn{2}{c}{5\% records} & \multicolumn{2}{c}{25\% records}\\\cmidrule(lr){2-3}\cmidrule(l){4-5}
System & F1 & Exact & F1 & Exact\\\midrule
Jeffreys--GFM & 0.8199 & 0.6662 & 0.9448 & 0.8626\\
Empirical Bayes--GFM & 0.8097 & 0.6553 & 0.9443 & 0.8598\\
Learned edges--GFM & 0.8203 & 0.7091 & 0.9421 & 0.8626\\
Learned mixture--GFM & 0.8136 & 0.6865 & 0.9172 & 0.8169\\
Raw-count UC & 0.8220 & 0.7093 & 0.9458 & 0.9144\\
Count soft plug-in$^{\dagger}$ & 0.8187 & 0.7130 & 0.9464 & 0.8929\\
Relational head & 0.7751 & 0.5791 & 0.8863 & 0.7166\\
Ordinary head & 0.6929 & 0.5008 & 0.8244 & 0.6692\\
\bottomrule\end{tabular*}
\caption{Secondary nested-fraction source macro results using the unchanged native decisions and profile/source weighting.}\label{tab:posterior-native-fractions}\end{table}

\subsection{Actual compute and inference cost}
Each of five fitted systems receives a common \emph{upper} allocation of 120 synchronized wall-clock seconds per fold/initialization, including initialization, fitting, development/Monte Carlo work, checkpointing, selection, and bookkeeping. Ninety allocations total 10,800 seconds. Every recorded fit satisfies the fixed 2\% upper-allocation tolerance, but actual charged time totals 6,064.06 seconds. The converged empirical-Bayes scalar optimizer uses only 1.50 seconds across its eighteen fits; neural systems typically finish in about 84 seconds per allocation. Equal caps therefore do not imply equal actual compute, equal statistical complexity, or matched supervision. Mixture utility candidates split one allocation; additional candidate fitting is not hidden outside the cap.

Table~\ref{tab:posterior-native-training} gives measured fit costs and update counts. Held-out inference, common data preparation, analysis, and audits are outside fitting allocations. Table~\ref{tab:posterior-native-inference} separately reports all saved prediction calls across fractions, profiles, folds, and initializations. Their arithmetic mean times differ from the source-macro weighting of quality metrics. The recorded fit-and-test wall time is 6,219.68 seconds; held-out evaluation is 147.00 seconds. Timing is for the single tested L40S and software environment, not a portable hardware scaling claim.
\begin{table}[ht]\centering\footnotesize
\begin{tabular*}{\linewidth}{@{\extracolsep{\fill}}lrrrrr@{}}\toprule
System & Fits & Allocated s & Charged s & Fit range (s) & Updates\\\midrule
Empirical Bayes--GFM & 18 & 2160 & 1.50 & 0.02--0.14 & 0\\
Ordinary head & 18 & 2160 & 1513.45 & 84.03--84.13 & 403,880\\
Relational head & 18 & 2160 & 1515.35 & 84.04--85.74 & 323,530\\
Learned edges--GFM & 18 & 2160 & 1515.09 & 84.05--84.31 & 349,842\\
Learned mixture--GFM & 18 & 2160 & 1518.67 & 84.11--84.62 & 208,304\\
\bottomrule\end{tabular*}
\caption{Actual fitting/development costs. Empirical Bayes uses a fitted scalar optimizer with no neural gradient updates; zero in that column does not mean an unfitted prior. Neural update counts include candidate training and do not measure equivalent statistical information.}\label{tab:posterior-native-training}\end{table}
\begin{table}[ht]\centering\footnotesize
\begin{tabular*}{\linewidth}{@{\extracolsep{\fill}}lrrrr@{}}\toprule
System & Calls & Mean ms & Total s & Draws/input\\\midrule
Jeffreys--GFM & 2,592 & 7.415 & 19.22 & 4,096\\
Empirical Bayes--GFM & 2,592 & 7.467 & 19.36 & 4,096\\
Learned edges--GFM & 2,592 & 9.034 & 23.42 & 4,096\\
Learned mixture--GFM & 2,592 & 9.472 & 24.55 & 4,096\\
Raw-count UC & 2,592 & 0.033 & 0.09 & 0\\
Count soft plug-in$^{\dagger}$ & 2,592 & 3.058 & 7.93 & 0\\
Relational head & 2,592 & 1.791 & 4.64 & 0\\
Ordinary head & 2,592 & 1.618 & 4.19 & 0\\
\bottomrule\end{tabular*}
\caption{Held-out prediction cost on the L40S, including saved native posterior draws where applicable. Sampling/GFM is more expensive than the raw covering decision; these held-out costs occur outside the fitting cap.}\label{tab:posterior-native-inference}\end{table}

\subsection{Certificates: logical scope and descriptive diagnostics}
The completion enclosure in Proposition~\ref{prop:uc-completion-enclosure} and fixed-design statement in Corollary~\ref{cor:uc-fixed-design-enclosure} require explicit allowed edge orientations (Appendix~\ref{app:posterior-certificates-theory}). Independent Beta posteriors with positive parameters support every strict tournament, so unrestricted support alone cannot certify nontrivial individual UC membership. Human-profile count intervals are therefore inspected only as restrictions defining a completion family. Their necessary/possible outputs describe that family; they do not gain frequentist coverage from correlated, finite-population record subsampling. The fixed-count IID Bernoulli assumptions used for the supporting concentration theorem are not the human subsampling design.

Table~\ref{tab:posterior-native-certificates} reports source-macro human diagnostics across the specified per-case failure allocations and fractions. The saved cases have no false certified inclusions or exclusions relative to the finite reference, but substantial abstention. Zero observed errors neither prove calibration nor justify a zero-risk or voter-population claim. The failure allocation describes the IID theorem's fixed-case calculation, not simultaneous or anytime coverage over these human profiles.
\begin{table}[ht]\centering\footnotesize
\begin{tabular*}{\linewidth}{@{\extracolsep{\fill}}rrrrrrr@{}}\toprule
$\delta$ & Records & Certified & Abstained & $|L|$ & $|U|$ & False in/out\\\midrule
0.01 & 5\% & 0.0976 & 0.9024 & 0.048 & 6.553 & 0/0\\
0.01 & 10\% & 0.1698 & 0.8302 & 0.088 & 6.131 & 0/0\\
0.01 & 25\% & 0.3923 & 0.6077 & 0.295 & 4.891 & 0/0\\
0.05 & 5\% & 0.1151 & 0.8849 & 0.061 & 6.425 & 0/0\\
0.05 & 10\% & 0.2069 & 0.7931 & 0.130 & 5.992 & 0/0\\
0.05 & 25\% & 0.4593 & 0.5407 & 0.369 & 4.444 & 0/0\\
0.10 & 5\% & 0.1269 & 0.8731 & 0.071 & 6.374 & 0/0\\
0.10 & 10\% & 0.2357 & 0.7643 & 0.138 & 5.829 & 0/0\\
0.10 & 25\% & 0.4890 & 0.5110 & 0.396 & 4.262 & 0/0\\
\bottomrule\end{tabular*}
\caption{Finite-cohort certificate diagnostics with no human-profile coverage guarantee. $L/U$ are inner/outer membership bounds under the calculated edge restrictions. ``Certified'' is the fraction of alternative statuses decided; it is not the fraction of exactly recovered complete sets.}\label{tab:posterior-native-certificates}\end{table}

A separate independently audited CPU diagnostic contains 6,144 cases under six \emph{fixed} probability matrices: singleton strong-core/weak-remainder, size-three cyclic-core/weak-remainder, and uniform-margin families at $n=12,24$. It uses 16, 64, 256, or 1,024 IID Bernoulli comparisons per pair and 256 repetitions per cell, with $\delta=0.05$ per case. The saved cases have twelve whole-graph interval failures, zero UC-envelope failures, two confidently wrong edges, and 3,177 exact correct certificates. At 256 observations per pair, every replicate in the four strong-core/weak-remainder cells certifies the true singleton or size-three core, while none fully orients the graph. These chosen examples illustrate output certification before whole-graph orientation. Every pair receives the same budget: they establish no adaptive query or allocation savings, population calibration, or learned-method advantage.

\subsection{Independent verification and excluded secondary diagnostic}
Preserved evidence includes raw ordinal records and multiplicities, full/subsample counts, fold identities, reference sets, development choices, selected weights, training devices/gradients, time accounting, sampled tournaments/memberships, and per-case decisions. An external review independently reconstructs all 864 ballot/count inputs, inspects 20,736 prediction files, recomputes 41,472 recorded metric rows and their summaries, and verifies archive payload hashes. The original production audit regenerates the Monte Carlo draws; the external review independently rebuilds their joint membership statistics and native GFM decisions. On the original CUDA device and deterministic settings, 12,960 selected-model forward replays reproduce the stored outputs exactly. This checks saved-state consistency, not independent retraining or optimizer convergence. A partial CPU replay reached a strict intermediate-logit tolerance limit and is not represented as full cross-device validation.

Independent structural decoding exposes a genuine invalid \emph{secondary} mixture hard-orientation diagnostic. Floating-point mixture summation made nominal diagonal $q_{aa}=1/2$ entries slightly larger than $1/2$, introducing self-loops when thresholded. In 320 of 2,592 saved mixture cases (101 at the main 10\% fraction), the recorded hard relation was therefore not irreflexive and its purported UC decision was invalid. Seven mixture cases also contain raw orientation probabilities slightly outside $[0,1]$; no bidirectional pairs were observed. The entire archived mixture hard-orientation aggregate is excluded, rather than deleting affected cases and interpreting an outcome-selected remainder. Raw mixture-orientation Brier/calibration summaries are also excluded as measures of a valid normalized law.

The original internal audit reproduced the flawed calculation, so its pass did not establish validity of every endpoint. Native GFM instead samples a normalized component choice and strict tournaments with zero diagonal; its decisions, membership Brier, and the direct-head/raw/soft-plug-in controls were verified separately. The defect does not change the native results reported here. Frozen predictions and all invalid secondary rows remain preserved; corrected orientation reconstructions are post hoc diagnostics and no corrected aggregate is promoted as a main result. Public code adds explicit probability and irreflexivity safeguards in a separately identified correction, with no claim of a newly trained or rerun experiment.

\FloatBarrier
\section{Posterior decisions and UC certificates}\label{app:posterior-certificates-theory}
This section specifies the mathematical objects used by the additional posterior and certificate study. It connects established completion, confidence, and decision rules to UC inference; these supporting results do not establish a new tournament solution, an adaptive comparison algorithm, or improved learning. They are separate from the approximation guarantees for the canonical STE operator in Appendix~\ref{app:proofs}. Throughout, $V$ is a fixed set of $n\ge2$ vertices and every complete tournament is strict and irreflexive.

\subsection{Certificates from partially known orientations}\label{app:uc-completion-certificates}
For a partial tournament $G$, let $\mathcal C(G)$ be its unrestricted strict completions and define
\[
\mathcal N(G)=\bigcap_{T\in\mathcal C(G)}\UC(T),\qquad
\mathcal P(G)=\bigcup_{T\in\mathcal C(G)}\UC(T).
\]
These are individual necessary and possible winners, rather than possible entire winning sets \citep{aziz2015partial}. An arbitrary subset of $\mathcal P(G)$ need not be the UC of a completion. Write $x\to_G y$ for a known edge, and define
\begin{align}
I(G)&=\{v:\ \forall u\ne v,\ v\to_Gu\ \text{or}\ \exists w,\ v\to_Gw\to_Gu\},\label{eq:uc-inner}\\
O(G)&=V\setminus\bigl\{v:\ \exists u\ne v,\ u\to_Gv\ \text{and}\notag\\
&\hspace{5em}\forall w\notin\{u,v\},\ w\to_Gv\ \text{or}\ u\to_Gw\bigr\}.\label{eq:uc-outer}
\end{align}
Only the stated soundness of these constructive tests is needed below.

\begin{proposition}[Logical UC enclosure]\label{prop:uc-completion-enclosure}
For every partial tournament $G$ and every $T\in\mathcal C(G)$,
\[
I(G)\subseteq\mathcal N(G)\subseteq\UC(T)\subseteq\mathcal P(G)\subseteq O(G).
\]
If $I(G)=O(G)$, this set is the UC of every completion. Each $v\in I(G)$ has a known-edge spanning out-tree of depth at most two and exactly $n-1$ edges certifying its inclusion.
\end{proposition}
\begin{proof}
First, a vertex $v$ is uncovered in a strict tournament exactly when it reaches every opponent in at most two edges. If $u$ covers $v$, then $u\to v$, and $u\to w$ for every $v\to w$, so neither $v\to u$ nor $v\to w\to u$ is possible. Conversely, if $v$ cannot reach $u$ in at most two edges, strictness forces $u\to v$ and $u\to w$ for every $v\to w$; hence $u$ covers $v$.

Every known path in Eq.~\eqref{eq:uc-inner} remains in every completion, proving $I(G)\subseteq\mathcal N(G)$. If $v$ is excluded by Eq.~\eqref{eq:uc-outer}, retain its witness $u$. Every completion has $u\to v$. Whenever that completion has $v\to w$, the alternative known edge $w\to_Gv$ cannot hold, so the test supplies $u\to_Gw$. Thus $u$ covers $v$ in every completion, proving $\mathcal P(G)\subseteq O(G)$. The middle inclusions follow from intersection and union; equality of the endpoints forces equality of the intermediate sets.

For $v\in I(G)$, retain $v\to_Gu$ for every known out-neighbor $u$ of $v$. Each remaining vertex has a two-step witness through one of these out-neighbors; choose exactly one such incoming edge. Every nonroot vertex now has one parent, either the root or its direct child. These $n-1$ edges form the required spanning out-tree, and their preservation guarantees two-step reachability.
\end{proof}
The two-step characterization and depth-two support structure are established tournament properties \citep{aziz2015partial,contet2026supports}. The implementation checks all candidate witnesses in $O(n^3)$ time; it does not choose or pay for new comparisons.

\begin{corollary}[Fixed-design frequentist enclosure]\label{cor:uc-fixed-design-enclosure}
Let $P_{ij}+P_{ji}=1$ and $P_{ij}\ne1/2$ off the diagonal. For each unordered pair, its count $m_{ij}$ is fixed before outcomes are observed, and its observations are i.i.d. Bernoulli$(P_{ij})$. Cross-pair independence is unnecessary. Put $M=n(n-1)/2$ and choose a fixed $\delta\in(0,1)$. For $m_{ij}>0$, set
\[
r_{ij}=\sqrt{\frac{\log(2M/\delta)}{2m_{ij}}},\qquad
[L_{ij},U_{ij}]=[\max(0,\widehat P_{ij}-r_{ij}),\min(1,\widehat P_{ij}+r_{ij})].
\]
Omitted pairs receive $[0,1]$, and reversed intervals are $[1-U_{ij},1-L_{ij}]$. Form $G$ with edge $i\to j$ only when $L_{ij}>1/2$. Then
\[
\Pr_P\{I(G)\subseteq\UC(T(P))\subseteq O(G)\}\ge1-\delta.
\]
\end{corollary}
\begin{proof}
Hoeffding's inequality \citep{hoeffding1963probability} bounds the failure probability of an observed interval by $2e^{-2m_{ij}r_{ij}^2}=\delta/M$. Clipping does not exclude an admissible $P_{ij}$ previously in the interval, and an omitted interval always covers it. A union bound over at most $M$ pairs therefore gives simultaneous coverage with probability at least $1-\delta$; reversed intervals describe the same events. On this event every known edge is a true edge, so $T(P)\in\mathcal C(G)$. Proposition~\ref{prop:uc-completion-enclosure} applies. The argument covers witnesses selected after seeing the intervals because it uses a single simultaneous whole-graph event.
\end{proof}
The strict inequality matters: an interval endpoint at $1/2$ never licenses an orientation. Latent ties are excluded. This is one fixed design and one fixed confidence level, with no outcome-dependent stopping, adaptive allocation, or anytime guarantee. Human ordinal records sampled without replacement from a finite cohort do not satisfy this conditional i.i.d. observation model. For those records, the logical enclosure remains valid \emph{if its supplied edges are correct}, but these intervals supply no asserted finite-population or population coverage. Equal inner and outer sets can identify a UC without identifying every edge; this structural possibility alone establishes no allocated-comparison savings.

\subsection{Orientation laws, support, and posterior membership}\label{app:posterior-law-support}
Under independent positive Beta priors and a factorized pairwise Bernoulli likelihood, the pair posterior is Beta$(a_{ij}+W_{ij},b_{ij}+W_{ji})$. Its orientation probability is
\[
q_{ij}=\Pr(P_{ij}>1/2\mid D)=1-I_{1/2}(a_{ij}+W_{ij},b_{ij}+W_{ji}),
\]
where $I_x$ is the regularized incomplete beta function. This quantity differs from $\mathbb E[P_{ij}\mid D]$. Independent posterior orientations follow from the stated prior and likelihood factorization, not from reciprocal marginal probabilities alone. The posterior puts no mass at $P_{ij}=1/2$.

\begin{lemma}[Support and dependence distinctions]\label{lem:posterior-support-dependence}
With finite counts and positive Beta parameters, every strict tournament has positive probability under this independent posterior. Consequently its individual support-possible UC winners are $V$, and its support-necessary UC winners are empty. More generally, edge marginals do not identify posterior UC membership probabilities.
\end{lemma}
\begin{proof}
The Beta density is positive on $(0,1)$, so each of its two halves has positive probability: $0<q_{ij}<1$. Each strict tournament probability is a finite product of positive orientation probabilities. For any vertex $v$, a supported tournament making $v$ a Condorcet winner has $v\in\UC$; one making $v$ a Condorcet loser has another vertex covering it, so $v\notin\UC$. Taking support union and intersection proves the first assertions.

For the final assertion, compare two laws on three vertices. Independent fair orientations give the uniform law on eight tournaments and UC inclusion probability $1/2$ at every vertex, as in Lemma~\ref{lem:distinct-summaries}. A law assigning probability $1/2$ to each of the two oppositely oriented three-cycles has the same edge marginals $q_{ij}=1/2$, but every vertex belongs to UC with probability one. Thus equal marginals can have different UC laws.
\end{proof}
Informative confidence-completion sets therefore differ from the unrestricted support of a positive Beta posterior. A learned finite mixture can model dependent orientations, but its training objective does not make that law a calibrated Bayesian posterior.

\begin{proposition}[Conditional certificate probability]\label{prop:posterior-tree-bound}
Fix data $D$, a joint tournament law conditional on $D$, and a depth-two spanning out-tree $C_v$ rooted at $v$. If $q_e$ is the marginal probability of each required orientation, then
\[
\Pr(v\in\UC(T)\mid D)\ge\max\{0,1-\textstyle\sum_{e\in C_v}(1-q_e)\}.
\]
If those orientations are jointly independent, their product $\prod_{e\in C_v}q_e$ is a valid, at least as large, lower bound.
\end{proposition}
\begin{proof}
When every tree edge holds, the root reaches every vertex in at most two edges and is uncovered. A union bound on failures of the required orientations gives the first bound. Under independence, the probability that all required edges hold is their product. Applying the same union bound to that event shows that the product is at least $1-\sum_e(1-q_e)$; it is also nonnegative.
\end{proof}
The tree may be selected from $D$ before this conditional probability is evaluated. This is an assertion under a specified conditional law, not repeated-data coverage or optimal tree selection. In a mixture with weights $\rho_c$ and conditional independent edge probabilities $q_{c,e}$, the tree event has probability $\sum_c\rho_c\prod_{e\in C_v}q_{c,e}$, by conditioning on $c$. Multiplying the mixture marginals instead generally gives a different law.

\begin{proposition}[Independent-law UC integration]\label{prop:posterior-uc-integration}
Under mutually independent unordered orientations, let $W=N_T^+(v)$, whose elements $w\ne v$ are independently included with probabilities $q_{vw}$. Define
\[
H_v(W)=\prod_{u\notin W\cup\{v\}}\left(1-\prod_{w\in W}q_{uw}\right).
\]
Then $\pi_v:=\Pr(v\in\UC(T)\mid D)=\mathbb E[H_v(W)\mid D]$, and
\[
\operatorname{Var}(H_v(W)\mid D)
=\pi_v(1-\pi_v)-\mathbb E[H_v(W)(1-H_v(W))\mid D]
\le\pi_v(1-\pi_v).
\]
\end{proposition}
\begin{proof}
Conditional on $W$, precisely the vertices outside $W\cup\{v\}$ beat $v$. Such a vertex $u$ covers $v$ exactly when it beats every $w\in W$, with conditional probability $\prod_{w\in W}q_{uw}$. Conditioning fixes only edges incident to $v$; the remaining orientations retain independence. The covering events for different outside vertices involve disjoint unordered edges, so their failures multiply, giving $H_v(W)=\Pr(v\in\UC(T)\mid W,D)$. Total expectation proves the identity. Let $Y=\mathbf1\{v\in\UC(T)\}$. Total variance gives
\[
\operatorname{Var}(Y\mid D)=\operatorname{Var}(\mathbb E[Y\mid W,D]\mid D)+\mathbb E[\operatorname{Var}(Y\mid W,D)\mid D].
\]
The two Bernoulli variances are $\pi_v(1-\pi_v)$ and $H_v(W)(1-H_v(W))$, proving the claim.
\end{proof}
Empty products give $H_v(\varnothing)=0$ for $n>1$ and $H_v(V\setminus\{v\})=1$. Averaging independent $W$ draws is an unbiased scalar Rao--Blackwell estimator; its variance divides by the draw count. Summing explicitly over all $2^{n-1}$ choices of $W$ gives an exponential exact oracle, not a polynomial-time method. Within a mixture, apply the identity conditional on its component and average over components; using only mixture marginal edges in the products is generally invalid. This identity specifies the estimator; no novelty is claimed for conditioning or Rao--Blackwellization. Its scalar variance result does not imply reduced cross-vertex covariance or better set-level decisions.

\subsection{Joint cardinality information and expected F1}\label{app:posterior-gfm}
Let $S=\UC(T)$ under any specified joint law, and define $F(A,S)=2|A\cap S|/(|A|+|S|)$ for an action $A\subseteq V$. The target $S$ is always nonempty: a maximum-outdegree vertex cannot be covered, because its coverer would have strictly greater outdegree. Thus the empty action has utility zero. Write
\[
J_{i,s}=\Pr(i\in S,|S|=s\mid D),\qquad
\Delta_{i,k}=\sum_{s=1}^n\frac{2J_{i,s}}{k+s}.
\]
\begin{proposition}[Established GFM rule for UC labels]\label{prop:uc-gfm}
For each $k\in\{1,\ldots,n\}$, take the $k$ largest $\Delta_{i,k}$ and then choose the cardinality whose selected weights have greatest sum. This maximizes $U(A)=\mathbb E[F(A,S)\mid D]$ over all subsets $A\subseteq V$.
\end{proposition}
\begin{proof}
For $|A|=k$, linearity gives $U(A)=\sum_{i\in A}\Delta_{i,k}$. The greatest sum at that size selects the $k$ greatest weights, with deterministic tie breaking. Searching all sizes therefore optimizes all nonempty actions. Some nonempty action has positive utility because $S$ is nonempty, so the empty action does not improve the maximum.
\end{proof}
This is General F-measure Maximization \citep{waegeman2014fmeasure}, applied to UC labels. The joint size weights, rather than membership marginals alone, encode the F1 denominator. No label or edge independence is required, and the optimal action need not itself be a UC realizable by a tournament. GFM applied to posterior draws maximizes the empirical modeled utility. It does not certify that the model matches the empirical reference or a population target.

\begin{corollary}[Finite posterior-simulation regret]\label{cor:posterior-gfm-mc}
For a fixed conditional law, take $B$ i.i.d. draws $S_b$ and let $\widehat A$ maximize $\widehat U(A)=B^{-1}\sum_bF(A,S_b)$. For $\epsilon>0$ and $\delta_{\mathrm{MC}}\in(0,1)$, if
\[
B\ge\frac{\log(2\cdot2^n/\delta_{\mathrm{MC}})}{2\epsilon^2},
\]
then, with probability at least $1-\delta_{\mathrm{MC}}$ over this simulation, $U(A^*)-U(\widehat A)\le2\epsilon$, where $A^*$ maximizes $U$.
\end{corollary}
\begin{proof}
Each F1 summand lies in $[0,1]$. Hoeffding's inequality and a union bound over the $2^n$ actions give $\sup_A|\widehat U(A)-U(A)|\le\epsilon$ at the stated probability. On that event,
\[
U(A^*)\le\widehat U(A^*)+\epsilon\le\widehat U(\widehat A)+\epsilon\le U(\widehat A)+2\epsilon.
\]
\end{proof}
This finite-class concentration consequence concerns integration error under a fixed law, independently of the comparison-sampling confidence level in Corollary~\ref{cor:uc-fixed-design-enclosure}. It supplies neither posterior calibration nor a guarantee for learned predictions.

\subsection{The finite-mixture training objective}\label{app:mixture-score-identity}
The additional mixture model has finite differentiable logits $\alpha_c(\theta)$ and $z_{c,e}(\theta)$ for components $c\in\{1,\ldots,C\}$ and unordered edges $e=(i,j)$, $i<j$. Put $\rho_c=\operatorname{softmax}(\alpha)_c$, $q_{c,e}=\sigma(z_{c,e})$, and let $x_e(T)$ indicate $i\to j$. The joint sampled law is
\[
p_\theta(c,T)=\rho_c\prod_e q_{c,e}^{x_e(T)}(1-q_{c,e})^{1-x_e(T)}.
\]
It has positive support on every $(c,T)$ at finite logits; marginalizing $c$ generally creates dependent edges. For fixed reference $S^*$, let $R(T)=F(\UC(T),S^*)$ and $\mathcal J(\theta)=\mathbb E_\theta[R(T)]$.

\begin{proposition}[Score identity and leave-one-out baseline]\label{prop:mixture-score-identity}
Define
\[
g_\theta(c,T)=\nabla_\theta\alpha_c-\sum_d\rho_d\nabla_\theta\alpha_d
+\sum_e(x_e(T)-q_{c,e})\nabla_\theta z_{c,e}.
\]
Then $\nabla_\theta\mathcal J=\mathbb E_\theta[R(T)g_\theta(c,T)]$. For $b=1,\ldots,B$, $B\ge2$, independently sample $(c_b,T_b)$ from the same current law and set $b_b=(B-1)^{-1}\sum_{a\ne b}R(T_a)$. The estimator
\[
\widehat g=B^{-1}\sum_b(R(T_b)-b_b)g_\theta(c_b,T_b)
\]
is unbiased for $\nabla_\theta\mathcal J$.
\end{proposition}
\begin{proof}
The state space is finite and $R$ is independent of $\theta$, so differentiating the sum for $\mathcal J$ yields
\[
\nabla_\theta\mathcal J
=\sum_{c,T}R(T)\nabla_\theta p_\theta(c,T)
=\mathbb E_\theta[R(T)\nabla_\theta\log p_\theta(c,T)].
\]
The second equality uses positive probabilities. Softmax differentiation gives $\nabla\log\rho_c=\nabla\alpha_c-\sum_d\rho_d\nabla\alpha_d$. Differentiating each Bernoulli log probability gives $(x_e-q_{c,e})\nabla z_{c,e}$, establishing the displayed score. Normalization implies $\mathbb E_\theta[g_\theta]=\sum_{c,T}\nabla p_\theta(c,T)=\nabla1=0$. For each $b$, its baseline depends only on other independent draws; hence $\mathbb E[b_bg_b]=\mathbb E[b_b]\mathbb E[g_b]=0$. The expectation of each uncentered term is the first identity, so centering and averaging preserve it.
\end{proof}
This is a finite-state likelihood-ratio/REINFORCE construction \citep{williams1992reinforce}. In the sampled log-probability surrogate, rewards and baselines must be held constant during differentiation. Sampling and scoring must use the same law; inconsistent clipping would break the identity. It establishes neither optimizer convergence nor variance reduction for every baseline. The sampled-reference objective $\mathcal J$, GFM's utility of an action against a modeled random UC, and held-out empirical F1 are three different quantities. Increasing the first does not, by this result, guarantee improvement of the other two. The completed mixture comparison is therefore an empirical test of this modeling choice, rather than a consequence of the gradient identity.

\section{Implementation conventions and historical provenance}\label{app:implementation-provenance}\label{app:implementation}\label{app:historical-overview}
Appendices~\ref{app:ibm-gpu}--\ref{app:real-learning} provide the current studies' complete protocols. This section records operator distinctions and summarizes earlier evidence. The source package retains the earlier protocols, complete result tables, and their associated code and archives separately. Their data and analyses are not pooled with any current study.

\subsection{Operator conventions and missing comparisons}
Algorithm~\ref{alg:lse} returns canonical scores before a selection rule is applied. Its streamed NumPy reference and optional autograd implementation agree on tested values; numerical checks cover finite differences, finite-temperature inequalities, and all strict tournaments with two, three, and four alternatives. These checks complement the proofs in Appendix~\ref{app:proofs} rather than constituting formal verification.

Hard TC uses Boolean transitive closure with self-reachability. Hard UC tests the direct edge and every covering witness. An empirically missing or exactly tied pair contributes neither hard direction. This partial-graph convention does not arbitrarily complete a strict tournament; strict-margin guarantees are not applied to its tied targets. Native hard sets differ from oracle-size selections of their binary scores, whose ties may require a stated tie rule. Canonical missing-pair probabilities of $1/2$ encode neutrality, not observed equality or posterior certainty.

\paragraph{Weighted/Boltzmann variant.}\label{app:variant-details}
For decisive counts $m_{ab}=w_{ab}+w_{ba}$ and Jeffreys means $\widehat P_{ab}=(w_{ab}+1/2)/(m_{ab}+1)$, \WB\ uses the off-diagonal edges
\begin{equation}\label{eq:weighted}
 D^{\rm WB}_{ab}=\frac{m_{ab}}{m_{ab}+5}\,\sigma((\widehat P_{ab}-1/2)/\tau).
\end{equation}
Its UC extrema use the Boltzmann mean
\begin{equation}\label{eq:boltzmann}
 B_\gamma(z)=\frac{\sum_i z_i e^{z_i/\gamma}}{\sum_i e^{z_i/\gamma}}.
\end{equation}
The constant 5 is a fixed historical confidence scale: its count factor is one half at five decisive comparisons. It is a retained heuristic; no separate tuning or optimality claim is made for it. Canonical \LSE\ and the matched-epoch study omit this count factor.

For TC, put $\bar K=\min(K,n-1)$ for a requested $K\ge1$; an omitted or nonpositive code argument uses $n-1$. The executed weighted recurrence is
\begin{align}
Q_{\rm WB}^{(1)}&=D^{\rm WB},\qquad
Q_{{\rm WB},ab}^{(k)}=\max_{c\in\agents}\min\{Q_{{\rm WB},ac}^{(k-1)},D^{\rm WB}_{cb}\},\nonumber\\
R_{{\rm WB},ab}&=\max_{1\le k\le\bar K}Q_{{\rm WB},ab}^{(k)},\qquad
t^{\rm WB}_a=\smin_\gamma\{R_{{\rm WB},ab}:b\ne a\}.\label{eq:historical-wb-tc}
\end{align}
The diagonals of $D^{\rm WB}$ and the final $R_{\rm WB}$ are zero. Unlike the canonical recurrence, the path product and length aggregation here use exact extrema and are only piecewise differentiable. UC uses Eqs.~\eqref{eq:violation}--\eqref{eq:uc} with $D^{\rm WB}$ and $B_\gamma$ replacing both maxima. The Boltzmann sums include every logically eligible witness or cover, including zero entries; missing edges are not removed from their denominators. The code clips edge logits to $[-60,60]$ and floors edge/operator temperatures at $10^{-8}/10^{-12}$, below every reported setting. The finite-temperature outputs are not interchangeable with \LSE, and lowering temperature alone does not remove the confidence-weighting error in Appendix~\ref{app:proofs}. The source package preserves the historical implementation separately from the canonical reference.

\subsection{Context mixtures after conditioning on decisive outcomes}\label{app:decisive-mixtures}
Let $r^{(c)}_{ab}=\Pr(\text{decisive}\mid a,b,c)$, and let $q_c$ now denote a context distribution \emph{before} observing whether a comparison is decisive. The law of total probability gives, when $\sum_c q_cr^{(c)}_{ab}>0$,
\[
\Pr(a\succ b\mid a,b,\text{decisive};q)
=\frac{\sum_c q_cr^{(c)}_{ab}P^{(c)}_{ab}}{\sum_c q_cr^{(c)}_{ab}}
=\sum_c\widetilde q_{c,ab}P^{(c)}_{ab},
\qquad
\widetilde q_{c,ab}=\frac{q_cr^{(c)}_{ab}}{\sum_d q_dr^{(d)}_{ab}}.
\]
These conditioning weights can depend on the pair. They equal the prespecified common weights $q_c$ when decisiveness rates are constant across contexts for that pair. Alternatively, an analyst can define a \emph{standardized} target $\sum_cq_cP^{(c)}_{ab}$ directly and sample or reweight decisive observations to that definition. The context experiment does exactly the latter with Bernoulli comparisons, so $r^{(c)}_{ab}=1$ and there is no distinction. Historical human-preference diagnostics condition on each release's observed decisive comparisons; they do not identify an unobserved common task mixture. The identity above clarifies the estimand; it introduces no additional observations or experiment.

Even strictly oriented context relations can yield an exactly tied mixture pair. Its reference graph omits both directions at $P_{ab}(q)=1/2$, and strict-margin recovery guarantees are not applied. Historical context comparisons supply the resulting reference core size to every method; reweighting also benefits ranking baselines, so this is neither automatic cardinality recovery nor an STE-specific identification guarantee.

\subsection{Historical recovery without oracle size}\label{app:newcontrolled}\label{sec:unknown}
The canonical unknown-size study fits method/target thresholds on 256 development graphs and evaluates 4,050 shared test graphs. The test design crosses five families, $n\in\{9,17,33\}$, count intensities $\{3,10,30\}$, missing fractions $\{0,.25,.5\}$, and 30 cases per cell. Development excludes the random family and uses sizes 9/17, intensities 8/32, missingness 0/.3, and eight cases per cell. Fifty test cases belong to an earlier inspected pilot and are counted once. Retained-pair counts are lognormally heterogeneous $1+\operatorname{Poisson}(bH)$ with unit-mean $H$.

The development-selected UC/TC cutoffs are .625/.550; native hard sets use neither an oracle size nor a fitted cutoff. Mean LSE UC/TC F1 is .552/.624, while its UC advantage over Rank Centrality is .011. LSE's TC selections average 14.25 alternatives against reference size 9.56. The separated-family TC is the entire panel, so all-selected is perfect there; generator-specific ordering and alternative reporting rules differ. The complete historical analysis retains all methods, primary/native and secondary rules, and its 540 Holm-adjusted planned comparisons. This is descriptive recovery under that design, not an additional confirmation of the current learning claims.

The older confidence-weighted archive instead uses 64,800 factorial instances and oracle-size lists. Reused streams give only 894 distinct trial-seed values, so design cells are not independent replications. Its WB operator and known-size endpoint are separate from the canonical unknown-size study and both GPU studies.

\subsection{Historical single-collection learning}\label{app:newlearning}
The exploratory CPU study uses one 512/128-graph training/development collection at $n=12$ and 256 tests at each $n=12,24$, reused across ten initializations. Its ten-feature, width-32 tanh pair encoder differs from the later sixteen-feature GPU encoder. Pairwise-only, auxiliary membership, and STE-core objectives use matched initial parameters and graph orders; the latter two receive identical UC labels. Adam uses learning rate .001, batch 16, norm-5 clipping, and 40 epochs. Development selects $\lambda\in\{.1,1\}$ and thresholds on the .05 grid for the evaluated readout. Fixed epochs do not imply convergence or equal compute; uncertainty conditions on this one training collection.

A subsequent frozen-checkpoint common-soft-UC comparison gives STE--auxiliary F1 differences .0144/.0066 at $n=12/24$, with unadjusted crossed-bootstrap intervals $[-.0013,.0303]$/$[-.0071,.0207]$. Both cross zero. At $n=24$, STE predicts mean size 11.15 against reference 6.95; all-selected has higher exact recovery on the natural mixture despite lower F1. These limits remain in the preserved results.

\subsection{Historical ten-collection comparison}\label{app:confirmation}
A separate CPU study contains ten independent collections, each with 512 training and 128 development graphs at $n=12$ and 256 tests at each $n=12,24$. Three paired initializations share a collection. The same ten-feature encoder, an ordinary auxiliary head, and a two-round relational head are trained for 160 epochs with the preceding optimizer; development selects loss weight and the .05-grid cutoff separately for each readout. Epoch 40 is primary; selecting among epochs 40/80/160 is prespecified secondary. Only ten collection means determine inference, with approximate $t_9$ intervals and Holm adjustment of the two primary exact sign-flip tests.

Primary STE--ordinary-auxiliary differences under common soft UC are $-.0012/.0091$, with intervals $[-.0061,.0038]/[-.0064,.0246]$ and adjusted $p=.602/.430$. The secondary selected-budget $n=24$ difference is .0276 ($[.0081,.0471]$). Yet the relational head's native F1 .6753/.5754 compares with STE-soft .6520/.5716. At $n=24$, STE selects mean size 10.92 against 6.94 and has 8.91\% exact recovery. The secondary gain is not a replacement primary endpoint or uniform complete-system advantage.

\subsection{Historical recorded-profile and human-data evidence}\label{app:profiles-confirmation}
The earlier PrefLib reconstruction \citep{mattei2013preflib,oneill2013openstv,popov2015consensus} uses 91 development profiles from three sources and 47 tests from six disjoint sources. Twenty nested whole-ballot samples per profile use fractions .05/.10/.25/.50, respecting multiplicities and within-record dependence. Cutoffs are chosen only on development; aggregation averages repetitions, profiles within sources, and sources equally. Reference cores describe the full recorded profiles, not population truth. Forty-six test UCs are singletons and only one is a selective non-singleton. LSE/posterior macro F1 is .8259/.8884, with a descriptive source-level difference interval $[-.1112,-.0139]$. A post hoc full-count check finds within-profile separation on all 47 tests but no universally separating absolute cutoff; decision-rule reanalyses reuse these profiles and do not reverse the original primary comparison.

Other preserved diagnostics include respondent-held-out SUSHI panels \citep{kamishima2009orders,sushidata}, Arena/WebDev topology-matched simulations \citep{chiang2024chatbotarena,arena100data,webdevdata}, context reweighting, and real-count bootstrap stability. SUSHI keeps whole respondents in a fixed 3,000/2,000 split and 1,000 block resamples; on its larger panel BTL is stronger at a common ten-item budget. The Arena/WebDev simulations retain real observation patterns but generate latent outcomes and supply every method the true core size. They do not establish labeled human-core accuracy. Real-count bootstrap stability conditions on observed counts and is neither latent accuracy nor a user-cluster confidence interval. Shared-source releases and language subsets are not independent benchmark families. Public task-log conversions with incomplete denominators are excluded from validated comparative evidence.

\subsection{Reference workspace and timing scope}\label{app:runtime}
The canonical forward reference streams candidate columns for TC and one target at a time for UC, with quadratic workspace; optional autograd retains a larger graph. The archived WB path implementation broadcasts cubic candidates. Its CPU timings and the cloud GPU's vectorized training timings therefore concern different kernels and cannot be interchanged. Forward-only CPU checks use isolated subprocesses, one warmup, three repetitions, and one BLAS thread; resident-memory peaks include interpreter/library allocations. A depth $K<n-1$ changes the TC approximation target and loses the general full-TC recovery guarantee. Appendix~\ref{app:ibm-gpu} separately specifies the current GPU checkpointing and complete-training-step measurements.

\section{Reproduction map and notation}\label{app:reproduction}
The manuscript source and research repository provide scientific code, study-specific environments, manuscript-build instructions, experiment commands, and output-reanalysis instructions. Large full-evidence archives are preserved separately. Table and figure generation reaggregates the preserved outputs without fitting new models. Deterministic checks compare Algorithm~\ref{alg:lse} with the numerical implementation and verify the refined bounds and historical operator equations.

The matched-epoch study, equal-budget soft-UC comparison, hard-UC confirmation, exploratory human-profile learning, the posterior/control extension, and earlier studies retain separate protocols and analyses. The matched-epoch review verifies scores, calibration, and diagnostics, but lacks weights and several sampled arrays needed for complete replay. The equal-budget audit additionally regenerates data and replays predictions from selected weights, as described in Appendix~\ref{app:uc-equal-compute}. The hard-UC and human-profile audits also inspect raw inputs, selection, saved predictions, and metric aggregation (Appendices~\ref{app:hard-confirm}--\ref{app:real-learning}). The posterior/control extension additionally verifies native joint draws, selected-checkpoint outputs and independently reconstructed summaries, while excluding its invalid secondary mixture hard-decoding aggregate (Appendix~\ref{app:posterior-native}). Full optimizer trajectories were not independently retrained. Source attribution corrects the mixed election/sports interpretation of the recorded-data cohort (Appendix~\ref{app:gpu-real-inventory}); preserved measurements are unchanged. Earlier protocols and complete outputs remain separately preserved, with their own operators, observation dependencies, and provider terms; they are not pooled into any current study. Restricted respondent-level records are not redistributed, and provider snapshots were not pinned for every historical release. Acquisition and reconstruction instructions record the applicable access conditions.

\begin{table}[ht]\centering\small
\begin{tabular}{@{}ll@{}}\toprule
Symbol / implementation & Meaning\\\midrule
$P,\widehat P$ & Latent and estimated reciprocal pair probabilities; diagonal $1/2$.\\
$W,T$ & Decisive wins and observed tied counts; masks retain missingness.\\
$D_\tau$ & Zero-diagonal sigmoid edge strength; soft surrogate.\\
$Q^{(k)},R,t$ & Exact-length soft walks, aggregated reachability, TC score.\\
$v,\operatorname{cov},u$ & Covering witness, cover strength, UC score.\\
$\tau,\gamma,\gamma_c$ & Edge, path and covering temperatures.\\
$K$ & Path depth; $K\ge n-1$ required for full hard TC recovery.\\
STE-LSE & Canonical normalized log-sum-exp operators.\\
STE-WB & Historical confidence-weighted/Boltzmann variant.\\
STE-Post & Mean hard-core membership under a specified posterior.\\
$\AP$ / DG & Score-ordering AP / descriptive dominance gap.\\
$J$ & Bootstrap agreement diagnostic, as defined per study.\\\bottomrule
\end{tabular}
\caption{Notation and implementation distinctions. Scores are not calibrated membership probabilities; posterior outputs are conditional on their model.}
\end{table}

\end{document}